\pdfoutput=1
\documentclass{article} %
\usepackage[accepted]{arxiv}

\usepackage{amsmath,amsfonts,bm}
\usepackage{hyperref}
\usepackage{url}
\usepackage{float}
\usepackage{graphicx}
\usepackage{booktabs}
\usepackage{algorithm}
\usepackage{algorithmic}
\usepackage[capitalize]{cleveref}
\usepackage{pgfplots}
\usepackage{tikz}
\usetikzlibrary{backgrounds,calc,shadings,shapes.arrows,shapes.symbols,shadows,fit,positioning,topaths,hobby}
\usepackage{wrapfig}
\usepackage{amsthm}
\usepackage{dsfont}
\usepackage{caption}
\usepackage{subfigure}
\usepackage{titlesec}
\usepackage{enumitem} %
\usepackage{xspace}

\newtheorem{theorem}{Theorem}

\newtheorem{lemma}{Lemma}

\newtheorem{proposition}{Proposition}

\newcommand{\rev}[1]{\textcolor{black}{#1}}

\newcommand{\algglobalrand}{\textsc{BiCompFL-GR}}
\newcommand{\alglocalrand}{\textsc{BiCompFL-PR}}
\newcommand{\alg}{\textsc{BiCompFL}\xspace}

\newcommand{\dimension}{\ensuremath{d}}

\newcommand{\client}{i}
\newcommand{\nclients}{n}
\newcommand{\epoch}{t}
\newcommand{\localepoch}{m}
\newcommand{\nlocalepochs}{L}
\newcommand{\sampleidx}{\ell}
\newcommand{\nmasksul}{n_{\text{UL}}}
\newcommand{\nmasksdl}{n_{\text{DL}}}
\newcommand{\maskul}{y_{\client, \sampleidx}^\epoch}
\newcommand{\maskdl}{x_{\client, \sampleidx}^\epoch}
\newcommand{\maskinf}{x^\epoch}

\newcommand{\scores}[1][\localepoch]{\mathbf{s}_{\client, \epoch}^{(#1)}}
\newcommand{\model}[1][\epoch]{\theta_{#1}}

\newcommand{\weights}[1][\epoch]{w}
\newcommand{\commonmodel}[1][\epoch]{\hat{\theta}_{#1}}
\newcommand{\modelest}[1][\epoch]{\hat{\theta}_{\client, #1}}
\newcommand{\modelestj}[1][\epoch]{\hat{\theta}_{j, #1}}

\newcommand{\prior}{p^\epoch}
\newcommand{\priorul}[1][\epoch]{p_{\client, u}^{#1}}
\newcommand{\priordl}{p_{\client, d}^\epoch}

\newcommand{\posterior}[1][\epoch]{q_{\client}^{\epoch}}
\newcommand{\posteriorgen}[1][\epoch]{q}
\newcommand{\posteriorentry}[1][\epoch]{q_{\client, \entry}^{\epoch}}
\newcommand{\posteriorentrytmp}{q_{\entry}}
\newcommand{\gradient}[1][\epoch]{g_{\client}^{\epoch}}
\newcommand{\gradientest}[1][\epoch]{\hat{g}_{\client}^{\epoch}}
\newcommand{\gradiententry}[1][\epoch]{g_{\client, \entry}^{\epoch}}
\newcommand{\gradientvec}[1][\epoch]{\mathbf{g}}
\newcommand{\gradientvecentry}[1][\epoch]{g_\entry}
\newcommand{\posteriortmp}[1][\epoch]{\tilde{q}_{\client}^\epoch}
\newcommand{\posteriorj}[1][\epoch]{q_{j}^\epoch}
\newcommand{\posteriorest}[1][\epoch]{\hat{q}_{\client}^\epoch}
\newcommand{\posteriorestj}[1][\epoch]{\hat{q}_{j}^\epoch}
\newcommand{\auxul}[1][\epoch]{\tilde{q}_{\client}^\epoch}

\newcommand{\nblocks}{\ensuremath{B}}
\newcommand{\blockidx}{\ensuremath{b}}
\newcommand{\issampleidx}{\ensuremath{I_{\client, \sampleidx}^\blockidx}}
\newcommand{\priorulblock}{\ensuremath{p_{\client, u, \blockidx}^\epoch}}
\newcommand{\posteriorblock}{\ensuremath{q_{\client, \blockidx}^\epoch}}

\newcommand{\kl}[2]{\mathrm{D}_\textrm{KL}\left(#1 \Vert #2\right)}
\newcommand{\klinline}[2]{\mathrm{D}_\textrm{KL}(#1 \Vert #2)}

\newcommand{\klball}{\ensuremath{\rho}}

\newcommand{\data}{\mathcal{D}_\client}
\newcommand{\loss}{F(\commonmodel, \data)}

\newcommand{\ngenericsamples}{\ensuremath{K}}
\newcommand{\isamplesel}{\ensuremath{X_{\sampleidx, I_\sampleidx}}}
\newcommand{\isampleidx}{\ensuremath{i}}
\newcommand{\isampleall}[1][\isampleidx]{\ensuremath{X_{\sampleidx, #1}}}
\newcommand{\binkl}[2]{\ensuremath{\textrm{d}_\textrm{KL}\left(#1||#2\right)}}
\newcommand{\binklinline}[2]{\ensuremath{\textrm{d}_\textrm{KL}(#1||#2)}}

\newcommand{\define}{\ensuremath{:=}}

\newcommand{\iid}{i.i.d.}
\newcommand{\noniid}{non-i.i.d.}
\newcommand{\fracrv}{\mathrm{M}}

\newcommand{\lr}{\eta}

\newcommand{\nqints}{s}
\newcommand{\qchoice}{\tau_\entry}
\newcommand{\stochquant}[1][\cdot]{Q_s(#1)}
\newcommand{\iscompression}[1][\cdot]{\mathcal{C}_\mathrm{mrc}(#1)}
\newcommand{\iscompressionvar}[2]{\mathcal{C}_\mathrm{mrc}(#1, #2)}

\newcommand{\locallipschitz}{\ensuremath{L_\client}}
\newcommand{\lipschitztmp}{\ensuremath{L^\prime}}
\newcommand{\lipschitz}{\ensuremath{L}}
\newcommand{\gradvar}{\ensuremath{\sigma^2}}
\newcommand{\contraction}{\ensuremath{\delta}}

\newcommand{\tmpsample}[1][1]{\ensuremath{x_{#1}}}
\newcommand{\bsample}{\ensuremath{\mathrm{L}}}
\newcommand{\ber}[1]{\ensuremath{\mathrm{Ber}\left(#1\right)}}

\newcommand\numberthis{\addtocounter{equation}{1}\tag{\theequation}}

\newcommand{\isidx}{\ensuremath{i}}
\newcommand{\genericprior}{\ensuremath{P}}
\newcommand{\genericposterior}{\ensuremath{Q}}
\newcommand{\auxdist}{\ensuremath{W}}
\newcommand{\isample}{\ensuremath{X_{\isidx}}}
\newcommand{\nisamples}{\ensuremath{n_\textrm{IS}}}

\newcommand{\entry}{\ensuremath{e}}
\newcommand{\tmpvec}{\mathbf{x}}
\newcommand{\tmpvecentry}{x_\entry}

\newcommand{\posteriorapproxentry}{\ensuremath{\tilde{q}_{\entry}}}
\newcommand{\priorgen}{\ensuremath{p}}
\newcommand{\priorentry}{\ensuremath{p_{\entry}}}
\newcommand{\norm}[1]{\Vert #1 \Vert}

\icmltitlerunning{BiCompFL: Stochastic Federated Learning with Bi-Directional Compression}

\begin{document}

\twocolumn[
\icmltitle{BiCompFL: Stochastic Federated Learning with Bi-Directional Compression}

\begin{icmlauthorlist}
\icmlauthor{Maximilian Egger}{yyy}
\icmlauthor{Rawad Bitar}{yyy}
\icmlauthor{Antonia Wachter-Zeh}{yyy}
\icmlauthor{Nir Weinberger}{sch}
\icmlauthor{Deniz Gündüz}{comp}
\end{icmlauthorlist}

\icmlaffiliation{yyy}{Technical University of Munich}
\icmlaffiliation{comp}{Imperial College London}
\icmlaffiliation{sch}{Israel Institute of Technology}

\icmlcorrespondingauthor{Maximilian Egger}{maximilian.egger@tum.de}

\icmlkeywords{Machine Learning, ICML}

\vskip 0.3in
]

\printAffiliationsAndNotice{This project has received funding from the German Research Foundation (DFG) under Grant Agreement Nos. BI 2492/1-1 and WA 3907/7-1, and from UKRI for project AI-R (ERC-Consolidator Grant, EP/X030806/1). The work of N.W. was partly supported by the Israel Science Foundation (ISF), grant no. 1782/22.}

\begin{abstract}
     \rev{We address the prominent communication bottleneck in federated learning (FL). We specifically consider stochastic FL, in which models or compressed model updates are specified by distributions rather than deterministic parameters. Stochastic FL offers a principled approach to compression, and has been shown to reduce the communication load under perfect downlink transmission from the federator to the clients.
     However, in practice, both the uplink and downlink communications are  constrained. We show that bi-directional compression for stochastic FL has inherent challenges, which we address by introducing \alg. Our \alg is experimentally shown to reduce the communication cost by an order of magnitude compared to multiple benchmarks, while maintaining state-of-the-art accuracies. Theoretically, we study the communication cost of \alg through a new analysis of an importance-sampling based technique, which exposes the interplay between uplink and downlink communication costs.}
\end{abstract}

\section{Introduction}

Federated learning (FL) is a widely used and well-studied machine learning (ML) framework, where multiple clients orchestrated by a federator collaborate to train an ML model \citep{mcmahan2017communication}. Communication efficiency, privacy, security, and data heterogeneity are critical challenges in FL that have been extensively studied \citep{zhang2021survey,wen2023survey}. In principle, FL is a \textit{bi-directional} process, and with the increasing size of ML models, massive amounts of data are communicated between the federator and the clients. Reducing \emph{uplink} communication from clients to the federator has been the focus of many studies, mainly within the framework of lossy gradient compression, e.g., \citep{seide2014onebit, alistarh2017qsgd, isik2024adaptive}. However, reducing the cost of \textit{downlink} transmission to communicate the updated global model from the federator to the clients has received relatively less attention, although it is as costly and can be a major bottleneck when training over a wireless network.
An ongoing body of research aims to study the communication bottleneck in downlink transmission, by combining tools from gradient compression, momentum, and error-feedback \citep{stich2018sparsified, tang2019doublesqueeze, xie2020cser,  Amiri:arXiv:20, philippenko2020bidirectional, gruntkowska2023ef21, tyurin2023direction,dorfman2023docofl, gruntkowska2024improving}. All these works are focused on non-stochastic (or non-Bayesian) settings. However, the state-of-the-art performance under limited uplink communication is achieved by stochastic compression methods, such as QSGD \cite{alistarh2017qsgd}, QLSD \cite{vono2022qlsd}, dithered quantization \cite{Abdi:arXiv:19} and FedPM \cite{isik2023sparse}, in which the clients send samples from a local distribution, and the federator estimates the mean of the clients' distributions by averaging these samples. To address this gap, in this work, we study the performance of stochastic FL with limited communication in both directions, and propose a method that obtains state-of-the-art results. \rev{Moreover, we show that our method can actually reduce the communication cost even in conventional FL with stochastic compression.}

A fundamental approach to both uni-directional and bi-directional compression schemes involves quantizing transmitted update vectors to finite resolutions. The trade-off between communication cost (or compression) and the quantization distortion has been extensively studied under the framework of rate-distortion theory \citep{thomas2006elements}.  
However,  classical rate-distortion is not well suited for analyzing how quantization affects the convergence of stochastic gradient-based optimization, %
as they rely on the joint compression of many samples and assume additive distortion measures. Consequently, it becomes difficult to characterize the fundamental trade-off between the communication cost and the convergence rates.

An alternative stochastic FL approach was proposed by \citet{isik2024adaptive}, which applies to a variety of Bayesian FL solutions as well as to standard gradient-based methods with stochastic compression. %
Communication reduction is achieved by \textit{minimal random coding} (MRC), which allows the federator to directly sample from the updated local distributions, rather than obtaining quantized versions of samples locally generated by each of the clients. This enables a direct evaluation of the communication cost when a shared common prior distribution, referred to as \textit{side information}, and sufficient common randomness are available between the federator and the clients. When the downlink communication is unlimited, the global model distribution at the federator can then be shared with all the clients, and serves as a natural side information, i.e., common prior. However, this is impossible under downlink communication constraints. This necessitates developing new algorithms and analysis, as we carry in this paper.

The core research question we address is: \textit{Can joint uplink and downlink compression reduce communication bottlenecks in \rev{stochastic FL}?} We answer this question in the affirmative, and we develop and analyze \rev{stochastic FL} algorithms with bi-directional compression. %
We utilize MRC with appropriate priors, and accurately characterize the uplink and downlink communication costs \rev{and the compression error}. \rev{When applied to conventional gradient-based methods, we prove a contraction property of our compression method, thereby facilitating convergence analysis for both uni- and bi-directional MRC-based stochastic compression.} We also examine key performance factors including client data heterogeneity, availability of shared randomness among clients\rev{, and various hyperparameters}. Our main contributions are summarized next. %

\subsection{Contributions}
\begin{enumerate}[itemsep=0.5em,parsep=0em, label={$\bullet$}, wide, labelindent =-0pt]
    \item We propose two algorithms for bi-directional \rev{stochastic FL} based on the availability of shared randomness: one for the case when globally shared randomness is available, and another for the case when only private shared randomness between each client and the federator is available. Both algorithms use carefully chosen side information to transmit samples from the desired distribution through MRC.
    \item We experimentally validate our method on existing baselines, and demonstrate \textit{order-wise} reductions in the communication cost, \rev{while maintaining similar accuracies.} \rev{We thoroughly investigate the role of shared randomness and the choice of side information.}
    \item \rev{We apply our method to stochastic compression in conventional FL, achieving substantial reductions in communication cost. We establish convergence guarantees by proving a contraction property for the biased compressors used in our algorithms.}
    \item We develop a theoretical framework for MRC to quantify communication costs in stochastic FL with bi-directional compression. Our findings go beyond the established analysis of \citet{chatterjee2018sample}, providing refined results for Bernoulli distributions that may be of independent interest. \rev{Our theoretical framework further allows targeted convergence analysis, and provides techniques applicable to other distributions.}
\end{enumerate}

\section{Preliminaries: Stochastic FL with Bi-Directional Compression}

\newcommand{\aggrule}[1][\cdot]{\ensuremath{R\left(#1 \right)}}

We propose a general stochastic FL algorithm \rev{that employs stochastic bi-directional compression based on MRC}. In what follows, we shortly review these concepts.
 
\textbf{Stochastic FL.} A set of $\nclients$ clients collaboratively and iteratively train a model, e.g., a neural network, under the orchestration of a federator. Client $\client \in [\nclients]\define \{1, \dots, n\}$ possesses a dataset $\data$. %
We differentiate between homogeneous data, where $\data$ is drawn independently from the same distribution for all clients (i.i.d.), and heterogeneous data, where each $\data$ may come from a different distribution (non i.i.d.). 
At each iteration $t$ of the training, the federator holds a model $\model$ described by a probability distribution. %
After downlink transmission, each client $\client$ has an estimate $\modelest$ of $\model$, and locally optimizes $\modelest$ to obtain a local probabilistic model called \textit{posterior} $\posterior$. 
Compressed versions of the clients' posteriors $\posterior$ are transmitted back to the federator on the uplink to obtain an estimate $\posteriorest$. The federator aggregates the received posteriors using an aggregation rule $\aggrule$ 
to obtain a refined global probability distribution $\model[\epoch+1] = \aggrule[\{\posteriorest\}_{\client \in [\nclients]}]$. A simple aggregation rule $\aggrule$ is the average over all clients' posteriors. This process is repeated until a certain convergence criterion is met. In many stochastic FL settings, the transmitted client updates $\posteriorest$ are samples from the posterior distribution $\posterior$. 

\rev{Furthermore, our definition of Stochastic FL encompasses conventional FL with stochastic quantization. The same procedure as above follows with those differences: (i) the federator holds a model $\model$ with deterministic parameters; (ii) each client $\client$ locally optimizes $\modelest$ to obtain a local gradient $\gradient$. A stochastic compression $\stochquant$ is applied on the client's gradient to obtain a posterior distribution $\posterior$ from $\stochquant[\gradient]$; (iii) samples of $\posterior$ are transmitted to the federator on the uplink to obtain an estimate of the gradient, which we still denote by $\posteriorest$; and (iv) the federator updates the global model as $\model[\epoch+1] = \model[\epoch] - \lr \aggrule[\{\posteriorest\}_{\client \in [\nclients]}]$, with learning rate $\lr$. %
We will investigate both settings, with a prominent focus on the former.} %

\rev{\textbf{Stochastic Compression by MRC.}
To efficiently transmit samples from the posterior $\posterior$, we employ MRC \citep{havasi2018minimal} to leverage common side information present at the federator and the clients, and shared randomness. This method serves as stochastic compressor $\iscompression$, which takes as input a posterior distribution $\genericposterior$ and a prior distribution $\genericprior$, and outputs a sample from a distribution $\hat{\genericposterior}$ close to $\genericposterior$. In MRC, the encoder and decoder generate 
$\nisamples$ samples $\{\isample\}_{\isidx\in[\nisamples]}$ from $\genericprior$. The encoder computes a categorical distribution $\auxdist$, with $\auxdist(\isidx) = \frac{\genericposterior(\isample)/\genericprior(\isample)}{\sum_{\isidx=1}^{\nisamples} \genericposterior(\isample)/\genericprior(\isample)}$, and transmits an index $\isidx \sim \auxdist$ with $\log_2(\nisamples)$ bits. %
The encoder sets $\nisamples = \Theta(\exp(\kl{\genericposterior}{\genericprior}))$, where $\kl{\genericposterior}{\genericprior}$ denotes the KL-divergence between $\genericposterior$ and $\genericprior$  \citep{chatterjee2018sample,havasi2018minimal}. For two Bernoulli distributions with parameters $q$ and $p$ we use the short notations $\binkl{q}{p}$ and $\iscompressionvar{q}{p}$.}

\section{\alg} \label{sec:bicompfl} %
In this section, we introduce our proposed scheme, \alg, a bi-directional stochastic compression strategy that uses MRC to reduce both uplink and downlink communication costs. %
The scheme relies on the availability of shared randomness between each of the clients and the federator, which can be implemented using pseudo-random sequences generated from a common seed. %
We distinguish between two types of shared randomness: private shared randomness (between individual clients and the federator) and global shared common randomness (among all parties), with the latter being more challenging to implement in practice. We assume all clients and the federator share the same global model $\commonmodel[0]$ at initialization. This does not incur any communication cost when global shared randomness is available, but necessitates an initial model transmission from the federator to clients when only private shared randomness exists. 

\textbf{\alg: The General Algorithm.} Our method serves as a general framework for stochastic optimization procedures. We explain \alg for Bayesian FL and show in the sequel how it can be used for conventional FL with stochastic quantization. Consider probabilistic mask training (similar to FedPM, \citep{isik2023sparse}) as an example of Bayesian FL. \rev{Let $[0,1] \define \{x\in \mathbb{R}: 0 \leq x \leq 1\}$. The models $\model \in [0, 1]^\dimension$ of dimension $\dimension$ are parameters of Bernoulli distributions.
Those parameters determine for each weight of a randomly initialized network with fixed weights $\weights$ whether it is activated or not. During inference, the weights $\weights$ are masked with samples $\maskinf \in \{0, 1\}^\dimension \sim \model$, i.e., the network weights are $\weights \odot \maskinf$.} %
We start with a general description, which is valid for the cases of global and private shared randomness. 

At iteration $\epoch=0$, each client $\client \in [\nclients]$ shares with the federator the same global model, i.e., $\modelest[0] = \model[0]$, for all $\client \in [\nclients]$. At iteration $\epoch$, each client $\client$ locally trains model $\modelest$ in $\nlocalepochs$ local iterations. \rev{In our previous example, when training Bernoulli distributions to mask a random network, the parameters are mapped to scores in a dual space, which are then trained for $\nlocalepochs$ local iterations $\localepoch \in [\nlocalepochs]$ using stochastic gradient descent. Mapping the trained scores back to the primal space, each client $\client$ obtains a model update in terms of a posterior $\posterior$. We refer to \cref{sec:fedpm} for details.} 
This \rev{optimization principle} is a special instance of mirror descent, which, in the special case of optimizing over Bernoulli distributions, leads to a point-wise minimization with respect to a KL-proximity term \rev{(as opposed to the Euclidean distance in standard SGD,} cf. \cref{app:mirror_descent} for \rev{details}). \rev{The KL-divergence between the updated local model and the global model directly determines the communication cost. Hence, %
we \textit{regularize} the minimization of the loss function by the communication cost.} This property renders our method superior to various baselines. %

To convey the model update $\posterior$ to the federator, each client employs $\iscompression$ in $\nblocks$ blocks of size $\dimension/\nblocks$ each (assuming $\nblocks\vert \dimension$) with a prior distribution $\priorul$, which is set to $\priorul[0] = \modelest[0]$ at iteration $\epoch=0$. The choice of $\priorul$ for $t>0$ will be clarified later. For each block $\blockidx \in [\dimension/\nblocks]$, client $\client$ conveys $\nmasksul$ samples $\{\maskul\}_{\sampleidx \in [\nmasksul]}$ %
of $\posterior$ to the federator by transmitting for each block $\blockidx$ an index $\issampleidx$ with $\log_2(\nisamples)$ bits, where $\nisamples$ is the number of samples per block, generated from the prior distribution $\priorul$ at both the client and the federator using the available shared randomness. %
The samples of all blocks are concatenated for each $\sampleidx$. 
Hence, the federator obtains an estimate of client $\client$'s posterior distribution using the empirical average $\posteriorest = \frac{1}{\nmasksul} \sum_{\sampleidx=1}^{\nmasksul} \maskul$.

By averaging the estimates $\posteriorest$ for all the clients' models, the federator updates the global model as $\model[\epoch+1] = \frac{1}{\nclients} \sum_{\client=1}^\nclients \posteriorest$. To transmit the new model to each client $\client$, we assume the existence of a common prior $\priordl$ shared by the federator and the clients. With $\priordl$, the federator performs MRC in $\nblocks$ blocks of size $\dimension/\nblocks$ to make client $\client$ sample from, and thereby estimate, the latest global model $\model[\epoch+1]$. The client samples $\nmasksdl$ masks $\{\maskdl\}_{\sampleidx \in [\nmasksdl]}$, each incurring a communication cost of $\nblocks \log_2(\nisamples)$ bits. An estimate of the updated global model is obtained by concatenating the reconstructed samples for all the blocks $\blockidx \in [\nblocks]$, and averaging over all masks $\modelest[\epoch+1] = \frac{1}{\nmasksdl} \sum_{\sampleidx=1}^{\nmasksdl} \maskdl$.

Since the number of clients is typically large, it often suffices to choose $\nmasksul=1$. The clients' contributions are averaged at the federator, effectively reducing the noise due to the MRC step. This allowed \citet{isik2024adaptive} to theoretically analyze the uplink communication cost for  importance sampling-based stochastic communication of model updates.  
We will follow a similar approach for downlink communication; however, since downlink communication cannot benefit from the averaging effect of multiple clients, we reduce the variance of the model estimate in the downlink by setting 
$\nmasksdl = \nclients \cdot \nmasksul$.

The choice of the priors $\priorul$ and $\priordl$ for MRC in the uplink and downlink channels, respectively, crucially affects the performance and the communication cost of the algorithm. As a first-order characterization, the communication cost of MRC is determined by $\klinline{\posterior}{\priorul}$ in the uplink and by $\klinline{\model[\epoch+1]}{\priordl}$ in the downlink. %
\setlength{\textfloatsep}{9pt}
\begin{algorithm}[!t]
\caption{\algglobalrand\ with Global Randomness}
\label{alg:globalrand}
\begin{algorithmic}[1] %
\REQUIRE Both clients and federator initialize the same global model \rev{$\model[0]$} using a shared seed \\
\ENSURE Clients set prior $\prior = \modelest[0] = \model[0], \forall \client \in [\nclients]$ \\
\REPEAT
\FOR{Client $i \in [n]$}
\STATE \rev{$\posterior \gets$ Local training of $\modelest$} \\ %
\STATE Sample indices $\issampleidx, \sampleidx \in [\nmasksul], \blockidx \in [\nblocks]$ from $\posterior$ with prior $\prior$ and transmit to federator to reconstruct $\posteriorest$%
\ENDFOR
\STATE Federator updates global model $\model[\epoch+1] = \frac{1}{\nclients} \sum_{\client=1}^\nclients \posteriorest$ \\
\STATE Federator relays to client $j$ the other clients' indices $\{\issampleidx\}_{\sampleidx \in [\nmasksul], \blockidx \in [\nblocks], \client \in [\nclients]\setminus \{j\}}$ \\
\FOR{Clients $i \in [n]$}
\STATE Reconstruct $\modelest[\epoch+1] = \frac{1}{\nclients} \sum_{\client=1}^\nclients \posteriorest$ from $\{\issampleidx\}$ %

\ENDFOR
 \STATE Clients and federator set prior $\prior = \commonmodel[\epoch+1]$ \\
\STATE $t \gets t+1$
\UNTIL{Convergence}
\end{algorithmic}
\end{algorithm}

\textbf{Global Randomness.} When global shared randomness is available, all clients can maintain the same priors at each iteration $\epoch$, and, thereby, obtain the same global model estimates $\modelest$. The global model is known to the clients and the federator from initialization, and synchronization among all clients is ensured by choosing as prior $\priorul=\priordl$ the latest estimate of the global model $\modelest$. The clients utilize the globally shared randomness to sample the exact same samples from the same prior for uplink transmission at all iterations. Selected indices of such samples are transmitted to the federator to convey an estimate $\posteriorest$ of the posterior $\posterior$, who reconstructs the global model $\model[\epoch+1]$. Using the same prior in the downlink, i.e., the global model from the previous iteration, the updated model can be transmitted to the clients through MRC. Leveraging the shared randomness, all clients $\client \in [\nclients]$ sample from the same prior, and thus obtain the exact same estimate of the global model $\modelest[\epoch+1] = \commonmodel[\epoch+1]$, for all $\client \in [\nclients]$. Hence, we have that $\priorul = \priordl = \commonmodel$ for all $\client \in [\nclients]$.

In this version, the federator reconstructs the global model from estimates of the client posteriors $\posteriorest$. However, in the uplink, all clients sample from the same prior, which enables further improvements. Naively, the federator will reconstruct the global model using the indices $\issampleidx$ for $\blockidx \in [\nblocks], \sampleidx \in [\nmasksul]$ received by the clients $\client \in [\nclients]$ through MRC, followed by an additional round of MRC for downlink transmission.  Instead, and more efficiently, the federator can simply relay the indices to the respective other clients (i.e., client $j$ receives $\issampleidx$ for $\blockidx \in [\nblocks], \client \in [\nclients]\setminus \{j\}, \sampleidx \in [\nmasksul]$), which reconstruct the same updated global model individually. This avoids introducing additional noise by a second round of compression and allows better convergence without additional communication facilitated by global randomness. We term this approach \algglobalrand\ and summarize the procedure in \cref{alg:globalrand}.

\textbf{Private Randomness.}
Without global randomness, maintaining the same prior among all clients is impossible without introducing additional communication. Instead, an additional round of MRC is needed for the downlink transmission, and each client obtains a different estimate of the global model $\modelest$ at each iteration. Hence, the clients' local trainings start from different estimates of the global model. In a non-stochastic setting, such a phenomenon has only been considered by \citet{philippenko2021preserved, gruntkowska2024improving}. This raises the questions of the additional cost incurred due to lack of shared randomness in terms of both the convergence speed and the communication load and the choice of the priors $\priorul$ and $\priordl$.

\begin{algorithm}[!t]
\caption{\alglocalrand\ with Private Randomness}
\label{alg:localrand}
\begin{algorithmic}[1] %
\REQUIRE Both clients and federator initialize the same global model  \rev{$\model[0]$} using a shared seed \\
\ENSURE Clients set prior $\priorul = \priordl = \modelest[0] = \model[0], \forall \client \in [\nclients]$ \\
\REPEAT
\FOR{Client $i \in [n]$}
\STATE \rev{$\posterior \gets $ Local training of $\modelest$} \\ %
\STATE Federator \rev{employs $\iscompressionvar{\posterior}{\priorul}$ to draw} $\nmasksul$ \rev{samples} $\maskul \sim \posterior$ \rev{using} prior $\priorul$ \\
\STATE Federator est. client's posterior $\posteriorest = \frac{1}{\nmasksul} \sum_{\sampleidx=1}^{\nmasksul} \maskul$
\ENDFOR
\STATE Federator updates global model $\model[\epoch+1] = \frac{1}{\nclients} \sum_{\client=1}^\nclients \posteriorest$ \\
\FOR{Clients $i \in [n]$}
\STATE Client \rev{employs $\iscompressionvar{\model[\epoch+1]}{\priordl}$ to draw} $\nmasksdl$ \rev{samples} $ \maskdl \sim \model[\epoch+1]$ \rev{using} prior $\priordl$ \\
\STATE Client est. global model: $\modelest[\epoch+1] = \frac{1}{\nmasksdl} \sum_{\sampleidx=1}^{\nmasksdl} \maskdl$
\STATE Clients set prior $\priorul = \priordl = \modelest[\epoch+1]$ \\ 
\ENDFOR
\STATE $t \gets t+1$
\UNTIL{Convergence}
\end{algorithmic}
\end{algorithm}

For the uplink transmission of client $\client$, any convex combination of $\modelest$ and $\posteriorest$ can be used as prior, i.e., $\priorul = \lambda \modelest + (1-\lambda) \posteriorest[\epoch-1]$, for some $0 \leq \lambda \leq 1$.\footnote{This adds a negligible cost of transmitting $\lambda$ if it is to be optimized at each round\rev{, cf. \cref{sec:optimize_prior} for details}.} This is due to the availability of both quantities at the federator and client $\client$. However, small $\lambda$ values are not expected to reduce the cost of communication reflected by $\binkl{\posterior}{\priorul}$ since the \rev{previous} global model estimate is likely to be similarly different from the posterior (in terms of the KL-divergence) than the \rev{previous} posterior estimate of the federator. 
Indeed, our numerical experiments have shown that the savings from choosing $\lambda \neq 1$, i.e., priors other than $\modelest$, are not significant. For simplicity, we thus propose to use $\priorul = \priordl = \modelest$. We term this approach \alglocalrand\, and summarize the procedure in \cref{alg:localrand}. Choosing different priors is possible and only affects line 11 in \cref{alg:localrand}. 
We mention in passing that  \alglocalrand\ allows partial client participation, which is incompatible with shared randomness and the method \algglobalrand.

\textbf{Block Allocation. }
We consider three different block allocation strategies: 1) fixed block size (referred to as ``Fixed'' in the experiments), where each block $\blockidx \in [\nblocks]$ is of the same size and constant across all $t$; 2) adaptive block allocation (Adaptive) as proposed by \citet{isik2024adaptive}, where each block size is separately optimized each iteration $t$; and 3) adaptive average allocation (Adaptive-Avg), where the block sizes are equal but optimized at each iteration $t$ according to the average KL-divergence per block. We refer the reader to \cref{app:block_allocation} for a detailed discussion on this. 

\section{Experiments} \label{sec:experiments}

We conduct experiments to evaluate the performance of our proposed $\algglobalrand$ and $\alglocalrand$ schemes, and compare against baseline FL strategies without compression (FedAvg or PSGD) \cite{mcmahan2017communication} and several non-stochastic bi-directional compression schemes that employ different combinations of compression, error-feedback, and momentum. 
\begin{figure}
    \centering
    \includegraphics[width=1\linewidth]{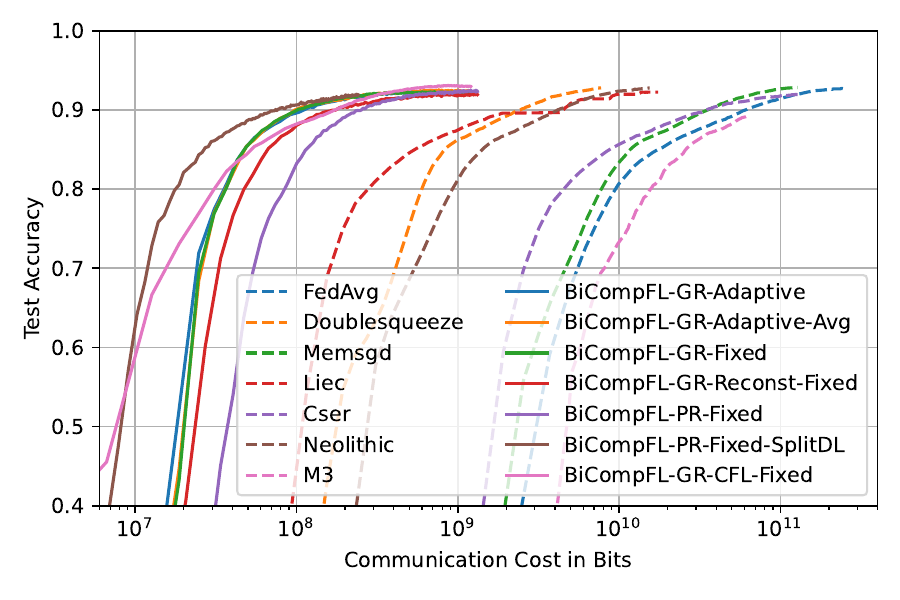} \vspace{-.7cm}
    \caption{Test accuracy for \alg\ and baselines on Fashion MNIST 4CNN on \iid\ data. %
    }
    \label{fig:mnist4cnn_uniform}
\end{figure}
\begin{figure*}[!t]
\subfigure[MNIST 4CNN \iid]{\includegraphics[width = .3298\linewidth]{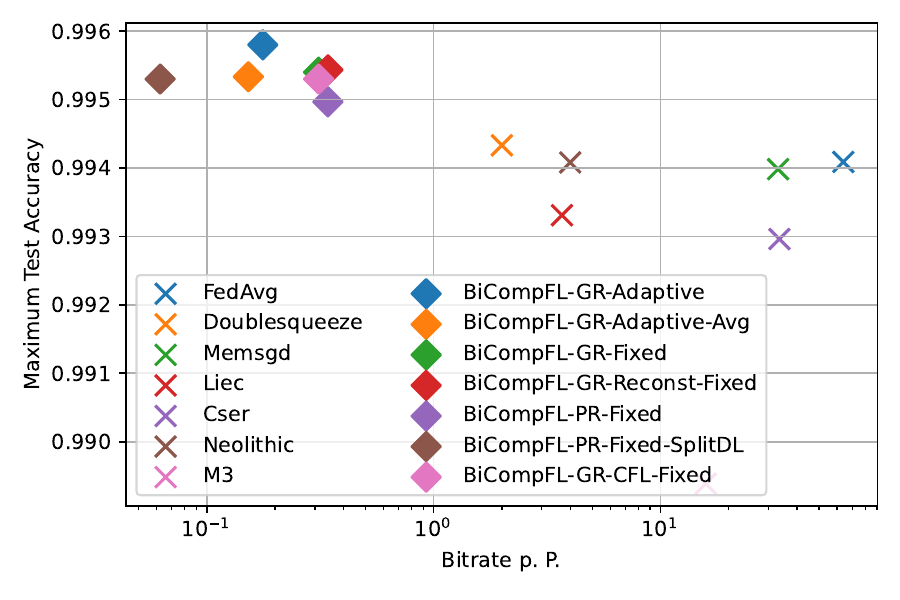}\label{fig:acc_bitrate_fmnist_iid}}
\subfigure[MNIST 4CNN \noniid]{\includegraphics[width = .3302\linewidth]{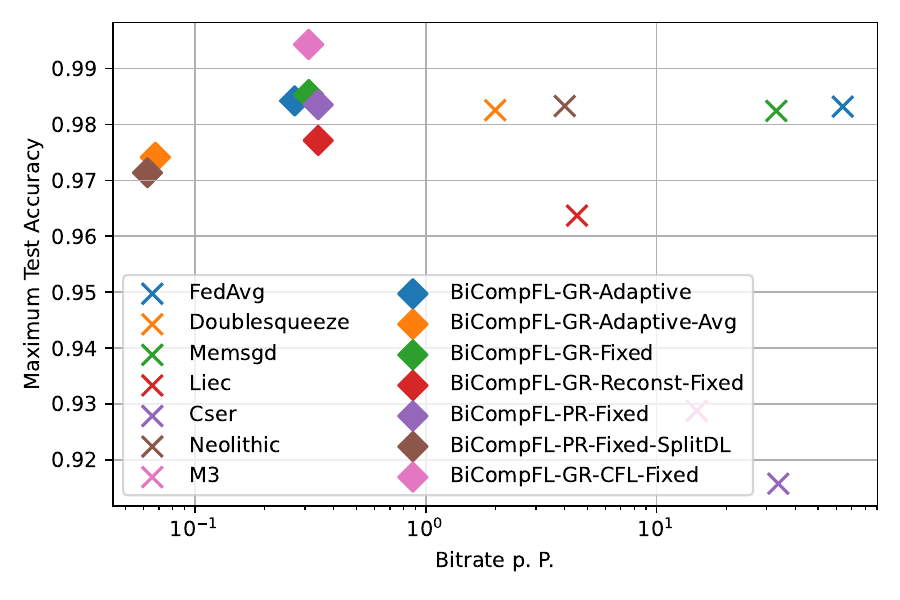}\label{fig:acc_bitrate_fmnist_noniid}}
\subfigure[CIFAR-10 6CNN \iid]{\includegraphics[width = .33\linewidth]{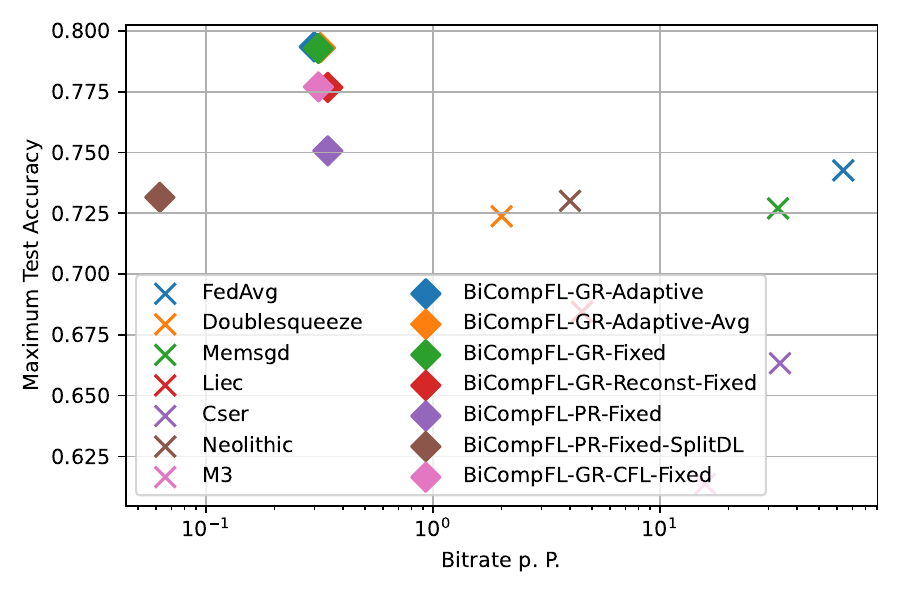}\label{fig:acc_bitrate_cifar_iid}}
\vspace{-.1cm}
\caption{Maximum test accuracy as a function of the total communication cost measured as the bitrate per parameter.} \vspace{-.2cm}
\end{figure*}
In particular, we compare against \textsc{DoubleSqueeze} \citep{tang2019doublesqueeze}, \textsc{Mem-SGD} \citep{stich2018sparsified}, \textsc{Neolithic} \citep{huang2022lower}, \textsc{Cser} \citep{xie2020cser}, and the recently proposed \textsc{Liec} \citep{cheng2024communication}. 
SignSGD \citep{seide2014onebit} serves to compress the transmitted gradients for all the schemes. 
We further compare with \textsc{M3} \citep{gruntkowska2024improving}, which partitions the model into disjoint parts for downlink transmission and transmits to each client a different part of the model. While M3 is focused on RandK compression for the uplink (i.e., transmitting random $K$ entries of the gradient), we use TopK \citep{wangni2018gradient, shi2019distributed}, which we found to achieve much more stable results.

\newcommand{\signsgdconst}{\ensuremath{K}}
\newcommand{\algglobalrandcfl}{\algglobalrand-CFL}
\rev{As mentioned above, the mirror descent approach outlined in \cref{sec:bicompfl} inherently minimizes the communication cost as a by-product; and hence, provides a strong candidate for communication-efficient stochastic FL. %
Nonetheless, we show how our method can be used to improve the communication efficiency in conventional FL by using the uplink and downlink compression $\iscompression$ combined with stochastic quantizers, e.g., \citep{alistarh2017qsgd}. In \cref{sec:theory}, we pave the way to convergence guarantees by proving a contraction property of $\iscompression$ composed with a stochastic quantization $\stochquant$ of gradients $\gradient$. To compare our method to the baselines that use SignSGD as compressor, we evaluate \algglobalrand\ in a conventional federated learning (CFL) task with a stochastic variant of SignSGD. %
We replace the mirror descent over Bernoulli masks by a standard learning procedure over a deterministic model, which takes as input the global model estimate $\modelest$, computes a gradient $\gradient$ (over $\nlocalepochs$ local epochs), and outputs a distribution $\stochquant[\gradient]$. In stochastic SignSGD, $\stochquant$ transforms each gradient entry $\gradiententry$ to a Bernoulli random variable with parameter $\posteriorentry = 1/(1+\exp(-\gradiententry/\signsgdconst))$ for some $\signsgdconst > 0$, where the random variable takes value $+1$ with probability $\posteriorentry$, and $-1$ otherwise. We then employ $\iscompressionvar{\posterior}{\priorul}$ to obtain samples $\maskul$, where the compression is performed element-wise. We apply this method to \algglobalrand\, where Step 6 is replaced by $\model[\epoch+1] = \model[\epoch] - \lr_s \frac{1}{\nclients} \sum_{\client=1}^\nclients \posteriorest$, where $\posteriorest=\frac{1}{\nmasksul} \sum_{\sampleidx=1}^{\nmasksul} \maskul$ and $\lr_s$ is the federator's learning rate. Step 9 is modified accordingly. The priors $\prior$ are chosen to be Bernoulli random variables with parameter $0.5$. We will refer to this method as \algglobalrandcfl.
}

We study the setting of $\nclients=10$ clients collaboratively training a convolutional neural network (CNN)-based classifier for the datasets MNIST, Fashion-MNIST and CIFAR-10 under the orchestration of a federator. For MNIST, we use two different models, LeNet-5 \citep{lecun1998gradient} and a $4$-layer convolutional neural network (4CNN) proposed by \citet{ramanujan2020whats}. The latter is also used to train on Fashion MNIST. For CIFAR-10, we use a larger neural network with $6$ convolutional layers (6CNN). We train MNIST and Fashion-MNIST for $200$ global iterations and CIFAR-10 for $400$ global iterations. Through all experiments and datasets, we carry $\nlocalepochs=3$ local iterations per client per global iteration. We evaluate the performance of the schemes in two different settings: with uniform data allocation (\iid) to model homogeneous systems and a \noniid\ setting to model heterogeneous systems, where data allocation for each client is drawn from a Dirichlet distribution with parameter $\alpha=0.1$. This is considered a rather challenging regime due to high-class imbalance. Every result shows the average across three simulation runs with different seeds. Details on the simulation setup and the network architectures are deferred to \cref{app:sim_setup}. Consistently throughout all experiments, our proposed methods provide order-wise improvements in the communication cost while achieving state-of-the art accuracies.

We plot in \cref{fig:mnist4cnn_uniform} the test accuracies for all the schemes as a function of the total communication cost in bits per parameter \rev{and per global iteration}. While all the schemes achieve approximately the same maximum test accuracy, \algglobalrand\ and \alglocalrand\ require substantially less communication. Hence, when the bandwidths of uplink and downlink transmissions are limited, both variations of the proposed method achieve better test accuracies. Turning our focus to the different variations of our scheme, it can be observed that, without partitioning the model for downlink compression, $\alglocalrand$ convergences significantly slower than $\algglobalrand$ for any block allocation method. This highlights the intuition above that the additional MRC step in downlink incurs further noise, which reduces the convergence speed. However, when we partition the model in the downlink and only send disjoint parts to each client through MRC (\alglocalrand-Fixed-SplitDL), the downlink communication cost reduces by a factor of $\nclients$. In the regime of Fashion MNIST with uniform data allocation, this comes without performance degradation, and is hence the method of choice in this regime. We additionally simulated \algglobalrand\ with the suboptimal implementation (\algglobalrand-Reconst-Fixed), in which the federator first reconstructs the global model, and then performs an additional MRC step for downlink transmission. This naturally reduces the convergence speed per iteration without gains in the communication cost. Hence, justifying the choice of \algglobalrand. %
\rev{We show that, in conventional FL, \algglobalrandcfl\ substantially reduces the communication cost without loss in performance. In some cases, especially for \noniid\ data, we even observe improved performance, which we attribute to implicit regularization. %
Note that \algglobalrandcfl\ provides improvements even without error-feedback or momentum. However, our method is fully compatible with such techniques, and can be used as a plug-in approach to further minimize the communication cost in many existing schemes. We study the convergence in \cref{sec:theory}.}

We plot in \cref{fig:acc_bitrate_fmnist_iid} the average bitrate of each scheme over the maximum test accuracy for MNIST and 4CNN. 
The average bitrate is reduced by more than a factor of $1000$ compared to FedAvg, and more than a \textbf{factor of $\mathbf{32}$} compared to \textsc{DoubleSqueeze}, $\textsc{Neolithic}$ and \textsc{Liec}, which perform best among the conventional bi-directional compression methods.

We perform the same study for \noniid data allocation according to a Dirichlet distribution with parameter $\alpha=0.1$, and show the maximum test accuracies over the average bitrate in \cref{fig:acc_bitrate_fmnist_noniid}. It can be found that partitioning the model in \alglocalrand\ worsens the final accuracy of the model. While the model converges faster, it does not achieve the same accuracies as \algglobalrand\ and \alglocalrand\ without partitioning. This hints towards hybrid schemes for \alglocalrand, where the training begins with partitioning on the downlink, and the scheme later switches to full transmission.

In \cref{fig:acc_bitrate_cifar_iid}, we provide the results for CIFAR-10 and uniform data allocation. \algglobalrand\ and \alglocalrand\ both achieve better results with a bitrate \textbf{smaller by a factor of $\mathbf{5}$} than the best baselines. More detailed numerical results can be found in \cref{app:additional_experiments,sec:hyperparameters}.

The adaptive block allocation (Adaptive) of \citet{isik2024adaptive} saves communication costs in many settings and provides better performance than the fixed block allocation (Fixed), due to more accurate MRC tailored to the exact divergences. The proposed low complexity adaptive strategy based on the average KL-divergence (Adaptive-Avg) per block can additionally save in communication (and computation) with no or little performance degradation. We refer the reader to \cref{app:additional_experiments} for further extensive experiments\rev{, graphs for accuracies over epochs, separate studies of uplink and downlink costs, and comparisons for the case of an available broadcast channel from federator to the clients. Further, we refer to \cref{sec:hyperparameters} for various ablation studies analyzing the sensitivity of \alg\ with respect to the choices of the priors, $\nclients$, $\nmasksdl$, $\nisamples$, and the block size $\dimension/\nblocks$.}

\section{Theoretical Results} \label{sec:theory}

\rev{\textbf{Convergence.} In stochastic FL, the exact dynamics of the system over time are challenging to analyze due to the round-dependent interplay of the learning procedure with the transmission noise. However, when using \alg\ for conventional FL with stochastic quantization (cf. \algglobalrandcfl), convergence guarantees can be given. For a comprehensive understanding, we prove the convergence for a general and widely used class of stochastic quantizers $\stochquant$, which are natively unbiased. %
$\stochquant$ takes as input the entry $\gradientvecentry$ of a gradient vector $\gradientvec \in \mathbb{R}^\dimension$  and operates as follows. Let $\nqints$ be the number of quantization intervals, and let $0 \leq \qchoice < \nqints$ be an integer such that $\frac{\qchoice}{\nqints} \leq \frac{\vert \gradientvecentry\vert}{\Vert \gradientvec \Vert} \leq \frac{\qchoice+1}{\nqints}$, then $\stochquant[\gradientvecentry]$ outputs $\Vert \gradientvec \Vert \cdot \text{sign}(\gradientvecentry) (\qchoice+1) / \nqints$ with probability $\frac{\vert \gradientvecentry\vert}{\Vert \gradientvec \Vert} \nqints - \qchoice$, and $\Vert \gradientvec \Vert \cdot \text{sign}(\gradientvecentry) \qchoice/\nqints$ otherwise. $\stochquant$ is unbiased, i.e., $\mathbb{E}[\stochquant[\mathbf{x}]] = \mathbf{x}$, and its variance satisfies $\mathbb{E}[\Vert\stochquant[\mathbf{x}] - \mathbf{x} \Vert^2] \leq \min\{\dimension/\nqints^2, \sqrt{\dimension}/\nqints\} \Vert \mathbf{x} \Vert_2^2$ \citep{alistarh2017qsgd}.}

\rev{Replacing stochastic SignSGD by $\stochquant$ in \algglobalrandcfl, the posterior is given by a Bernoulli distribution with parameter %
$\posteriorentry = \frac{\vert \gradiententry\vert}{\Vert \gradient \Vert} \nqints - \qchoice$. The values $\norm{\gradientvec}$, $\text{sign}(\gradientvec)$, and $\qchoice$ can be encoded independently, e.g., using Elias coding. %
With a slight abuse of notation, let $\iscompressionvar{\stochquant}{\cdot}$ denote the composition of $\stochquant$ and MRC with $\nisamples$ samples per entry. The compression $\iscompressionvar{\stochquant[\gradient]}{\cdot}$ takes a gradient $\gradient$ and outputs samples from a distribution close to $\stochquant[\gradient]$, and falls in the class of biased compressors. We can prove the following contraction property for $\iscompressionvar{\stochquant}{\cdot}$, which will facilitate convergence analysis. A prominent biased contractive compressor is TopK.}
\rev{
\begin{lemma} \label{lemma:contraction}
For any $\mathbf{x} \in \mathbb{R}^\dimension$ and corresponding posterior $\posteriorgen$ following $\stochquant[\mathbf{x}]$, and a prior $\priorgen \in [0,1]^\dimension$, let $\bar{\Delta} \define \max_{\entry \in [\dimension]} \frac{\posteriorentrytmp}{\priorentry} - \frac{1-\posteriorentrytmp}{1-\priorentry}$, $\bar{\Delta}^\prime \define \max_{\entry \in [\dimension]} \posteriorentrytmp \left(\frac{\priorentry}{\posteriorentrytmp} + \frac{1-\priorentry}{1-\posteriorentrytmp}\right)$, and $\bar{p} \define \max_{\entry \in [\dimension]} \priorentry$. The compressor $\iscompression[\stochquant]$ satisfies the following contraction property for $\nisamples = \mathcal{O}(\max\{\sqrt{2\bar{\Delta}^\prime}, \log(6 \bar{p} (\bar{\Delta} + \bar{\Delta}^2)) \sqrt{6 \bar{p} (\bar{\Delta} + \bar{\Delta}^2)}\})$ and $s \geq \sqrt{2\dimension}$:
\begin{align*}
\mathbb{E}[\Vert \iscompression[\stochquant[\mathbf{x}]]- \mathbf{x} \Vert^2] \leq (1-\delta) \Vert \mathbf{x} \Vert^2,
\end{align*}
for $\delta = 1-\frac{\dimension}{\nqints^2} \left(1+\frac{\bar{\Delta}^\prime}{\nisamples^2} + \mathcal{O}\left((\bar{\Delta} + \bar{\Delta}^2) \sqrt{\frac{6 \bar{p} \log\left(2\nisamples\right)}{\nisamples}}\right)\right)$.
\end{lemma}
The underlying core result is a refinement of the MRC analysis, cf. \cref{lemma:qdiv} (\cref{app:proofs}). Hence, for sufficiently large $\nisamples$, the compressor $\iscompressionvar{\stochquant}{\cdot}$ can be used as an alternative to common compressors such as $\stochquant$. The use of MRC introduces a bias into the otherwise unbiased stochastic quantization. 
Based on the contraction property in \cref{lemma:contraction}, standard convergence results follow easily, cf. \cref{sec:convergence_analysis}. %
}

\rev{\textbf{Communication Cost.} We continue to analyze the communication cost} in a specific iteration $t$ and comment on the inter-round dependency later. When the latest global model estimate $\modelest$ is chosen as a prior in MRC, the cost of communication on the uplink is mainly determined by how far the model evolves during the client's training, i.e., $\binklinline{\posterior}{\priorul} = \binklinline{\posterior}{\modelest}$. After communicating samples of the posteriors, the federator obtains an estimate $\posteriorest$ for all $\client \in [\nclients]$. The cost of communication on the downlink to client $\client$ is then determined by $\binklinline{\frac{1}{\nclients} \sum_{\client=1}^{\nclients} \posteriorest}{\modelest}$. While $\binklinline{\posterior}{\modelest}$ depends on the progress during client training, the core challenge is to bound the expected KL-divergence of each model estimate $\binklinline{\posteriorest}{\modelest}$ in the presence of potentially different priors, i.e., $\modelest \neq \modelestj, \client \neq j$. %
For each client $\client$, the overall communication cost is in the order of \vspace{-.2cm}
\begin{equation*}
    \nmasksdl \! \exp\!\bigg(\!\!\mathrm{d}_\textrm{KL}\bigg(\frac{1}{\nclients} \! \sum_{\client=1}^{\nclients} \posteriorest \Vert \priordl \bigg)\bigg) + \nmasksul \exp\!\left(\binkl{\posterior}{\priorul}\!\right)\!. \vspace{-.3cm}
\end{equation*}
We will next quantify $\binklinline{\frac{1}{\nclients} \sum_{\client=1}^{\nclients} \posteriorest}{\modelest}$ for the case $\priorul=\priordl$, however, the analysis can be extended to $\priorul \neq \priordl$ by an additional assumption on the divergence between the two priors.

For the theoretical analysis, we focus on the scalar case for a single iteration $t$, where client $\client\in [\nclients]$ has a posterior $\genericposterior_\client$, and the federator and client $\client$ share a common prior $\genericprior_\client$, both are Bernoulli distributions with parameters $q_\client$ and $p_\client$, respectively. In the context of FL, the client locally trains $\genericprior_\client$ and results with $\genericposterior_\client$. According to \citet{chatterjee2018sample} and the multi-client extension of \citet{isik2024adaptive}, the communication cost in the uplink is determined by $\exp(\binkl{\genericposterior_\client}{\genericprior_\client})$. After uplink transmission, the federator obtains an estimate $\hat{q}_\client$ of $q_\client$; and hence, the updated global model is given by $\frac1n \sum_{\client=1}^\nclients \hat{q}_\client$. The communication cost in the downlink for client $\client$ is determined by $\binkl{\frac{1}{\nclients} \sum_{\client=1}^\nclients \hat{q}_\client}{p_\client}$. Our theoretical contribution is  a new high probability upper bound on this quantity, which refines previous  MRC analysis, for the special case of Bernoulli distributions. 
Let $X$ be a Bernoulli sample obtained through MRC. 
As an initial step, we derive an upper bound on the difference between $q_i$ and the probability $\Pr(X)=1$ that the samples are drawn from, which vanishes when $p_\client = q_\client$ (and hence $\binkl{q_\client}{p_\client} = 0$). We note that the bound of \citet[Theorem 1.1]{chatterjee2018sample} does not saitsfy this natural property. We formally state the result in \cref{prop:loose_bound} in \cref{app:proofs}, which, however, does not yet capture the dependency on the number of samples $\nisamples$ used in MRC to sample an index. We refine \cref{prop:loose_bound} with \cref{lemma:qdiv} (cf. \cref{app:proofs}), which additionally captures this dependency, and will allow us to derive an upper bound on $\binkl{\frac1n \sum_{\client=1}^\nclients \hat{q}_\client}{p_\client}$. \cref{lemma:qdiv}  is of independent interest and can be seen as a refinement of the analysis by \citet{chatterjee2018sample} for Bernoulli distributions. It is required to prove \cref{thm:downlink_kl}.

\newcommand{\priordiv}{\ensuremath{\zeta}}
For the statement of the following theorem, we assume that the progress by one local client training is bounded by $\vert q_j - p_j \vert \leq \klball$ for all $j \in [\nclients]$. Using Pinsker's inequality to bound $\vert q_j - p_j \vert \leq \frac{1}{2}\sqrt{\binkl{q_j}{p_j}/2}$, this is a natural assumption given from the KL-proximity term of mirror descent (for one local iteration), and can be strictly enforced through the projection of $q_j$ onto a KL ball around $p_j$ of fixed divergence. We assume that the difference between the clients' priors, i.e., their global model estimates in our algorithms, are bounded as $\vert p_\client - p_j\vert \leq \priordiv$ for all $\client, j \in [\nclients]$.
\begin{theorem} \label{thm:downlink_kl}
Assume $p_j \! > \! \priordiv$ for all $j \! \in \! [\nclients]$, for $\Delta_j \! \define \! \frac{q_j}{p_j-\priordiv} - \frac{1-q_j}{1-p_j+\priordiv}$ and $\Delta^\prime_j \define \! q_j\big(\frac{p_j+\priordiv}{q_j} + \frac{1-p_j+\priordiv}{1-q_j}\big)$, with probability $1-\delta^\prime$, the global model divergence $\binklinline{\frac{1}{\nclients} \sum_{j=1}^\nclients \hat{q}_j}{p_\client}$ is upper bounded by \vspace{-.2cm}
\begin{align*}
\sum_{j=1}^\nclients \frac{2}{\nclients \min\{p_\client, 1-p_\client\}} \Bigg(\frac{\Delta^\prime_j}{\nisamples^2} + \! + \! \sqrt{\frac{\ln(2/\delta^\prime)}{2 \nmasksul}} \! + \! \klball \! + \! \priordiv^2 +\!\!\!\!\\
+\mathcal{O}\bigg((\Delta_j + \Delta^2_j) \sqrt{\frac{6 (p_\client+\priordiv) \log\left(2\nisamples\right)}{\nisamples}}\bigg)\Bigg)
\end{align*} \vspace{-.5cm}
\end{theorem}
By \citet{chatterjee2018sample}, this provides an immediate bound on the cost of downlink transmission. The bound applies to both algorithms \alglocalrand\ and \algglobalrand. However, when all priors $p_j$ are the same (such as in \algglobalrand-Reconst), i.e., $\priordiv=0$, the bound simplifies accordingly. The explicit dependency on the factor $1/\sqrt{\nmasksul}$ reflects the interplay between uplink and downlink cost. The parameter $\priordiv$ gives rise to an inter-round dependency of the communication cost. The more accurate the estimation of the global model in the previous iteration (given the priors are chosen as $\modelest$), the smaller $\priordiv$, and hence the lower the transmission cost in the subsequent iteration. The proofs of \cref{prop:loose_bound,lemma:qdiv,thm:downlink_kl} can be found in \cref{app:proofs}.

\section{Related Work}

Followed by the introduction of FL by \citet{mcmahan2017communication}, lossy compression of gradients or model updates has been a long studied narrative in FL, with prominent representatives such as SignSGD, also known as 1-bit Stochastic Gradient Descent (SGD) \citep{seide2014onebit}, QSGD \citep{alistarh2017qsgd}, TernGrad \citep{wen2017terngrad}, SignSGD with error feedback \citep{karimireddy2019error}, vector-quantized SGD \citep{gandikota2021vqsgd} and natural compression \citep{horvoth2022natural}. Such methods retain satisfactory final model accuracy even with aggressive quantization. Sparsification-based methods have also been considered as alternatives, e.g., TopK \citep{wangni2018gradient, shi2019distributed}. The importance of bi-directional gradient compression in many settings was outlined by \citet{philippenko2020bidirectional}. Many schemes were proposed that leverage combinations of gradient compression in the uplink and downlink, error-feedback, and momentum, e.g., {Mem-SGD} \citep{stich2018sparsified}, {DoubleSqueeze} \citep{tang2019doublesqueeze}, block-wise SignSGD with momentum \citep{zheng2019communication}, communication-efficient SGD with error reset ({Cser}) \citep{xie2020cser}, \rev{Artemis \citep{philippenko2020bidirectional}}, {Neolithic} \citep{huang2022lower}, \textsc{DoCoFL} \citep{dorfman2023docofl}, EF21-P and friends \citep{gruntkowska2023ef21}, 2Direction \citep{tyurin2023direction}, M3 \citep{gruntkowska2024improving}, and LIEC \citep{cheng2024communication}. With the exception of the methods MCM \citep{philippenko2021preserved} and M3 \citep{gruntkowska2024improving}, each client receives the same broadcast, potentially compressed, global gradient or model update. %
\citet{isik2024adaptive} studied uplink compression for stochastic FL and showed significant communication reduction with competitive performance. Their framework, termed KLMS, applies to a variety of stochastic compressors and to Bayesian FL settings, e.g., QLSD \cite{vono2022qlsd}. The compression is based on importance sampling and MRC, thoroughly studied by \citet{chatterjee2018sample} and \citet{havasi2018minimal}. Such methods, known as relative entropy coding, have been used in FL in conjunction with differential privacy, cf. DP-REC \citep{triastcyn2022dprec}.

Since the lottery ticket hypothesis \citep{frankle2018lottery}, finding sparse subnetworks of neural networks that achieve satisfactory accuracy was investigated. \citet{ramanujan2020whats} showed that randomly weighted networks contain suitable subnetworks of large neural networks capable of achieving competitive performance. \citet{isik2023sparse} formulated a probabilistic method of training neural network masks collaboratively in an FL context.

\section{Conclusion}
\vspace{-.2cm}
We illuminated the problem of bi-directional compression in stochastic FL using the specific instance of federated probabilistic mask training\rev{, which we showed to inherently optimize both the learning objective and the communication costs.} By leveraging side-information through carefully chosen prior distributions, the total communication costs can be reduced by factors between $5$ and $32$ compared to non-stochastic FL baselines while achieving state-of-the-art accuracies on classification tasks in both homogeneous and heterogeneous FL regimes. We thereby close the gap of downlink compression for stochastic FL and complement the existing literature on bi-directional compression for standard FL. \rev{Applying our methods to stochastic quantization in conventional FL, we paved the way to convergence analysis for MRC-based compression.} Allowing different priors among all clients, this work opens the door to studying compression under side-information in \emph{decentralized stochastic FL}, %
where a central coordinator is missing. 
Our theoretical results are of independent interest and may be applied in various scenarios where MRC is used.

\section*{Impact Statement}

This paper presents work whose goal is to advance the field of 
Machine Learning. There are many potential societal consequences 
of our work, none which we feel must be specifically highlighted here.

\bibliography{refs}

\begin{thebibliography}{39}
\providecommand{\natexlab}[1]{#1}
\providecommand{\url}[1]{\texttt{#1}}
\expandafter\ifx\csname urlstyle\endcsname\relax
  \providecommand{\doi}[1]{doi: #1}\else
  \providecommand{\doi}{doi: \begingroup \urlstyle{rm}\Url}\fi

\bibitem[Abdi \& Fekri(2019)Abdi and Fekri]{Abdi:arXiv:19}
Afshin Abdi and Faramarz Fekri.
\newblock Nested dithered quantization for communication reduction in distributed training.
\newblock \emph{arXiv preprint arXiv:1904.01197}, 2019.

\bibitem[Alistarh et~al.(2017)Alistarh, Grubic, Li, Tomioka, and Vojnovic]{alistarh2017qsgd}
Dan Alistarh, Demjan Grubic, Jerry Li, Ryota Tomioka, and Milan Vojnovic.
\newblock {{QSGD}}: {{Communication-efficient SGD}} via gradient quantization and encoding.
\newblock In \emph{Advances in Neural Information Processing Systems}, volume~30, 2017.

\bibitem[Amiri et~al.(2020)Amiri, Gunduz, Kulkarni, and Poor]{Amiri:arXiv:20}
Mohammad~Mohammadi Amiri, Deniz Gunduz, Sanjeev~R. Kulkarni, and H.~Vincent Poor.
\newblock Federated learning with quantized global model updates.
\newblock \emph{arXiv preprint arXiv:2006.10672}, 2020.

\bibitem[Chatterjee \& Diaconis(2018)Chatterjee and Diaconis]{chatterjee2018sample}
Sourav Chatterjee and Persi Diaconis.
\newblock The sample size required in importance sampling.
\newblock \emph{The Annals of Applied Probability}, 28\penalty0 (2):\penalty0 1099--1135, 2018.

\bibitem[Cheng et~al.(2024)Cheng, Shen, Xu, Qian, Wu, Zhou, Zhang, Tao, and Chen]{cheng2024communication}
Yifei Cheng, Li~Shen, Linli Xu, Xun Qian, Shiwei Wu, Yiming Zhou, Tie Zhang, Dacheng Tao, and Enhong Chen.
\newblock Communication-efficient distributed learning with local immediate error compensation.
\newblock \emph{arXiv preprint arXiv:2402.11857}, 2024.

\bibitem[Cover \& Thomas(2006)Cover and Thomas]{thomas2006elements}
Thomas Cover and Joy~A Thomas.
\newblock \emph{Elements of information theory}.
\newblock Wiley-Interscience, 2006.

\bibitem[Dorfman et~al.(2023)Dorfman, Vargaftik, Ben-Itzhak, and Levy]{dorfman2023docofl}
Ron Dorfman, Shay Vargaftik, Yaniv Ben-Itzhak, and Kfir~Yehuda Levy.
\newblock {DoCoFL}: Downlink compression for cross-device federated learning.
\newblock In \emph{International Conference on Machine Learning}, pp.\  8356--8388, 2023.

\bibitem[Frankle \& Carbin(2019)Frankle and Carbin]{frankle2018lottery}
Jonathan Frankle and Michael Carbin.
\newblock The lottery ticket hypothesis: Finding sparse, trainable neural networks.
\newblock In \emph{International Conference on Learning Representations}, 2019.

\bibitem[Gandikota et~al.(2021)Gandikota, Kane, Kumar~Maity, and Mazumdar]{gandikota2021vqsgd}
Venkata Gandikota, Daniel Kane, Raj Kumar~Maity, and Arya Mazumdar.
\newblock {{vqSGD}}: {{Vector}} quantized stochastic gradient descent.
\newblock In \emph{International Conference on Artificial Intelligence and Statistics}, volume 130, pp.\  2197--2205, 2021.

\bibitem[Gruntkowska et~al.(2023)Gruntkowska, Tyurin, and Richt{\'a}rik]{gruntkowska2023ef21}
Kaja Gruntkowska, Alexander Tyurin, and Peter Richt{\'a}rik.
\newblock {EF21-P} and friends: Improved theoretical communication complexity for distributed optimization with bidirectional compression.
\newblock In \emph{International Conference on Machine Learning}, pp.\  11761--11807, 2023.

\bibitem[Gruntkowska et~al.(2024)Gruntkowska, Tyurin, and Richt{\'a}rik]{gruntkowska2024improving}
Kaja Gruntkowska, Alexander Tyurin, and Peter Richt{\'a}rik.
\newblock Improving the worst-case bidirectional communication complexity for nonconvex distributed optimization under function similarity.
\newblock \emph{arXiv preprint arXiv:2402.06412}, 2024.

\bibitem[Havasi et~al.(2019)Havasi, Peharz, and Hernández-Lobato]{havasi2018minimal}
Marton Havasi, Robert Peharz, and José~Miguel Hernández-Lobato.
\newblock Minimal random code learning: Getting bits back from compressed model parameters.
\newblock In \emph{International Conference on Learning Representations}, 2019.

\bibitem[Horv{\'o}th et~al.(2022)Horv{\'o}th, Ho, Horvath, Sahu, Canini, and Richtarik]{horvoth2022natural}
Samuel Horv{\'o}th, Chen-Yu Ho, Ludovit Horvath, Atal~Narayan Sahu, Marco Canini, and Peter Richtarik.
\newblock Natural compression for distributed deep learning.
\newblock In \emph{Proceedings of Mathematical and Scientific Machine Learning}, volume 190, pp.\  129--141, 2022.

\bibitem[Huang et~al.(2022)Huang, Chen, Yin, and Yuan]{huang2022lower}
Xinmeng Huang, Yiming Chen, Wotao Yin, and Kun Yuan.
\newblock Lower bounds and nearly optimal algorithms in distributed learning with communication compression.
\newblock \emph{Advances in Neural Information Processing Systems}, 35:\penalty0 18955--18969, 2022.

\bibitem[Isik et~al.(2023)Isik, Pase, Gunduz, Weissman, and Michele]{isik2023sparse}
Berivan Isik, Francesco Pase, Deniz Gunduz, Tsachy Weissman, and Zorzi Michele.
\newblock Sparse random networks for communication-efficient federated learning.
\newblock In \emph{International Conference on Learning Representations}, 2023.

\bibitem[Isik et~al.(2024)Isik, Pase, Gunduz, Koyejo, Weissman, and Zorzi]{isik2024adaptive}
Berivan Isik, Francesco Pase, Deniz Gunduz, Sanmi Koyejo, Tsachy Weissman, and Michele Zorzi.
\newblock Adaptive compression in federated learning via side information.
\newblock In \emph{International Conference on Artificial Intelligence and Statistics}, pp.\  487--495, 2024.

\bibitem[Karimireddy et~al.(2019)Karimireddy, Rebjock, Stich, and Jaggi]{karimireddy2019error}
Sai~Praneeth Karimireddy, Quentin Rebjock, Sebastian Stich, and Martin Jaggi.
\newblock Error feedback fixes {{SignSGD}} and other gradient compression schemes.
\newblock In \emph{International Conference on Machine Learning}, volume~97, pp.\  3252--3261, 2019.

\bibitem[Kingma \& Ba(2015)Kingma and Ba]{kingma2014adam}
Diederik~P Kingma and Jimmy Ba.
\newblock Adam: A method for stochastic optimization.
\newblock In \emph{International Conference on Learning Representations}, 2015.

\bibitem[Lecun et~al.(1998)Lecun, Bottou, Bengio, and Haffner]{lecun1998gradient}
Y.~Lecun, L.~Bottou, Y.~Bengio, and P.~Haffner.
\newblock Gradient-based learning applied to document recognition.
\newblock \emph{Proceedings of the IEEE}, 86\penalty0 (11):\penalty0 2278--2324, 1998.

\bibitem[McMahan et~al.(2017)McMahan, Moore, Ramage, Hampson, and Arcas]{mcmahan2017communication}
Brendan McMahan, Eider Moore, Daniel Ramage, Seth Hampson, and Blaise Aguera~y Arcas.
\newblock {Communication-Efficient Learning of Deep Networks from Decentralized Data}.
\newblock In \emph{International Conference on Artificial Intelligence and Statistics}, volume~54, pp.\  1273--1282, 2017.

\bibitem[Philippenko \& Dieuleveut(2020)Philippenko and Dieuleveut]{philippenko2020bidirectional}
Constantin Philippenko and Aymeric Dieuleveut.
\newblock Bidirectional compression in heterogeneous settings for distributed or federated learning with partial participation: tight convergence guarantees.
\newblock \emph{arXiv preprint arXiv:2006.14591}, 2020.

\bibitem[Philippenko \& Dieuleveut(2021)Philippenko and Dieuleveut]{philippenko2021preserved}
Constantin Philippenko and Aymeric Dieuleveut.
\newblock Preserved central model for faster bidirectional compression in distributed settings.
\newblock \emph{Advances in Neural Information Processing Systems}, 34:\penalty0 2387--2399, 2021.

\bibitem[Ramanujan et~al.(2020)Ramanujan, Wortsman, Kembhavi, Farhadi, and Rastegari]{ramanujan2020whats}
Vivek Ramanujan, Mitchell Wortsman, Aniruddha Kembhavi, Ali Farhadi, and Mohammad Rastegari.
\newblock What's hidden in a randomly weighted neural network?
\newblock In \emph{IEEE/CVF conference on computer vision and pattern recognition}, pp.\  11893--11902, 2020.

\bibitem[Richt{\'a}rik et~al.(2021)Richt{\'a}rik, Sokolov, and Fatkhullin]{richtarik2021ef21}
Peter Richt{\'a}rik, Igor Sokolov, and Ilyas Fatkhullin.
\newblock Ef21: A new, simpler, theoretically better, and practically faster error feedback.
\newblock \emph{Advances in Neural Information Processing Systems}, 34:\penalty0 4384--4396, 2021.

\bibitem[Seide et~al.(2014)Seide, Fu, Droppo, Li, and Yu]{seide2014onebit}
Frank Seide, Hao Fu, Jasha Droppo, Gang Li, and Dong Yu.
\newblock 1-bit stochastic gradient descent and its application to data-parallel distributed training of speech {{DNNs}}.
\newblock In \emph{Interspeech}, pp.\  1058--1062, 2014.

\bibitem[Shi et~al.(2019)Shi, Wang, Zhao, Tang, Wang, Huang, and Chu]{shi2019distributed}
Shaohuai Shi, Qiang Wang, Kaiyong Zhao, Zhenheng Tang, Yuxin Wang, Xiang Huang, and Xiaowen Chu.
\newblock A distributed synchronous {SGD} algorithm with global top-k sparsification for low bandwidth networks.
\newblock In \emph{International Conference on Distributed Computing Systems (ICDCS)}, pp.\  2238--2247, 2019.

\bibitem[Srinivasan(2002)]{srinivasan2002importance}
Rajan Srinivasan.
\newblock \emph{Importance sampling: Applications in communications and detection}.
\newblock Springer Science \& Business Media, 2002.

\bibitem[Stich et~al.(2018)Stich, Cordonnier, and Jaggi]{stich2018sparsified}
Sebastian~U Stich, Jean-Baptiste Cordonnier, and Martin Jaggi.
\newblock Sparsified {SGD} with memory.
\newblock \emph{Advances in Neural Information Processing Systems}, 31, 2018.

\bibitem[Tang et~al.(2019)Tang, Yu, Lian, Zhang, and Liu]{tang2019doublesqueeze}
Hanlin Tang, Chen Yu, Xiangru Lian, Tong Zhang, and Ji~Liu.
\newblock Doublesqueeze: Parallel stochastic gradient descent with double-pass error-compensated compression.
\newblock In \emph{International Conference on Machine Learning}, pp.\  6155--6165, 2019.

\bibitem[Triastcyn et~al.(2022)Triastcyn, Reisser, and Louizos]{triastcyn2022dprec}
Aleksei Triastcyn, Matthias Reisser, and Christos Louizos.
\newblock {DP}-{REC}: Private \& communication-efficient federated learning, 2022.

\bibitem[Tyurin \& Richt{\'a}rik(2023)Tyurin and Richt{\'a}rik]{tyurin2023direction}
Alexander Tyurin and Peter Richt{\'a}rik.
\newblock {2Direction:} theoretically faster distributed training with bidirectional communication compression.
\newblock In \emph{Conference on Neural Information Processing Systems}, 2023.

\bibitem[Vono et~al.(2022)Vono, Plassier, Durmus, Dieuleveut, and Moulines]{vono2022qlsd}
Maxime Vono, Vincent Plassier, Alain Durmus, Aymeric Dieuleveut, and Eric Moulines.
\newblock {QLSD:} quantised langevin stochastic dynamics for bayesian federated learning.
\newblock In \emph{International Conference on Artificial Intelligence and Statistics}, pp.\  6459--6500, 2022.

\bibitem[Wangni et~al.(2018)Wangni, Wang, Liu, and Zhang]{wangni2018gradient}
Jianqiao Wangni, Jialei Wang, Ji~Liu, and Tong Zhang.
\newblock Gradient sparsification for communication-efficient distributed optimization.
\newblock \emph{Advances in Neural Information Processing Systems}, 31, 2018.

\bibitem[Weinberger \& Yemini(2023)Weinberger and Yemini]{weinberger2023multi}
Nir Weinberger and Michal Yemini.
\newblock Multi-armed bandits with self-information rewards.
\newblock \emph{IEEE Transactions on Information Theory}, 2023.

\bibitem[Wen et~al.(2023)Wen, Zhang, Lan, Cui, Cai, and Zhang]{wen2023survey}
Jie Wen, Zhixia Zhang, Yang Lan, Zhihua Cui, Jianghui Cai, and Wensheng Zhang.
\newblock A survey on federated learning: challenges and applications.
\newblock \emph{International Journal of Machine Learning and Cybernetics}, 14\penalty0 (2):\penalty0 513--535, 2023.

\bibitem[Wen et~al.(2017)Wen, Xu, Yan, Wu, Wang, Chen, and Li]{wen2017terngrad}
Wei Wen, Cong Xu, Feng Yan, Chunpeng Wu, Yandan Wang, Yiran Chen, and Hai Li.
\newblock {{TernGrad}}: {{Ternary}} gradients to reduce communication in distributed deep learning.
\newblock In \emph{Advances in Neural Information Processing Systems}, volume~30, 2017.

\bibitem[Xie et~al.(2020)Xie, Zheng, Koyejo, Gupta, Li, and Lin]{xie2020cser}
Cong Xie, Shuai Zheng, Sanmi Koyejo, Indranil Gupta, Mu~Li, and Haibin Lin.
\newblock Cser: Communication-efficient sgd with error reset.
\newblock \emph{Advances in Neural Information Processing Systems}, 33:\penalty0 12593--12603, 2020.

\bibitem[Zhang et~al.(2021)Zhang, Xie, Bai, Yu, Li, and Gao]{zhang2021survey}
Chen Zhang, Yu~Xie, Hang Bai, Bin Yu, Weihong Li, and Yuan Gao.
\newblock A survey on federated learning.
\newblock \emph{Knowledge-Based Systems}, 216:\penalty0 106775, 2021.

\bibitem[Zheng et~al.(2019)Zheng, Huang, and Kwok]{zheng2019communication}
Shuai Zheng, Ziyue Huang, and James Kwok.
\newblock Communication-efficient distributed blockwise momentum {SGD} with error-feedback.
\newblock \emph{Advances in Neural Information Processing Systems}, 32, 2019.

\end{thebibliography}
\bibliographystyle{iclr2025_conference}

\appendix
\newpage

\onecolumn

\section{Reproducibility}
In addition to the algorithmic details and the clients' training procedure function (cf. \cref{alg:local_training,alg:globalrand,alg:localrand}), we provide in \cref{sec:experiments} the most important hyperparameters used in our experiments, such as local and global iterations, and data allocation. Further parameter information, such as batch size, learning rates and the choice of the optimizer can be found in \cref{app:additional_experiments}, together with details on the neural network architectures and the hardware cluster used for running the experiments. Particularities of the block allocation required for the operation of our schemes are described in \cref{app:block_allocation}. All assumptions required for the theoretical analysis are stated in \cref{sec:theory}. Full proofs of all claims, including formal statements, can be found in \cref{app:proofs}.

\section{Proofs and Intermediate Results} \label{app:proofs}

In the following, we provide the formal statements of \cref{prop:loose_bound,lemma:qdiv} including their proofs. Parts of the proof of \cref{prop:loose_bound} will be used to prove \cref{lemma:qdiv}. We prove \cref{thm:downlink_kl} afterward.

\begin{proposition} \label{prop:loose_bound}
For a sample $X_\sampleidx$ transmitted by MRC with posterior and prior Bernoulli distributions with parameters $q$ and $p$, we have
\begin{align*}
     \vert \Pr(X_\sampleidx = 1) - q \vert &\leq q \left(\max\left\{\frac{p}{q}, \frac{1-p}{1-q}, \frac{q}{p}, \frac{1-q}{1-p}\right\} - 1\right).
\end{align*}
\end{proposition}

\begin{proof}[Proof of \cref{prop:loose_bound}]
Assume a party wants to sample from a Bernoulli distribution $\genericposterior$ with parameter $q$, which is held by another party. Both parties share a common prior $\genericprior$ in the form of a Bernoulli distribution with parameter $p$ and have access to shared randomness. Fix any sample index $\sampleidx$ for the moment (this index will be needed for the proof of \cref{thm:downlink_kl}). Both parties sample $\ngenericsamples \nisamples$ i.i.d. samples $\isampleall \sim \genericprior$ for $\isampleidx \in [\nisamples]$ independently and identically from $\genericprior$. The party holding $\genericposterior$ constructs an auxiliary distribution $$W_\sampleidx(\isampleidx) = \frac{\genericposterior(\isampleall)/\genericprior(\isampleall)}{ \sum_{\isampleidx=1}^{\nisamples} \genericposterior(\isampleall)/\genericprior(\isampleall)},$$ from which it samples to obtain an index $I_\sampleidx$. The index is transmitted to the other party, which reconstructs the corresponding sample $\isamplesel$.

To bound the difference $\vert \Pr(X_\sampleidx = 1) - q \vert$, i.e., the target Bernoulli parameter compared to the parameter which the sample is drawn from, by the independence of the samples $\isamplesel$ for different $\sampleidx$, we focus on a single sample $\sampleidx \in [\ngenericsamples]$, for which it holds that
\begin{align*}
    &\Pr(\isamplesel=1) \\
    &= \sum_{\isampleidx=1}^{\nisamples} \sum_{\{\tmpsample[1], \dots, \tmpsample[\nisamples]:\tmpsample[i]=\isampleidx\}}  \!\!\! \Pr(\isampleall[1]=\tmpsample[1], \dots, \isampleall[\nisamples]=\tmpsample[\nisamples]) \Pr(I_\sampleidx=\isampleidx \mid \isampleall[1]=\tmpsample[1], \dots, \isampleall[\nisamples]=\tmpsample[\nisamples]) \\
    & \overset{(a)}{=} \nisamples \sum_{\{\tmpsample[2], \dots, \tmpsample[\nisamples]\}} \Pr(\isampleall[1]=1, \isampleall[2]=\tmpsample[2], \dots, \isampleall[\nisamples]=\tmpsample[\nisamples]) \\
    & \hspace{2in} \cdot \Pr(I_\sampleidx=1 \mid \isampleall[1]=1, \isampleall[2]=\tmpsample[2], \dots, \isampleall[\nisamples]=\tmpsample[\nisamples]) \\
    & \overset{(b)}{=} \nisamples \sum_{\bsample=0}^{\nisamples-1} \sum_{\{\tmpsample[2], \dots, \tmpsample[\nisamples]: \sum_{\isampleidx=2}^{\nisamples}=\bsample\}} \Pr(\isampleall[1]=1, \isampleall[2]=\tmpsample[2], \dots, \isampleall[\nisamples]=\tmpsample[\nisamples])  \\
    & \hspace{2in} \cdot \Pr(I_\sampleidx=1 \mid \isampleall[1]=1, \isampleall[2]=\tmpsample[2], \dots, \isampleall[\nisamples]=\tmpsample[\nisamples]),
\end{align*}
where $(a)$ follows from symmetry, $(b)$ follows since by permutation invariance, the inner probability only depends on the number of ones in $\{\tmpsample[2], \dots, \tmpsample[\nisamples]\}$.

The inner probability is given by the distribution $\auxdist_\sampleidx(\isampleidx)$. Given that $\isampleall[1]=1$ and that $\sum_{\isampleidx=2}^{\nisamples} \isampleall[\sampleidx] = \bsample$, it holds that
\begin{align*}
    \sum_{\isampleidx=1}^{\nisamples} \genericposterior(\isampleall)/\genericprior(\isampleall) = (\bsample+1) \cdot \frac{q}{p} + (\nisamples-\bsample-1) \cdot \frac{1-q}{1-p}.
\end{align*}
Hence,
\begin{align*}
    \Pr(I_\sampleidx=1 \mid \isampleall[1]=1, \isampleall[2]=\tmpsample[2], \dots, \isampleall[\nisamples]=\tmpsample[\nisamples]) = \frac{\frac{q}{p}}{(\bsample+1) \cdot \frac{q}{p} + (\nisamples-\bsample-1) \cdot \frac{1-q}{1-p}},
\end{align*}
which is independent of the exact choice of $\{\tmpsample[2], \dots, \tmpsample[\nisamples]\}$ given their sum $\sum_{\isampleidx=2}^{\nisamples} \isampleall = \bsample$. 
Since $\Pr(\isampleall[1]=1, \isampleall[2]=\tmpsample[2], \dots, \isampleall[\nisamples]=\tmpsample[\nisamples]) = p^{\bsample+1} (1-p)^{\nisamples-\bsample-1}$ by the Bernoulli distribution assumption, we have
\begin{align*}
    &\Pr(\isamplesel=1) = \nisamples \sum_{\bsample=0}^{\nisamples-1} \binom{\nisamples-1}{\bsample} p^{\bsample+1} (1-p)^{\nisamples-\bsample-1} \frac{\frac{q}{p}}{(\bsample+1) \cdot \frac{q}{p} + (\nisamples-\bsample-1) \cdot \frac{1-q}{1-p}},
\end{align*}

Defining a binary random variable $\fracrv$ with sample space $\left\{\frac{q}{p}, \frac{1-q}{1-p}\right\}$, for a Bernoulli distribution $\ber{\frac{\bsample+1}{\nisamples}}$ with success probability parameter $\frac{\bsample+1}{\nisamples}$, where a success refers to the outcome $\fracrv=\frac{q}{p}$, we can write that
\begin{align}
    \Pr(\isamplesel = 1) &= q \cdot \sum_{\bsample=0}^{\nisamples-1} \binom{n-1}{\bsample} p^\bsample (1-p)^{\nisamples-\bsample-1} \frac{1}{\frac{\bsample+1}{\nisamples} \frac{q}{p} + \frac{\nisamples-\bsample-1}{\nisamples} \frac{1-q}{1-p}} \nonumber \\
    &= q \cdot \mathbb{E}\left[ \frac{1}{\frac{\bsample+1}{\nisamples} \frac{q}{p} + \frac{\nisamples-\bsample-1}{\nisamples} \frac{1-q}{1-p}} \right] = q \mathbb{E}\left[\frac{1}{\mathbb{E}_{\ber{\frac{\bsample+1}{\nisamples}}}[\fracrv]}\right] \label{eq:success_prob} \\
    &\overset{(a)}{\leq} q \mathbb{E}\left[\mathbb{E}_{\ber{\frac{\bsample+1}{\nisamples}}}\left[\frac{1}{\fracrv}\right]\right], \nonumber
\end{align}
where the outer expectation is over the binomial distribution with $\nisamples-1$ trials and success probability $p$, i.e., $\bsample \sim \text{Binomial}(\nisamples-1, p)$, and where $(a)$ follows from Jensen's inequality over the inner expectation.
Hence,
\begin{align}
    \Pr(\isamplesel = 1) - q &= q \left(\frac{\Pr(\isamplesel = 1)}{q}-1\right) \nonumber \\
    &\leq q \left(\mathbb{E}\left[\mathbb{E}_{\ber{\frac{\bsample+1}{\nisamples}}}\left[\frac{1}{\fracrv}\right]\right]-1\right) \label{eq:q_upper}
\end{align}

Since $\frac{1}{\mathbb{E}_{\ber{\frac{\bsample+1}{\nisamples}}}[\fracrv]} \geq 2-\mathbb{E}_{\ber{\frac{\bsample+1}{\nisamples}}}[\fracrv]$, it also follows from \eqref{eq:success_prob} that
\begin{align*}
    \Pr(\isamplesel = 1) &= q \cdot \mathbb{E}\left[ \frac{1}{\frac{\bsample+1}{\nisamples} \frac{q}{p} + \frac{\nisamples-\bsample-1}{\nisamples} \frac{1-q}{1-p}} \right] = q \mathbb{E}\left[\frac{1}{\mathbb{E}_{\ber{\frac{\bsample+1}{\nisamples}}}[\fracrv]}\right] \\
    &\geq q \mathbb{E}\left[ 2-\mathbb{E}_{\ber{\frac{\bsample+1}{\nisamples}}}[\fracrv] \right],
\end{align*}
from which we have
\begin{align}
     \Pr(\isamplesel = 1) - q \geq q \left(1-\mathbb{E}\left[\mathbb{E}_{\ber{\frac{\bsample+1}{\nisamples}}}\left[\fracrv\right]\right]\right). \label{eq:q_lower}
\end{align}
Combining the upper and lower bound in \eqref{eq:q_upper} and \eqref{eq:q_lower}, respectively, we derive
\begin{align*}
     \vert \Pr(\isamplesel = 1) - q \vert &\leq q \left(\max\left\{\mathbb{E}\left[1-\mathbb{E}_{\ber{\frac{\bsample+1}{\nisamples}}}\left[\fracrv\right]\right], \mathbb{E}\left[\mathbb{E}_{\ber{\frac{\bsample +1}{\nisamples}}}\left[\frac{1}{\fracrv}\right]\right] \right\} - 1\right) \\
     &\leq q \left(\mathbb{E}\left[\max\left\{\mathbb{E}_{\ber{\frac{\bsample +1}{\nisamples}}}\left[\fracrv\right], \mathbb{E}_{\ber{\frac{\bsample +1}{\nisamples}}}\left[\frac{1}{\fracrv}\right]\right\}\right] - 1\right)  \\
     &\leq q \left(\mathbb{E}\left[\mathbb{E}_{\ber{\frac{\bsample +1}{\nisamples}}}\left[\max\left\{\fracrv, \frac{1}{\fracrv}\right\}\right]\right] - 1\right) \\
     &\leq q \left(\mathbb{E}\left[\max\left\{\frac{p}{q}, \frac{1-p}{1-q}, \frac{q}{p}, \frac{1-q}{1-p}\right\}\right] - 1\right)  \\
     &= q \left(\max\left\{\frac{p}{q}, \frac{1-p}{1-q}, \frac{q}{p}, \frac{1-q}{1-p}\right\} - 1\right).
\end{align*}
This concludes the proof.
\end{proof}

\begin{lemma} \label{lemma:qdiv}
    For a sample $X_\sampleidx$ transmitted via MRC with posterior and prior being Bernoulli distributions with parameters $q$ and $p$, $\Delta \define \frac{q}{p} - \frac{1-q}{1-p}$ and $\Delta^\prime \define q \left(\frac{p}{q} + \frac{1-p}{1-q}\right)$, we have
    \begin{align*}
    \vert\Pr(X_\sampleidx = 1) - q\vert \leq \frac{\Delta^\prime}{\nisamples^2} + \mathcal{O}\left((\Delta + \Delta^2) \sqrt{\frac{6 p \log\left(2\nisamples\right)}{\nisamples}}\right).
    \end{align*}
\end{lemma}

\begin{proof}[Proof of \cref{lemma:qdiv}]

The proof starts with the same derivations as for the proof of \cref{prop:loose_bound}, which we follow until \eqref{eq:success_prob} to get
\begin{align*}
    \Pr(\isamplesel = 1) &= q \mathbb{E}\left[\frac{1}{\mathbb{E}_{\ber{\frac{\bsample+1}{\nisamples}}}[\fracrv]}\right]
\end{align*}
Since $\bsample$ is a random quantity that follows a Binomial distribution, we bound $\vert \Pr(\isamplesel = 1) - q\vert$ using a concentration bound on $\bsample$. The relative (multiplicative) Chernoff bound states that
\begin{align*}
    \Pr(\vert \bsample - \varepsilon(\nisamples p)\vert \geq \varepsilon \nisamples p) &= \Pr(\bsample - \varepsilon(\nisamples p) \geq \varepsilon \nisamples p) + \Pr(\bsample - \varepsilon(\nisamples p) \leq -\varepsilon \nisamples p) \\
    &\leq 2\exp\left(-\frac{\varepsilon^2 \nisamples p}{3}\right)
\end{align*}
for any $\varepsilon \in [0,1]$. Setting $\varepsilon = \sqrt{\frac{3\log(2/\delta)}{\nisamples p}}$ implies that
\begin{align*}
    \vert \bsample - \nisamples p \vert &\geq \sqrt{3\nisamples p\log(2/\delta)}
\end{align*}
with probability at most $\delta$. Setting $\delta = \frac{1}{\nisamples^2}$, we obtain for a concentration parameter\footnote{\rev{Note that we can assume $p + \eta_\delta \leq 1$ and $p-\eta_\delta \geq 0$, otherwise the concentration can be trivially bounded.}} $\eta_\delta \define \sqrt{\frac{6p \log(2\nisamples)}{\nisamples}}$ that
\begin{align*}
    \mathcal{E} &\define \{\vert \bsample - \nisamples p \vert \geq \nisamples \eta_\delta\}
\end{align*}
with probability $\Pr(\mathcal{E}) \leq \frac{1}{\nisamples^2}$.

Then, we can write 
\begin{align}
    \Pr(\isamplesel = 1) &= q \mathbb{E}\left[\frac{1}{\mathbb{E}_{\ber{\frac{\bsample+1}{\nisamples}}}[\fracrv]}\right] \nonumber \\
    &= q \mathbb{E}\left[\frac{1}{\mathbb{E}_{\ber{\frac{\bsample+1}{\nisamples}}}[\fracrv]} \cdot \mathds{1}\{\mathcal{E}^c\} \right] + q \mathbb{E}\left[\frac{1}{\mathbb{E}_{\ber{\frac{\bsample+1}{\nisamples}}}[\fracrv]} \cdot \mathds{1}\{\mathcal{E}\} \right] \label{eq:concentration_split}
\end{align}

Assume for now that $q<p$ (we will later proof the opposite event), then $\frac{1}{\mathbb{E}_{\ber{\frac{\bsample +1}{\nisamples}}}[\fracrv]}$ is strictly non-increasing in $\bsample$ since $\frac{q}{p}<\frac{1-q}{1-p}$, and hence, when $\mathcal{E}^c$ holds and hence $\bsample$ concentration around the average that

\begin{align*}
    \frac{1}{\mathbb{E}_{\ber{\frac{\bsample+1}{\nisamples}}}[\fracrv]} &\leq \frac{1}{\mathbb{E}_{\ber{\frac{(\bsample+1) \cdot (p-\eta_\delta)}{\nisamples}}}[\fracrv]} \\
    &= \frac{1}{\frac{(\nisamples-1)(p-\eta_\delta)+1}{\nisamples} \frac{q}{p} + \frac{\nisamples-1-(\nisamples-1)(p-\eta_\delta)}{\nisamples} \frac{1-q}{1-p}} \\
    &= \frac{1}{\left(p-\frac{p}{\nisamples}+\frac{\eta_\delta}{\nisamples}-\eta_\delta + \frac{1}{\nisamples}\right) \frac{q}{p} + \left(1-p-\frac{1}{\nisamples}+\frac{p}{\nisamples} + \eta_\delta - \frac{\eta_\delta}{\nisamples}\right) \frac{1-q}{1-p}} \\
    &= \frac{1}{1+\left(\frac{q}{p} - \frac{1-q}{1-p}\right) \left(\frac{1-p+\eta_\delta-n \eta_\delta}{\nisamples}\right)} \nonumber \\
    &= 1 + \sum_{\kappa = 1}^\infty (-1)^\kappa \left(\frac{q}{p} - \frac{1-q}{1-p}\right)^\kappa \left(\frac{1-p+\eta_\delta-n \eta_\delta}{\nisamples}\right)^\kappa,
\end{align*}
where the last step is by Taylor expansion. Using \eqref{eq:concentration_split} and the monotonicity of $\frac{1}{\mathbb{E}_{\ber{\frac{\bsample +1}{\nisamples}}}[\fracrv]}$, we write
\begin{align*}
    \Pr(\isamplesel = 1) &= q \mathbb{E}\left[\frac{1}{\mathbb{E}_{\ber{\frac{\bsample+1}{\nisamples}}}[\fracrv]}\right] \nonumber \\
    &\leq q \left(1 + \sum_{\kappa = 1}^\infty (-1)^\kappa \left(\frac{q}{p} - \frac{1-q}{1-p}\right)^\kappa \left(\frac{1-p+\eta_\delta-n \eta_\delta}{\nisamples}\right)^\kappa\right) + q \delta  \frac{p}{q},
\end{align*}
and hence
\begin{align*}
    \Pr(\isamplesel = 1) - q \leq \delta p + (1-\delta) \sum_{\kappa = 1}^\infty (-1)^\kappa  \left(\frac{q}{p} - \frac{1-q}{1-p}\right)^\kappa \left(\frac{1-p+\eta_\delta-n \eta_\delta}{\nisamples}\right)^\kappa
\end{align*}
Similarly, we get by bounding $\frac{1}{\mathbb{E}_{\ber{\frac{\bsample+1}{\nisamples}}}[\fracrv]} \geq \frac{1}{\mathbb{E}_{\ber{\frac{(\bsample+1) \cdot (p+\eta_\delta)}{\nisamples}}}[\fracrv]}$ and using \eqref{eq:concentration_split} that
\begin{align*}
    \Pr(\isamplesel = 1) - q \geq \delta q \frac{1-p}{1-q} + (1-\delta) \sum_{\kappa = 1}^\infty (-1)^\kappa  \left(\frac{q}{p} - \frac{1-q}{1-p}\right)^\kappa \left(\frac{1-p-\eta_\delta+n \eta_\delta}{\nisamples}\right)^\kappa \Leftrightarrow \\
    q - \Pr(\isamplesel = 1) \leq -\delta q \frac{1-p}{1-q} + (1-\delta) \sum_{\kappa = 1}^\infty (-1)^{\kappa+1}  \left(\frac{q}{p} - \frac{1-q}{1-p}\right)^\kappa \left(\frac{1-p-\eta_\delta+n \eta_\delta}{\nisamples}\right)^\kappa\!\!.
\end{align*}

When $p\leq q$, then $\frac{1}{\mathbb{E}_{\ber{\frac{\bsample +1}{\nisamples}}}[\fracrv]}$ is strictly non-decreasing, hence, under $\mathcal{E}$, we have 
\begin{align*}
    & \frac{1}{\mathbb{E}_{\ber{\frac{\bsample +1}{\nisamples}}}[\fracrv]} \leq \frac{1}{\mathbb{E}_{\ber{\frac{(\bsample+1) \cdot (p+\eta_\delta)}{\nisamples}}}[\fracrv]} = 1 \! + \! \sum_{\kappa = 1}^\infty (-1)^\kappa \! \left(\frac{q}{p} - \frac{1-q}{1-p}\right)^\kappa \!\! \left(\frac{1-p-\eta_\delta+n \eta_\delta}{\nisamples}\right)^\kappa\!\!\!\!,
\end{align*}
and thus from \eqref{eq:concentration_split} that
\begin{align*}
    \Pr(\isamplesel = 1) - q \leq q \delta \frac{1-p}{1-q} + (1-\delta) \sum_{\kappa = 1}^\infty (-1)^\kappa  \left(\frac{q}{p} - \frac{1-q}{1-p}\right)^\kappa \left(\frac{1-p-\eta_\delta+n \eta_\delta}{\nisamples}\right)^\kappa.
\end{align*}
Similarly, we bound $\frac{1}{\mathbb{E}_{\ber{\frac{\bsample +1}{\nisamples}}}[\fracrv]} \leq \frac{1}{\mathbb{E}_{\ber{\frac{(\bsample+1) \cdot (p+\eta_\delta)}{\nisamples}}}[\fracrv]}$ to obtain
\begin{align*}
    \Pr(\isamplesel = 1) - q \geq q \delta \frac{p}{q} + (1-\delta) \sum_{\kappa = 1}^\infty (-1)^\kappa  \left(\frac{q}{p} - \frac{1-q}{1-p}\right)^\kappa \left(\frac{1-p+\eta_\delta-n \eta_\delta}{\nisamples}\right)^\kappa \Leftrightarrow \\
    q-\Pr(\isamplesel = 1) \leq -q \delta \frac{p}{q} + (1-\delta) \sum_{\kappa = 1}^\infty (-1)^{\kappa+1}  \left(\frac{q}{p} - \frac{1-q}{1-p}\right)^\kappa \left(\frac{1-p+\eta_\delta-n \eta_\delta}{\nisamples}\right)^\kappa
\end{align*}

Since $0 \leq p+\eta_\delta\leq 1$ and $1\geq p-\eta_\delta \geq 0$ by an appropriate choice of the concentration intervals, we have by approximations up to second order terms that
\begin{align*}
    \vert\Pr(\isamplesel = 1) - q\vert &\leq q \delta \max\left\{\frac{p}{q}, \frac{1-p}{1-q}\right\} + \eta_\delta \left(\frac{q}{p} - \frac{1-q}{1-p}\right) + \left(\frac{q}{p} - \frac{1-q}{1-p}\right)^2 \! \mathcal{O}\!\left(\frac{1}{\nisamples^2} + \eta_\delta^2\right) \\
    & \hspace{-1cm} = \frac{q}{\nisamples^2} \left(\frac{p}{q} + \frac{1-p}{1-q}\right) + \mathcal{O}\left(\left[\left(\frac{q}{p} - \frac{1-q}{1-p}\right) + \left(\frac{q}{p} - \frac{1-q}{1-p}\right)^2\right] \sqrt{\frac{6 p \log\left(2\nisamples\right)}{\nisamples}}\right)\!\!.
\end{align*}
This concludes the proof.
\end{proof}

\begin{proof}[Proof of \cref{lemma:contraction}]
\rev{
Using \cref{lemma:qdiv}, we can show the following. 
Recall the following probability law of the stochastic quantizer $\stochquant$ \citep{alistarh2017qsgd} using $\nqints>0$ quantization intervals, which takes as input the entry $\tmpvecentry$ of a gradient $\tmpvec \in \mathbb{R}^\dimension$ vector. Let $0 \leq \qchoice < \nqints$ be an integer such that $\frac{\qchoice}{\nqints} \leq \frac{\vert \tmpvecentry\vert}{\norm{\tmpvec}} \leq \frac{\qchoice+1}{\nqints}$, then $\stochquant[\tmpvecentry]$ is defined as $\ber{\frac{\vert \tmpvecentry\vert}{\norm{\tmpvec}} \nqints - \qchoice}$, which outputs $\norm{\tmpvec} \cdot \text{sign}(\tmpvecentry) (\qchoice+1) / \nqints$ in case of success, and $\norm{\tmpvec} \cdot \text{sign}(\tmpvecentry) \qchoice/\nqints$ otherwise.}

\rev{Focusing on an entry $\tmpvecentry$, we prove a contraction property for MRC with stochastic quantization with posterior $\posteriorentrytmp=\frac{\vert \tmpvecentry\vert}{\norm{\tmpvec}} \nqints - \qchoice$, and an arbitrary prior $\priorentry$. In fact, the MRC methodology $\iscompression$ leads to sampling from an approximate distribution with parameter $\posteriorapproxentry$. To be more specific, $\iscompression[\tmpvecentry]$ outputs $\norm{\tmpvec} \cdot \text{sign}(\tmpvecentry) (\qchoice+1) / \nqints$ with probability $\posteriorapproxentry$, and $\norm{\tmpvec} \cdot \text{sign}(\tmpvecentry) \qchoice/\nqints$ with probability $1-\posteriorapproxentry$. We established in \cref{lemma:qdiv} an upper bound on $\vert \posteriorentrytmp-\posteriorapproxentry \vert$, which will be useful in the following.
}

\rev{
To prove a contraction property of the kind
\begin{align*}
\mathbb{E}[\Vert \iscompression[\tmpvec]- \tmpvec \Vert_2^2] \leq (1-\delta) \norm{\tmpvec}^2,
\end{align*}
we can write
\begin{align*}
    \mathbb{E}[&\norm{\iscompression[\tmpvec] - \tmpvec}^2] = \mathbb{E}\left[\sum_{\entry = 1}^\dimension \left( \iscompression[\tmpvecentry] - \tmpvecentry \right)^2\right] \\
    &= \norm{\tmpvec}^2 \sum_{\entry = 1}^\dimension \mathbb{E}\left[\left( \frac{\iscompression[\tmpvecentry]}{\norm{\tmpvec}} - \frac{\tmpvecentry}{\norm{\tmpvec}} \right)^2\right] \\
    &= \norm{\tmpvec}^2 \sum_{\entry = 1}^\dimension \left[ \posteriorapproxentry\left( \frac{\text{sign}(\tmpvecentry) (\qchoice+1)}{\nqints} - \frac{\tmpvecentry}{\norm{\tmpvec}} \right)^2 + (1-\posteriorapproxentry)\left( \frac{\text{sign}(\tmpvecentry) \qchoice}{\nqints} - \frac{\tmpvecentry}{\norm{\tmpvec}} \right)^2 \right] \\
    &= \norm{\tmpvec}^2 \sum_{\entry = 1}^\dimension \left[ (\posteriorapproxentry - \posteriorentrytmp + \posteriorentrytmp) \left( \frac{\qchoice+1}{\nqints} - \frac{\vert \tmpvecentry \vert}{\norm{\tmpvec}} \right)^2 + (1-\posteriorapproxentry - \posteriorentrytmp + \posteriorentrytmp)\left( \frac{\qchoice}{\nqints} - \frac{\vert \tmpvecentry \vert}{\norm{\tmpvec}} \right)^2 \right] \\
    &= \norm{\tmpvec}^2 \sum_{\entry = 1}^\dimension \Bigg[ (\posteriorentrytmp + \posteriorapproxentry - \posteriorentrytmp) \left(\left(\frac{\qchoice}{\nqints} - \frac{\vert \tmpvecentry \vert}{\norm{\tmpvec}} \right)^2 + \frac{1}{\nqints^2} + \frac{1}{\nqints}\left(\frac{\qchoice}{\nqints} - \frac{\vert \tmpvecentry \vert}{\norm{\tmpvec}} \right) \right) \\
    &\hspace{4cm} + (1-\posteriorentrytmp+\posteriorentrytmp-\posteriorapproxentry)\left( \frac{\qchoice}{\nqints} - \frac{\vert \tmpvecentry \vert}{\norm{\tmpvec}} \right)^2 \Bigg] \\
    &= \norm{\tmpvec}^2 \sum_{\entry = 1}^\dimension \Bigg[ (\posteriorapproxentry - q) \left(\frac{1}{\nqints^2} + \frac{1}{\nqints}\left(\frac{\qchoice}{\nqints} - \frac{\vert \tmpvecentry \vert}{\norm{\tmpvec}} \right) \right) + \posteriorentrytmp \left(\frac{1}{\nqints^2} + \frac{1}{\nqints}\left(\frac{\qchoice}{\nqints} - \frac{\vert \tmpvecentry \vert}{\norm{\tmpvec}} \right) \right) + \left( \frac{\qchoice}{\nqints} - \frac{\vert \tmpvecentry \vert}{\norm{\tmpvec}} \right)^2 \Bigg], \numberthis \label{eq:tmpcontraction} \\
\end{align*}
where
\begin{align*}
    &\posteriorentrytmp \left(\frac{1}{\nqints^2} + \frac{1}{\nqints}\left(\frac{\qchoice}{\nqints} - \frac{\vert \tmpvecentry \vert}{\norm{\tmpvec}} \right) \right) \\
    &= \left(\frac{\vert \tmpvecentry\vert}{\norm{\tmpvec}} \nqints - \qchoice\right) \left(\frac{1}{\nqints^2} + \frac{1}{\nqints}\left(\frac{\qchoice}{\nqints} - \frac{\vert \tmpvecentry \vert}{\norm{\tmpvec}} \right) \right) \\
    &= -\nqints \left(\frac{\qchoice}{\nqints} - \frac{\vert \tmpvecentry\vert}{\norm{\tmpvec}}\right) \frac{1}{\nqints} \left(\frac{1}{\nqints} + \left(\frac{\qchoice}{\nqints} - \frac{\vert \tmpvecentry \vert}{\norm{\tmpvec}} \right) \right) \\
    &= -\left(\frac{\qchoice}{\nqints} - \frac{\vert \tmpvecentry\vert}{\norm{\tmpvec}}\right)^2 - \frac{1}{\nqints} \left(\frac{\qchoice}{\nqints} - \frac{\vert \tmpvecentry\vert}{\norm{\tmpvec}}\right).
\end{align*}
Substituting the result in \eqref{eq:tmpcontraction}, obtain
\begin{align*}
    \mathbb{E}[\norm{\iscompression[\tmpvec] - \tmpvec \Vert}^2] &= \mathbb{E}\left[\sum_{\entry = 1}^\dimension \left( \iscompression[\tmpvecentry] - \tmpvecentry \right)^2\right] \\
    &= \norm{\tmpvec}^2 \sum_{\entry = 1}^\dimension \Bigg[ (\posteriorapproxentry - \posteriorentrytmp) \left(\frac{1}{\nqints^2} + \frac{1}{\nqints}\left(\frac{\qchoice}{\nqints} - \frac{\vert \tmpvecentry \vert}{\norm{\tmpvec}} \right) \right) - \frac{1}{\nqints} \left(\frac{\qchoice}{\nqints} - \frac{\vert \tmpvecentry\vert}{\norm{\tmpvec}}\right) \Bigg] \\
    &= \norm{\tmpvec}^2 \sum_{\entry = 1}^\dimension \Bigg[ (\posteriorapproxentry - \posteriorentrytmp) \frac{1}{\nqints}\left(\frac{\qchoice+1}{\nqints} - \frac{\vert \tmpvecentry \vert}{\norm{\tmpvec}} \right) - \frac{1}{\nqints} \left(\frac{\qchoice}{\nqints} - \frac{\vert \tmpvecentry\vert}{\norm{\tmpvec}}\right) \Bigg] \\
    &\leq \norm{\tmpvec}^2 \sum_{\entry = 1}^\dimension \Bigg[ \vert\posteriorapproxentry - \posteriorentrytmp \vert \frac{1}{\nqints}\left(\frac{\qchoice+1}{\nqints} - \frac{\vert \tmpvecentry \vert}{\norm{\tmpvec}} \right) + \frac{1}{\nqints} \left(\frac{\vert \tmpvecentry\vert}{\norm{\tmpvec}} - \frac{\qchoice}{\nqints} \right) \Bigg] \\
    &\leq \norm{\tmpvec}^2 ( \vert\posteriorapproxentry - \posteriorentrytmp \vert \frac{\dimension}{\nqints^2} + \frac{\dimension}{\nqints^2} ), \\
\end{align*}
where, by \cref{lemma:qdiv}, we have for $\Delta_\entry \define \frac{\posteriorentrytmp}{\priorentry} - \frac{1-\posteriorentrytmp}{1-\priorentry}$ and $\Delta^\prime_\entry \define \posteriorentrytmp \left(\frac{\priorentry}{\posteriorentrytmp} + \frac{1-\priorentry}{1-\posteriorentrytmp}\right)$ that
\begin{align*}
    \vert\posteriorapproxentry - \posteriorentrytmp \vert \leq \frac{\Delta^\prime_\entry}{\nisamples^2} + \mathcal{O}\left((\Delta_\entry + \Delta_\entry^2) \sqrt{\frac{6 \priorentry \log\left(2\nisamples\right)}{\nisamples}}\right).
\end{align*}
Let $\bar{\Delta} \define \max_{\entry \in [\dimension]} \frac{\posteriorentrytmp}{\priorentry} - \frac{1-\posteriorentrytmp}{1-\priorentry}$, $\bar{\Delta}^\prime \define \max_{\entry \in [\dimension]} \posteriorentrytmp \left(\frac{\priorentry}{\posteriorentrytmp} + \frac{1-\priorentry}{1-\posteriorentrytmp}\right)$, and $\bar{p} \define \max_{\entry \in [\dimension]} \priorentry$. 
}
\rev{
We will ensure that $\frac{\bar{\Delta}^\prime}{\nisamples^2} + \mathcal{O}\left((\bar{\Delta} + \bar{\Delta}^2) \sqrt{\frac{6 \bar{p} \log\left(2\nisamples\right)}{\nisamples}}\right) \leq 1$ by making each of the individual terms $\leq \frac{1}{2}$. 
By choosing $\nisamples \geq \sqrt{2\bar{\Delta}^\prime}$, we have $\frac{\bar{\Delta}^\prime}{\nisamples^2} \leq \frac{1}{2}$. 
To ensure that $(\bar{\Delta} + \bar{\Delta}^2) \sqrt{\frac{6 \bar{p} \log\left(2\nisamples\right)}{\nisamples}} \leq \frac{1}{2}$, we require $\frac{\log(2\nisamples)}{\nisamples} \leq \frac{1}{\sqrt{6 \bar{p} (\bar{\Delta} + \bar{\Delta}^2)}}$. By \citet[Lemma 15]{weinberger2023multi}, this holds when $\nisamples = \mathcal{O}(\log(6 \bar{p} (\bar{\Delta} + \bar{\Delta}^2)) \sqrt{6 \bar{p} (\bar{\Delta} + \bar{\Delta}^2)})$.
Hence, choosing $\nisamples = \mathcal{O}(\max\{\sqrt{2\bar{\Delta}^\prime}, \log(6 \bar{p} (\bar{\Delta} + \bar{\Delta}^2)) \sqrt{6 \bar{p} (\bar{\Delta} + \bar{\Delta}^2)}\})$, we have $\frac{\bar{\Delta}^\prime}{\nisamples^2} + \mathcal{O}\left((\bar{\Delta} + \bar{\Delta}^2) \sqrt{\frac{6 \bar{p} \log\left(2\nisamples\right)}{\nisamples}}\right) \leq 1$. Thus, we have $0 \leq \delta \leq 1$ if $\frac{2\dimension}{\nqints^2} \leq 1$, and hence $\nqints \geq \sqrt{2\dimension}$. This concludes the proof.
}

\end{proof}

\begin{proof}[Proof of \cref{thm:downlink_kl}]

Assume a party estimates the Bernoulli distributions $\genericposterior_j$ with parameters $q_j$ held by parties $j \in [\nclients]$. The estimating party shares with each of the other parties a common prior $\genericprior_j$ in the form of a Bernoulli distribution with parameter $p_j$ and access to unlimited shared randomness. To help estimate $\genericposterior_j$, the $j$-th party sends $\ngenericsamples$ samples to the estimator through MRC. Therefore, both parties sample $\ngenericsamples \nisamples$ i.i.d. samples $\isampleall \sim \genericprior_j$ for $\sampleidx \in [\ngenericsamples], \isampleidx \in [\nisamples]$, independently and identically from $\genericprior_j$. The party holding $\genericposterior_j$ constructs for each $\sampleidx \in [\ngenericsamples]$ an auxiliary distribution $$W_\sampleidx(\isampleidx) = \frac{\genericposterior_j(\isampleall)/\genericprior_j(\isampleall)}{ \sum_{\isampleidx=1}^{\nisamples} \genericposterior_j(\isampleall)/\genericprior_j(\isampleall)},$$ from which it samples to obtain an index $I_\sampleidx$. The index is transmitted to the estimating party, which reconstructs the corresponding sample $\isamplesel$.  Averaging the samples for all $\sampleidx \in [\ngenericsamples]$ gives an estimate $\hat{q}_j$ of $q_j$, i.e., $\hat{q}_j = \frac{1}{\ngenericsamples} \sum_{\sampleidx=1}^\ngenericsamples \isamplesel$. This process is repeated for all $j \in [\nclients]$.

We assume that $\vert q_j - p_j \vert \leq \klball$ for all $\client, j \in [\nclients]$, and that the difference between the priors, is bounded as $\vert p_\client - p_j\vert \leq \priordiv$ for all $\client, j \in [\nclients]$. The goal is to bound $\binkl{\frac1n \sum_{j=1}^\nclients \hat{q}_j}{p_\client}$ from above for any $\client \in [\nclients]$.

By the convexity of KL-divergence, we have
\begin{align*}
    \binkl{\frac1n \sum_{j=1}^\nclients \hat{q}_j}{p_\client} \leq \frac1n \sum_{\client=1}^\nclients \binkl{\hat{q}_j}{p_\client}.
\end{align*}
To bound $\binkl{\hat{q}_j}{p_\client}$ for any $\client, j \in [\nclients]$, by the triangle inequality, we can write
\begin{align*}
    \vert \hat{q}_j - p_\client \vert \leq \vert \hat{q}_j - \Pr(X_\sampleidx=1) \vert + \vert \Pr(X_\sampleidx=1) - q_j \vert + \vert q_j - p_j \vert + \vert p_j - p_\client \vert,
\end{align*}
where $\vert \hat{q}_j - \Pr(X_\sampleidx=1) \vert$ is bounded by \cref{lemma:qdiv}. By Hoeffding's inequality, we have with probability at least $1-\delta^\prime$ that
\begin{equation*}
    \vert\hat{q} - \Pr(X_{\sampleidx} = 1)\vert \leq \sqrt{\frac{-\ln(\delta^\prime/2)}{2 \nisamples}}.
\end{equation*}
Thus, with probability at least $1-\delta^\prime$, since $p_j \leq p_\client + \priordiv$, we have with $\Delta_j \define \frac{q_j}{p_j-\priordiv} - \frac{1-q_j}{1-p_j+\priordiv}$ and $\Delta^\prime_j \define q_j \left(\frac{p_j+\priordiv}{q_j} + \frac{1-p_j+\priordiv}{1-q_j}\right)$ that
\begin{align*}
    \vert \hat{q}_j - p_\client \vert \leq \frac{\Delta^\prime_j}{\nisamples^2} + \mathcal{O}\left((\Delta_j + \Delta^2_j) \sqrt{\frac{6 (p_\client + \priordiv) \log\left(2\nisamples\right)}{\nisamples}}\right) + \sqrt{\frac{-\ln(\delta^\prime/2)}{2 \nisamples}} + \klball + \priordiv.
\end{align*}
This holds under the assumption that $p_j > \priordiv$ for all $j\in [\nclients]$.
By the reversed Pinsker's inequality, we obtain
\begin{align*}
    \kl{\hat{q}_j}{p_\client} \leq \frac{2}{\min\{p_\client, 1-p_\client\}} &\left(\frac{\Delta^\prime_j}{\nisamples^2} + \mathcal{O}\left((\Delta_j + \Delta^2_j) \sqrt{\frac{6 (p_\client+\priordiv) \log\left(2\nisamples\right)}{\nisamples}}\right) \right. \\
    &\left.+ \sqrt{\frac{-\ln(\delta^\prime/2)}{2 \nisamples}} + \klball + \priordiv\right)^2.
\end{align*}
The statement of the theorem follows by the convexity of KL-divergence.
\end{proof}

\section{Convergence Analysis} \label{sec:convergence_analysis}

\rev{
Using the contraction property derived in \cref{lemma:contraction}, we can show that a straightforward extension of \algglobalrandcfl\ to error-feedback as used in \citep{richtarik2021ef21} leads to the following convergence guarantee. 
Therefore, assume that for all for $\mathbf{x}, \mathbf{y} \in \mathbb{R}^\dimension$ and $\client \in [\nclients]$, the following Lipschitz property holds:
\begin{align*}
    \Vert \nabla F(\mathbf{x}, \data) - \nabla F(\mathbf{y}, \data) \Vert \leq \locallipschitz \Vert \mathbf{x} - \mathbf{y} \Vert
\end{align*}
Let $F(\model[]) \define \frac{1}{\nclients} \sum_{\client=1}^\nclients \nabla F(\model[], \data)$ be the global loss function and $\lipschitztmp \define \sqrt{\frac{1}{\nclients} \sum_{\client=1}^\nclients \locallipschitz}$.
\begin{theorem}
If $F^\star \define \inf_{\model[] \in \mathbb{R}^\dimension} \{F(\model[])\} > -\infty$ and $\mathbb{E}[\Vert \mathbf{g}^\epoch - \nabla F(\model[\epoch]) \Vert^2] \leq \gradvar$, then with $\lr \leq \left(\lipschitz + \lipschitztmp \sqrt{\frac{1-\contraction}{(1-\sqrt{1-\contraction})^2}}\right)^{-1}\!\!\!\!\!$, $\nlocalepochs=1$, $\nqints \geq \sqrt{2\dimension}$, and $\nisamples$ satisfying \cref{lemma:contraction} in every iteration $\epoch$, we have for a straightforward extension of \algglobalrandcfl\ to error-feedback such that
\begin{align*}
    \sum_{\epoch=1}^T \mathbb{E}\left[ \Vert F(\model[\epoch]) \Vert^2 \right] \leq \frac{2(F(\model[0])-F^\star}{\lr T} + \frac{\gradvar}{(1-\sqrt{1-\contraction}) T}.
\end{align*}
\end{theorem}
Similarly, guarantees can be derived for other algorithms, such as modified versions of \alglocalrand\ with error-feedback and momentum, using \cref{lemma:contraction}. However, we emphasize the generality of \alg, reaching beyond conventional FL with stochastic compression to pure stochastic narratives.
}

\section{Gradient Descent with a KL-Proximity} \label{app:mirror_descent}

Mirror descent employs point-wise optimization in the form of a first-order approximation of $\loss$ with proximity term $D_F(p,q)$, where $D_F$ is the Bregman divergence associated with function $F(\cdot)$. When $F(x) = \Vert x \Vert^2$, and hence the Bregman divergence is the Euclidean distance, this is known as gradient descent. Let now $p$ and $q$ be vectors with the entries corresponding to independent Bernoulli parameters. When we choose $F(x) = x \log(x) + (1-x) \log (1-x)$, the Bregman divergence becomes $D_F(p, q) = \sum_{k=1}^d \kl{p_k}{q_k}$. Hence, we are optimizing with respect to a KL-proximity constraint. The mapping between dual and primal spaces is then given by $\nabla F(x) = \log(x) - \log(1-x)$ and $\left(\nabla F(x)\right)^{-1} = \frac{1}{e^{-x}+1}$, respectively; also known as the inverse sigmoid and the sigmoid functions.

\section{Block Allocation} \label{app:block_allocation}
The simplest yet effective strategy for block allocation is to partition the model into equally-sized blocks of size $\dimension/\nblocks$ for MRC (Fixed). The partitioning into blocks is required to make MRC practically feasible in this setting. It is known that for vanishing MRC error, the number of samples $\nisamples$ from a block $\priorulblock$ of the prior is supposed to be in the order of $\exp\left(\kl{\posteriorblock}{\priorulblock}\right)$, where $\posteriorblock$ is the $\blockidx$-th block of posterior $\posterior$. It was observed by \cite{isik2024adaptive} that the KL-divergence decreases as the training progresses with the global model used as a prior, which is intuitive since the local training will change the posterior less and less as training converges. To adapt the block size according to the divergence from the posterior with respect to the prior, \cite{isik2024adaptive} proposed an adaptive block allocation strategy (Adaptive), where upon realizing a large deviation from the target KL-divergence per block, clients partition their model into blocks with equal sums of parameter-wise KL-divergences and transmit the block intervals to the federator. The federator aggregates the indices of all the clients, and broadcasts the updated block allocation. We propose in this work a low complexity solution that adapts the block size according to the average KL-divergence per block (Adaptive-Avg). This alleviates the cost of computing and transmitting the exact block partitions, where the transmission of each block size requires $\log_2(b_{\max})$ bits, with $b_{\max}$ the maximum pre-defined block size. Instead, the transmission of one size is enough in our solution. If the average KL per block $\kl{\posteriorblock}{\priorulblock}$ deviates more than a given factor, the clients request to update the blocks. In the next iteration, each client proposes a block size, and the federator averages and broadcasts an updated size.

\section{Additional Experimental Details} \label{app:sim_setup}
We use the cross-entropy loss and a batch size of $128$ in all our experiments. We use Adam \citep{kingma2014adam} as an optimizer with learning rate $\eta = 0.0003$ for all non-stochastic methods, and $\eta=0.1$ for probabilistic mask training. For non-stochastic FL, we use a federator (server) learning rate of $0.1$, i.e., the clients' gradients are averaged, and the federator updates the global model with learning rate $0.1$\rev{, and with a learning rate of $0.005$ for \algglobalrand\ with SignSGD}. For \textsc{M3}, we use a federator learning rate of $0.02$ to obtain reliable results. For \textsc{Liec} and \textsc{Cser}, we use an average period of $50$ global iterations (cf. \citep{cheng2024communication,xie2020cser}). For M3, we use TopK with $K=\lfloor\dimension/\nclients\rfloor$. To run the simulations, we use a cluster of different architectures, which we list in the following table.

\begin{table}[H]
\centering
\begin{tabular}{|l|l|l|l|}
\hline
\textbf{CPU(s)}                                  & \textbf{RAM}   & \textbf{GPU(s)}                         & \textbf{VRAM} \\ \hline
2x Intel Xeon Platinum 8176 (56 cores)         & 256 GB         & 2x NVIDIA GeForce GTX 1080 Ti        & 11 GB                 \\ 
2x AMD EPYC 7282 (32 cores)                    & 512 GB         & NVIDIA GeForce RTX 4090              & 24 GB                 \\ 
2x AMD EPYC 7282 (32 cores)                    & 640 GB         & NVIDIA GeForce RTX 4090              & 24 GB                 \\ 
2x AMD EPYC 7282 (32 cores)                    & 448 GB         & NVIDIA GeForce RTX 4080              & 16 GB                 \\ 
2x AMD EPYC 7282 (32 cores)                    & 256 GB         & NVIDIA GeForce RTX 4080              & 16 GB                 \\ 
HGX-A100 (96 cores)                            & 1 TB           & 4x NVIDIA A100                      & 80 GB                 \\ 
DGX-A100 (252 cores)                           & 2 TB           & 8x NVIDIA Tesla A100                & 80 GB                 \\ 
DGX-1-V100 (76 cores)                          & 512 GB         & 8x NVIDIA Tesla V100                & 16 GB                 \\ 
DGX-1-P100 (76 cores)                          & 512 GB         & 8x NVIDIA Tesla P100                & 16 GB                 \\ 
HPE-P100 (28 cores)                            & 256 GB         & 4x NVIDIA Tesla P100                & 16 GB                 \\ \hline
\end{tabular}
\caption{System specifications of our simulation cluster.}
\end{table}

The details of the CNN architectures used in our experiments are summarized in the following. The parameter count is $61706$ for LeNet5, $1933258$ for 4CNN, and $2262602$ for 6CNN.

\begin{table}[H]
\centering
\caption{LeNet5 Architecture Overview} \label{tab:lenet5}
\begin{tabular}{|l|l|l|} \hline
\textbf{Layer} & \textbf{Specification} & \textbf{Activation} \\ \hline
5x5 Conv & 6 filters, stride 1 & ReLU, AvgPool (2x2) \\
5x5 Conv & 16 filters, stride 1 & ReLU, AvgPool (2x2) \\
Linear & 120 units & ReLU \\
Linear & 84 units & ReLU \\
Linear & 10 units & Softmax \\ \hline
\end{tabular}
\end{table}

\begin{table}[H]
\centering
\caption{4-layer CNN (4CNN) Architecture Overview} \label{tab:4cnn}
\begin{tabular}{|l|l|l|} \hline
\textbf{Layer} & \textbf{Specification} & \textbf{Activation} \\ \hline
3x3 Conv & 64 filters, stride 1 & ReLU \\
3x3 Conv & 64 filters, stride 1 & ReLU, MaxPool (2x2) \\
3x3 Conv & 128 filters, stride 1 & ReLU \\
3x3 Conv & 128 filters, stride 1 & ReLU, MaxPool (2x2) \\
Linear & 256 units & ReLU \\
Linear & 256 units & ReLU \\
Linear & 10 units & Softmax \\ \hline
\end{tabular}
\end{table}

\begin{table}[H]
\caption{6-layer CNN (6CNN) Architecture Overview} \label{tab:6cnn}
\centering
\begin{tabular}{|l|l|l|} \hline
\textbf{Layer} & \textbf{Specification} & \textbf{Activation} \\ \hline
3x3 Conv & 64 filters, stride 1 & ReLU \\
3x3 Conv & 64 filters, stride 1 & ReLU, MaxPool (2x2) \\
3x3 Conv & 128 filters, stride 1 & ReLU \\
3x3 Conv & 128 filters, stride 1 & ReLU, MaxPool (2x2) \\
3x3 Conv & 256 filters, stride 1 & ReLU \\
3x3 Conv & 256 filters, stride 1 & ReLU, MaxPool (2x2) \\
Linear & 256 units & ReLU \\
Linear & 256 units & ReLU \\
Linear & 10 units & Softmax \\ \hline
\end{tabular}
\end{table}

\rev{For the sake of clarity, in the paper we restrict the analysis to a fixed number of importance samples $\nisamples$, block sizes $\nblocks$, and choice of priors $\priorul, \priordl$. Our experiments have shown that, while increasing $\nisamples$ beyond the ones used in our algorithms slightly improves the convergence over the number of epochs, the convergence with respect to the communication cost did not significantly improve. The block size is mainly limited by the system resources at hand, and one would choose the largest possible for best efficiency while complying with memory resources. We investigated many different prior choices and found the former global model to be reasonably good in almost all cases. With high heterogeneity, it might be beneficial to use different convex combinations as priors, which mix the former global model with the latest posterior estimate of a certain client, but the gains we experienced were minor. Hence, we settled on the former global estimate for simplicity in presenting the algorithm.}

\section{Federated Probabilistic Mask Training} \label{sec:fedpm}

\rev{The idea in federated probabilistic mask training (FedPM) \cite{isik2023sparse} is to collaboratively train a probabilistic mask that determines which weights to maintain from a randomly initialized network. The motivation stems from the \textit{lottery-ticket hypothesis} \citep{frankle2018lottery}, which claims that randomly initialized networks contain sub-networks capable of reaching accuracy comparable to that of the full network. 
The weights $\weights$ of the network are randomly initialized at the start of training, and remain fixed. The  federator and clients only train a mask, which  determines for each  parameter whether it is activated or not, i.e., identifying an efficient subnetwork within the given fixed network. The probabilistic masks $\model$ are described by Bernoulli distributions, i.e., $\model \in [0, 1]^\dimension$ contains a Bernoulli parameter to be trained for each weight of the network. These parameters determine the probability of retaining the corresponding weights. During inference, the weights $\weights$ are masked with samples $\maskinf \in \{0, 1\}^\dimension \sim \model$ from the distribution $\model$, i.e., the inference is conducted on a network with weights $\weights \odot \maskinf$. In FedPM, clients sample from their locally trained models, and send these samples to the federator, which, in turn, updates the global model by averaging these samples. The communication cost of this scheme is fixed for all iterations, even though the communication cost can be reduced since the KL-divergence between the global model and the locally trained models diminishes as the training progresses.}

\rev{
\setlength{\textfloatsep}{15pt}
\begin{algorithm}[!tb]
\caption{Local Training at Client $\client$}
\begin{algorithmic}[1] %
\REQUIRE Model $\modelest$ \\
\STATE Map model to scores in the dual space: $\scores[0] = \sigma^{-1}(\modelest) = \log\left(\frac{\modelest}{1-\modelest}\right)$
\FOR{Local iterations $\localepoch \in [\nlocalepochs]$}
\STATE $\scores[\ell] = \scores[0] - \lr \nabla_{\scores[\ell-1]}{F(\modelest^{(\localepoch-1)}, \data)}$, where $\modelest^{(\localepoch-1)} = \sigma(\scores[\ell-1])$
\ENDFOR
\STATE Map back to primal space: $\posterior = \sigma(\scores[\nlocalepochs])$
\end{algorithmic}
\label{alg:local_training}
\end{algorithm}
}

\rev{
We adopt the following federated learning procedure for collaboratively learning network masks, and highlight in the following the parallels to mirror descent by referring to primal and dual spaces. Starting from a common model $\model[0]$, at iteration $\epoch$, each client $\client$ locally trains the model $\modelest$ in $\nlocalepochs$ local iterations. 
To enable gradient descent, the model $\modelest$ is mapped to scores $\scores[0]$ in a dual space by the inverse Sigmoid function $\scores[0] = \sigma^{-1}(\modelest) = \log(\modelest)-\log(1-\modelest)$. The scores are then trained for $\nlocalepochs$ local iterations $\localepoch \in [\nlocalepochs]$ by computing the gradient $\nabla_{\scores[\ell-1]}{F(\modelest^{(\localepoch-1)}, \data)}$, where the straight-through estimator is used to compute the gradient of the non-differentiable Bernoulli sampling operation based on the distribution $\modelest^{(\localepoch-1)} = \sigma(\scores[\ell-1])$, i.e., the gradient equals the Bernoulli parameter. By mapping the model back to the primal space, each client $\client$ obtains a model update in terms of a posterior $\posterior = \sigma(\scores[\nlocalepochs])$. The client training process is summarized in \cref{alg:local_training}.
}

\section{Minimal Random Coding (MRC)} \label{sec:importance_sampling}

\rev{\citet{isik2024adaptive} proposed a method, called KL minimization with side information (KLMS), to reduce the cost of transmitting the local models $\posterior$ to the federator.  Consequently, the communication cost depends on the KL-divergence between the desired distribution and the common prior. This method utilizes the common side information available at both the clients and the federator, as well as shared randomness. The idea is that instead of sampling locally and sending the samples to the federator, the federator in the  KLMS method samples from the desired distribution through MRC.  %
In a nutshell, MRC \citep{havasi2018minimal} is based on importance sampling \citep{srinivasan2002importance} and makes use of a common prior to sample from a desired distribution. Consider two distributions $\genericprior$ and $\genericposterior$, where $\genericprior$ is known to both parties, and $\genericposterior$ is only known to the client. To make the federator sample from $\genericposterior$, both parties sample $\nisamples$ samples $\{\isample\}_{\isidx\in[\nisamples]}$ from $\genericprior$. The client forms an auxiliary distribution $\auxdist(\isidx) = \frac{\genericposterior(\isample)/\genericprior(\isample)}{\sum_{\isidx=1}^{\nisamples} \genericposterior(\isample)/\genericprior(\isample)}$ capturing the importance of the samples. A sample from $\auxdist$ is fully described by its index $\isidx$, which can be transmitted with $\log_2(\nisamples)$ bits, and approximates a sample from $\genericposterior$. \citet{chatterjee2018sample} shown that importance sampling with posterior $\genericposterior$ and prior $\genericprior$ requires $\nisamples$ to be in the order of $\Theta(\exp(\kl{\genericposterior}{\genericprior}))$ , where $\kl{\genericposterior}{\genericprior}$ denotes the KL-divergence between distributions $\genericposterior$ and $\genericprior$. In what follows, we will also denote the KL-divergence between two Bernoulli distributions $\genericposterior$ and $\genericprior$ with parameters $q$ and $p$ by $\binkl{q}{p}$.}

\section{Additional Experiments} \label{app:additional_experiments}

We provide in the following experiments for both uniform (\iid) and heterogeneous (\noniid) data distributions for training LeNet5 and a 4-layer CNN on MNIST, a 4-layer CNN on Fashion MNIST, and a 6-layer CNN on CIFAR-10. The details of the neural networks can be found in \cref{tab:lenet5,tab:4cnn,tab:6cnn}. For each setting and method depicted, we show the average of three simulation runs with different seeds. We plot for each setting the test accuracies over the communication cost in bits, and the maximum test accuracy over the bitrate. We provide tables summarizing the maximum test accuracies with their standard deviation over multiple runs, the total bitrates and the bitrates split into uplink and downlink. \rev{The overall bitrates per parameter (bpp) are computed assuming point-to-point links between all participants, i.e., uplink and downlink costs have equal weight. For the case when a broadcast (BC) link between the federator and the clients is available, the bitrate per parameter for all baseline schemes reduces by a factor of $\nclients$. \algglobalrand profits similarly from the broadcast link, but \alglocalrand cannot profit due to the absence of shared randomness, giving the same overall bitrate compared to the point-to-point link scenario.} We highlight for each of the measures the scheme with the best result. Consistently throughout all experiments, \alg\ achieves order-wise savings in the bitrates per parameter while reaching state-of-the-art accuracies in the classification task. \rev{While the sampling can introduce an additional computational overhead depending on the implementation, the storage cost is similar to the baselines. Since we leverage as priors the former global model, the additional storage cost incurred is limited to storing until the next iteration the estimate of the former global model at each client, i.e., where the training started, which is usually not a bottleneck. This can be cheaper than some baselines, which require storing data for momentum and error-feedback.}

\begin{figure}[H]
\subfigure[Test Accuracy over Communication]{\includegraphics[width = .5\linewidth]{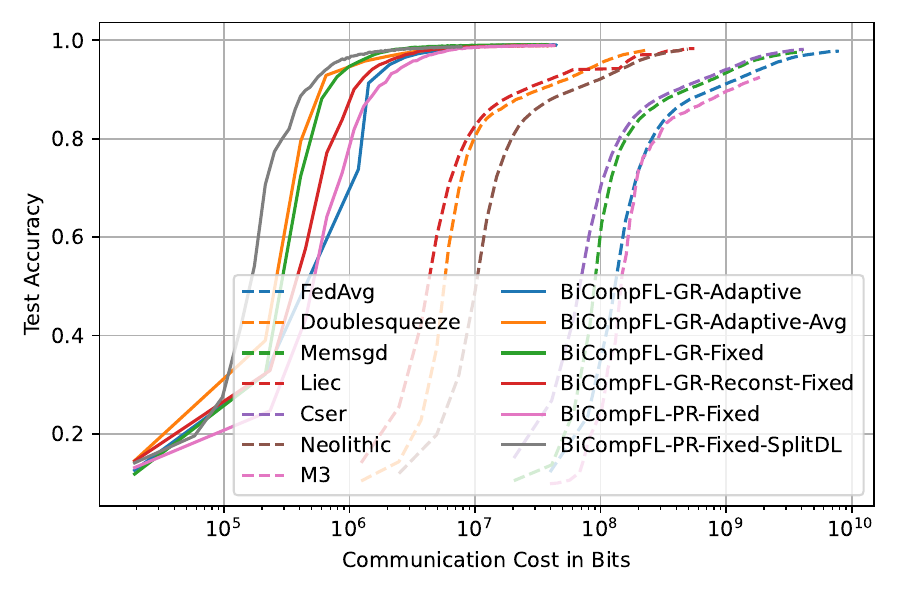}}
\subfigure[Test Accuracy over Bitrate]{\includegraphics[width = .5\linewidth]{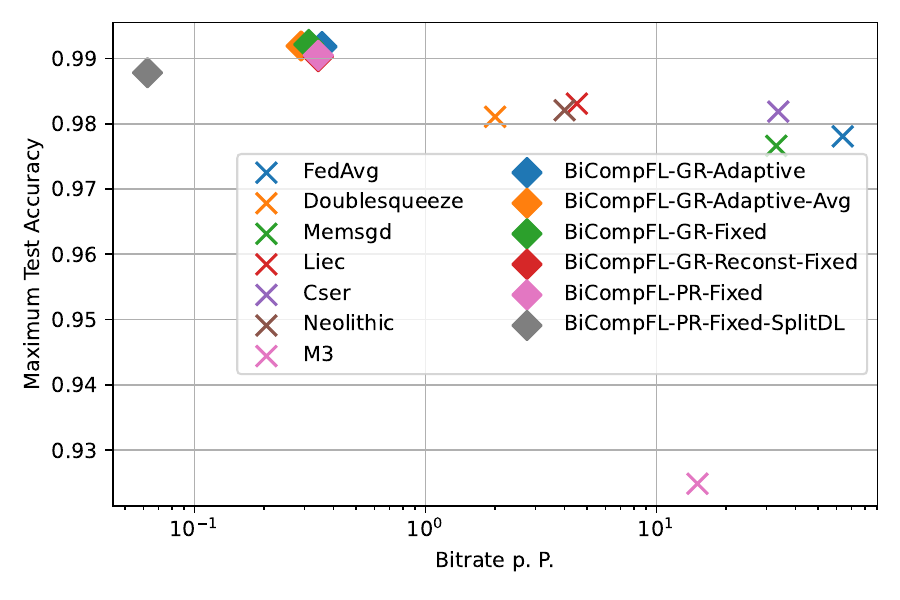}}
\vspace{-0.2cm}
\caption{MNIST LeNet \iid}
\vspace{-0.1cm}
\end{figure}

For LeNet5 on MNIST, it can be observed that all our proposed methods converge significantly faster to satisfying accuracies with respect to the communication cost, while achieving higher maximum accuracies after $200$ epochs than the non-stochastic baselines. Partitioning the model on the downlink can help to further reduce the communication cost with only a minor loss in performance, especially in the \iid\ setting. For \noniid\ data distribution, the loss in performance is larger than for \iid\ distribution. However, at the beginning of the training, the model improves faster with respect to the communication cost than all other schemes. The bitrates are comparable for all our methods, with the exception of \alglocalrand-Fixed-SplitDL. Further, \algglobalrand-Reconst-Fixed does not suffer notable performance degradation from employing an additional MRC step (especially for \iid\ data allocation).

\begin{table}[H]
    \centering
    \caption{MNIST LeNet \iid}
    \begin{tabular}{|l|c|c|c|c|c|}
\toprule
Method & Acc (mean ± std) & bpp & bpp (BC) & Uplink & Downlink \\
\midrule
FedAvg & 0.978 $\pm$ 0.1 & 64.0 & 35.0 & 32.0 & 32.0 \\
Doublesqueeze & 0.981 $\pm$ 0.1 & 2.0 & 1.1 & 1.0 & 1.0 \\
Memsgd & 0.977 $\pm$ 0.1 & 33.0 & 4.2 & 1.0 & 32.0 \\
Liec & 0.983 $\pm$ 0.1 & 4.5 & 2.5 & 2.3 & 2.3 \\
Cser & 0.982 $\pm$ 0.09 & 34.0 & 4.3 & 1.0 & 33.0 \\
Neolithic & 0.982 $\pm$ 0.1 & 4.0 & 2.2 & 2.0 & 2.0 \\
M3 & 0.925 $\pm$ 0.2 & 15.0 & 2.2 & 8.0 & 7.1 \\
BiCompFL-GR-Adaptive & \textbf{0.992 $\pm$ 0.0006} & 0.36 & 0.068 & 0.036 & 0.32 \\
BiCompFL-GR-Adaptive-Avg & 0.992 $\pm$ 0.0003 & 0.29 & \textbf{0.055} & \textbf{0.029} & 0.26 \\
BiCompFL-GR-Fixed & 0.992 $\pm$ 0.0002 & 0.31 & 0.059 & 0.031 & 0.28 \\
BiCompFL-GR-Reconst-Fixed & 0.99 $\pm$ 0.0002 & 0.34 & 0.063 & 0.031 & 0.31 \\
BiCompFL-PR-Fixed & 0.99 $\pm$ 0.0004 & 0.34 & 0.34 & 0.031 & 0.31 \\
BiCompFL-PR-Fixed-SplitDL & 0.988 $\pm$ 0.0009 & \textbf{0.063} & 0.063 & 0.031 & \textbf{0.031} \\
\bottomrule
\end{tabular}

    \label{tab:my_label}
\end{table}

\begin{figure}[H]
\subfigure[Test Accuracy over Communication]{\includegraphics[width = .5\linewidth]{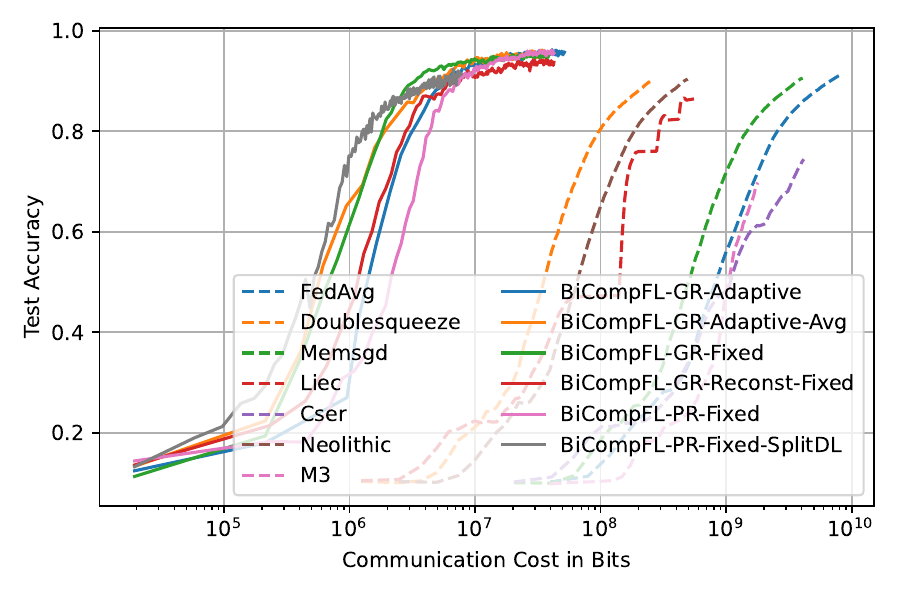}}
\subfigure[Test Accuracy over Bitrate]{\includegraphics[width = .5\linewidth]{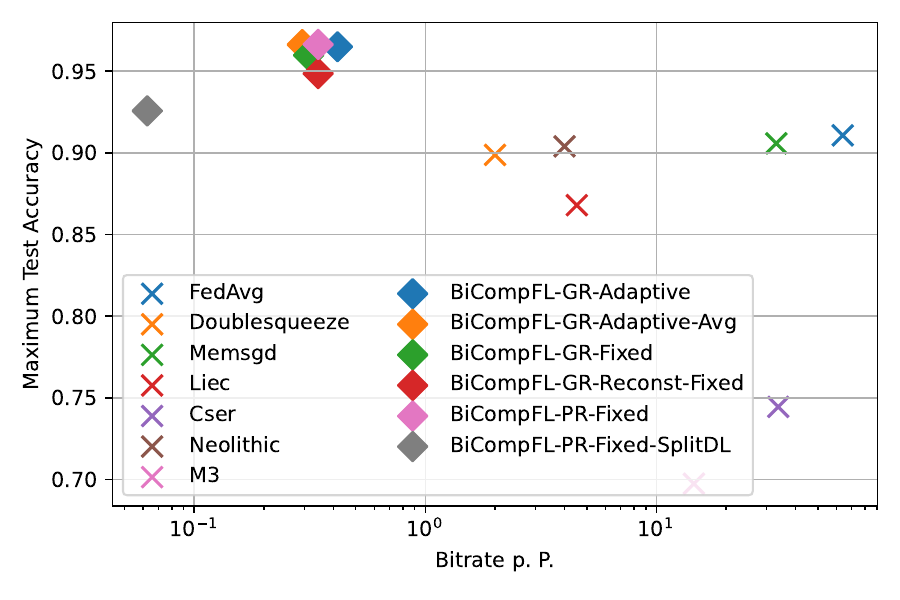}}
\caption{MNIST LeNet \noniid}
\end{figure}

\begin{table}[H]
    \centering
    \caption{MNIST LeNet \noniid}
    \begin{tabular}{|l|c|c|c|c|c|}
\toprule
Method & Acc (mean ± std) & bpp & bpp (BC) & Uplink & Downlink \\
\midrule
FedAvg & 0.911 $\pm$ 0.2 & 64.0 & 35.0 & 32.0 & 32.0 \\
Doublesqueeze & 0.899 $\pm$ 0.2 & 2.0 & 1.1 & 1.0 & 1.0 \\
Memsgd & 0.906 $\pm$ 0.2 & 33.0 & 4.2 & 1.0 & 32.0 \\
Liec & 0.866 $\pm$ 0.2 & 4.5 & 2.5 & 2.3 & 2.3 \\
Cser & 0.744 $\pm$ 0.2 & 34.0 & 4.3 & 1.0 & 33.0 \\
Neolithic & 0.904 $\pm$ 0.2 & 4.0 & 2.2 & 2.0 & 2.0 \\
M3 & 0.697 $\pm$ 0.2 & 15.0 & 2.2 & 7.3 & 7.2 \\
BiCompFL-GR-Adaptive & 0.965 $\pm$ 0.02 & 0.42 & 0.079 & 0.042 & 0.37 \\
BiCompFL-GR-Adaptive-Avg & \textbf{0.966 $\pm$ 0.02} & 0.29 & \textbf{0.056} & \textbf{0.029} & 0.26 \\
BiCompFL-GR-Fixed & 0.96 $\pm$ 0.03 & 0.31 & 0.059 & 0.031 & 0.28 \\
BiCompFL-GR-Reconst-Fixed & 0.949 $\pm$ 0.03 & 0.34 & 0.063 & 0.031 & 0.31 \\
BiCompFL-PR-Fixed & 0.966 $\pm$ 0.02 & 0.34 & 0.34 & 0.031 & 0.31 \\
BiCompFL-PR-Fixed-SplitDL & 0.926 $\pm$ 0.04 & \textbf{0.063} & 0.063 & 0.031 & \textbf{0.031} \\
\bottomrule
\end{tabular}

    \label{tab:my_label}
\end{table}

For 4CNN trained on MNIST, the differences between the proposed approaches become more visible. In the \iid\ setting, we can observe that the adaptive block allocations (both Adaptive and Adaptive-Avg) can drastically reduce the average bitrate in \algglobalrand. Partitioning the model in the downlink (\alglocalrand-Fixed-SplitDL) improves the accuracy over bitrate significantly compared to \alglocalrand-Fixed.

\begin{figure}[H]
\subfigure[Test Accuracy over Communication]{\includegraphics[width = .5\linewidth]{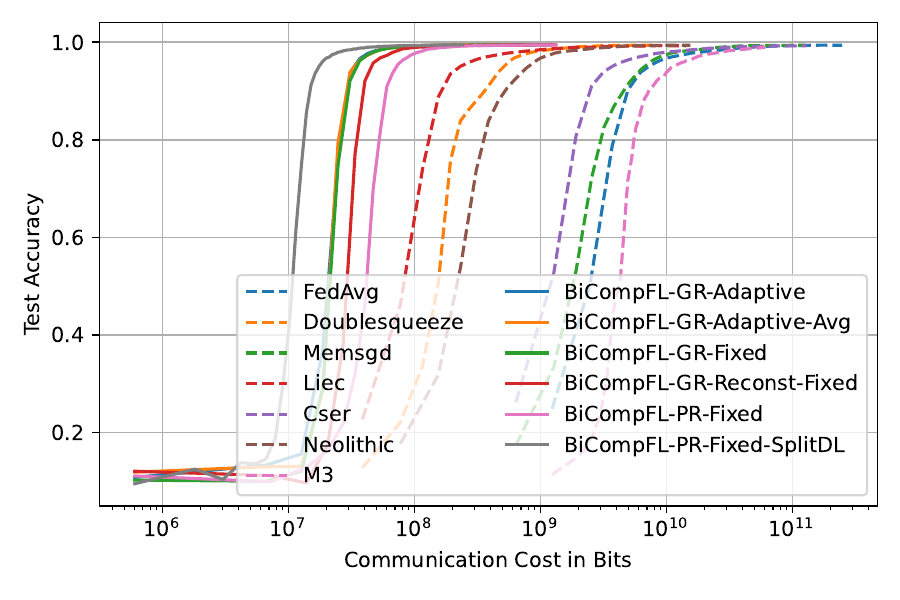}}
\subfigure[Test Accuracy over Bitrate]{\includegraphics[width = .5\linewidth]{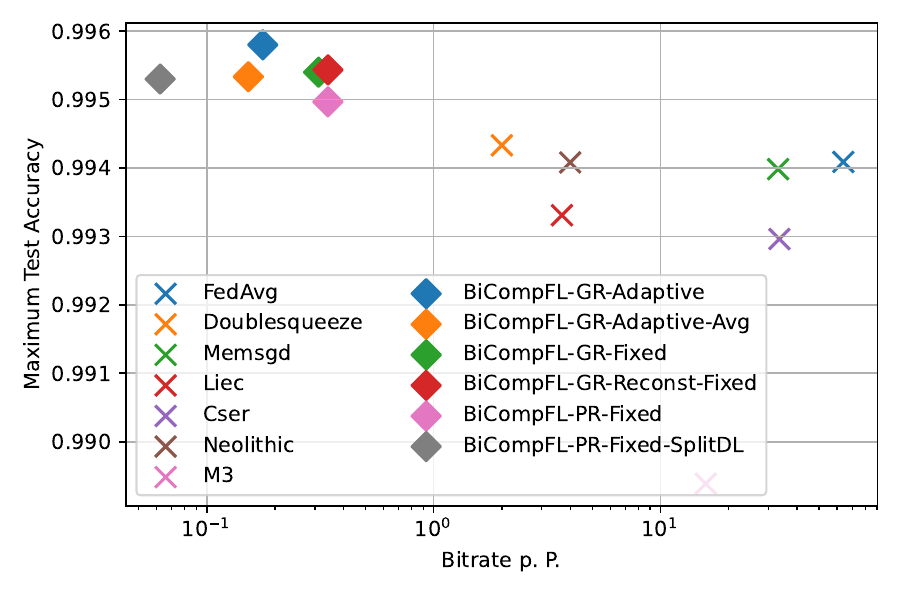}}
\caption{MNIST 4CNN \iid}
\end{figure}

\begin{table}[H]
    \centering
    \caption{MNIST 4CNN \iid}
    \begin{tabular}{|l|c|c|c|c|c|}
\toprule
Method & Acc (mean ± std) & bpp & bpp (BC) & Uplink & Downlink \\
\midrule
FedAvg & 0.994 $\pm$ 0.06 & 64.0 & 35.0 & 32.0 & 32.0 \\
Doublesqueeze & 0.994 $\pm$ 0.1 & 2.0 & 1.1 & 1.0 & 1.0 \\
Memsgd & 0.994 $\pm$ 0.08 & 33.0 & 4.2 & 1.0 & 32.0 \\
Liec & 0.993 $\pm$ 0.07 & 3.7 & 2.0 & 1.8 & 1.8 \\
Cser & 0.993 $\pm$ 0.06 & 33.0 & 4.3 & 1.0 & 32.0 \\
Neolithic & 0.994 $\pm$ 0.08 & 4.0 & 2.2 & 2.0 & 2.0 \\
M3 & 0.989 $\pm$ 0.2 & 16.0 & 2.2 & 8.4 & 7.4 \\
BiCompFL-GR-Adaptive & \textbf{0.996 $\pm$ 0.0001} & 0.18 & 0.034 & 0.018 & 0.16 \\
BiCompFL-GR-Adaptive-Avg & 0.995 $\pm$ 0.0001 & 0.15 & \textbf{0.029} & \textbf{0.015} & 0.14 \\
BiCompFL-GR-Fixed & 0.995 $\pm$ 0.0002 & 0.31 & 0.059 & 0.031 & 0.28 \\
BiCompFL-GR-Reconst-Fixed & 0.995 $\pm$ 0.0001 & 0.34 & 0.062 & 0.031 & 0.31 \\
BiCompFL-PR-Fixed & 0.995 $\pm$ 0.0002 & 0.34 & 0.34 & 0.031 & 0.31 \\
BiCompFL-PR-Fixed-SplitDL & 0.995 $\pm$ 0.0002 & \textbf{0.062} & 0.062 & 0.031 & \textbf{0.031} \\
\bottomrule
\end{tabular}

    \label{tab:my_label}
\end{table}

\begin{figure}[H]
\subfigure[Test Accuracy over Communication]{\includegraphics[width = .5\linewidth]{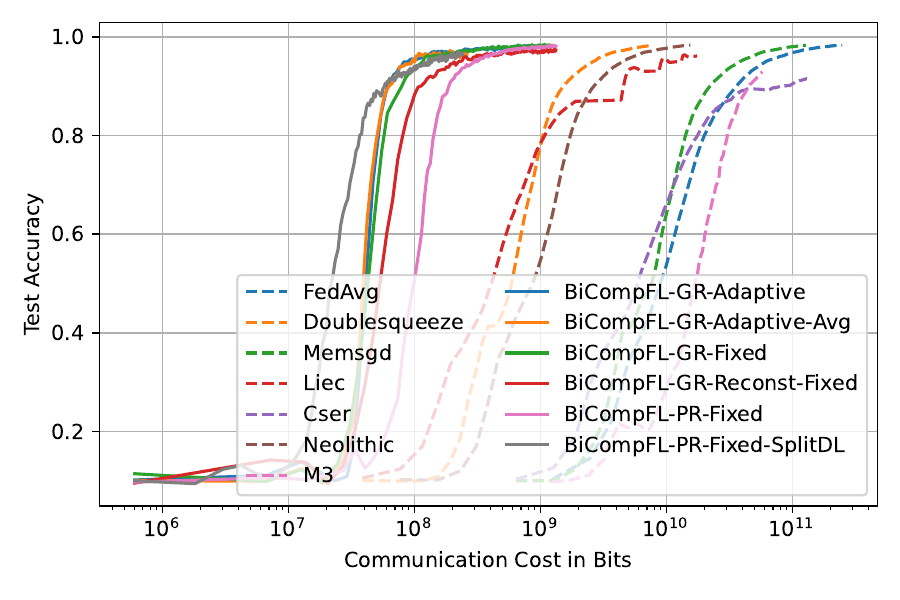}}
\subfigure[Test Accuracy over Bitrate]{\includegraphics[width = .5\linewidth]{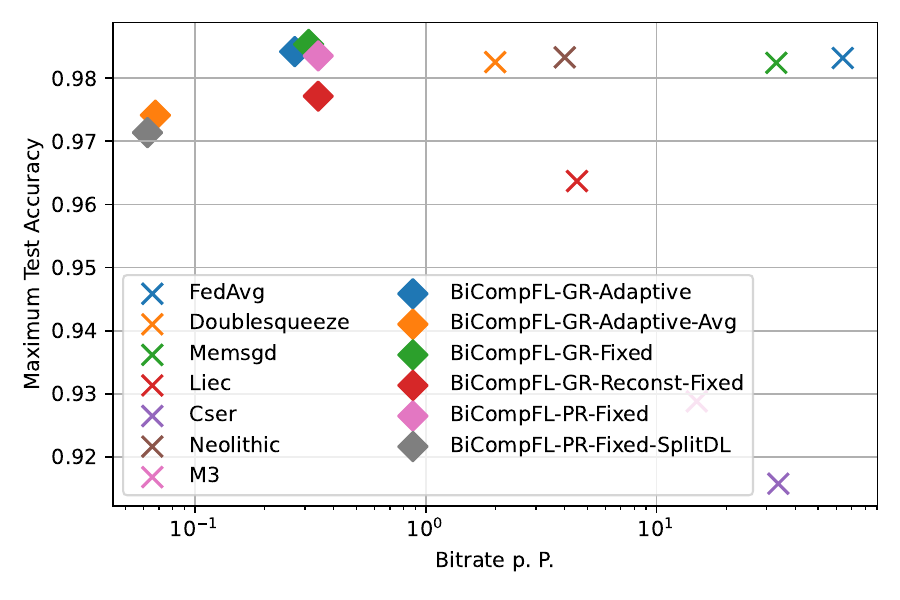}}
\caption{MNIST 4CNN \noniid}
\end{figure}

In the \noniid\ case of 4CNN on MNIST, the adaptive average allocation strategy provides a significant reduction in the bitrate for \algglobalrand, with similar loss in the accuracy as SplitDL for \alglocalrand. In this setting, it is also apparent that the reconstruction in \algglobalrand\ degrades the performance without gains in the bitrate compared to the proposed \cref{alg:globalrand}.

\begin{table}[H]
    \centering
    \caption{MNIST 4CNN \noniid}
    \begin{tabular}{|l|c|c|c|c|c|}
\toprule
Method & Acc (mean ± std) & bpp & bpp (BC) & Uplink & Downlink \\
\midrule
FedAvg & 0.983 $\pm$ 0.1 & 64.0 & 35.0 & 32.0 & 32.0 \\
Doublesqueeze & 0.982 $\pm$ 0.2 & 2.0 & 1.1 & 1.0 & 1.0 \\
Memsgd & 0.982 $\pm$ 0.2 & 33.0 & 4.2 & 1.0 & 32.0 \\
Liec & 0.963 $\pm$ 0.2 & 4.5 & 2.5 & 2.3 & 2.3 \\
Cser & 0.915 $\pm$ 0.1 & 34.0 & 4.3 & 1.0 & 33.0 \\
Neolithic & 0.983 $\pm$ 0.2 & 4.0 & 2.2 & 2.0 & 2.0 \\
M3 & 0.929 $\pm$ 0.3 & 15.0 & 2.2 & 7.8 & 7.1 \\
BiCompFL-GR-Adaptive & 0.984 $\pm$ 0.009 & 0.27 & 0.051 & 0.026 & 0.24 \\
BiCompFL-GR-Adaptive-Avg & 0.974 $\pm$ 0.02 & 0.067 & \textbf{0.013} & \textbf{0.0068} & 0.061 \\
BiCompFL-GR-Fixed & \textbf{0.985 $\pm$ 0.008} & 0.31 & 0.059 & 0.031 & 0.28 \\
BiCompFL-GR-Reconst-Fixed & 0.977 $\pm$ 0.01 & 0.34 & 0.062 & 0.031 & 0.31 \\
BiCompFL-PR-Fixed & 0.984 $\pm$ 0.009 & 0.34 & 0.34 & 0.031 & 0.31 \\
BiCompFL-PR-Fixed-SplitDL & 0.971 $\pm$ 0.02 & \textbf{0.062} & 0.062 & 0.031 & \textbf{0.031} \\
\bottomrule
\end{tabular}

    \label{tab:my_label}
\end{table}

\begin{figure}[H]
\subfigure[Test Accuracy over Communication]{\includegraphics[width = .5\linewidth]{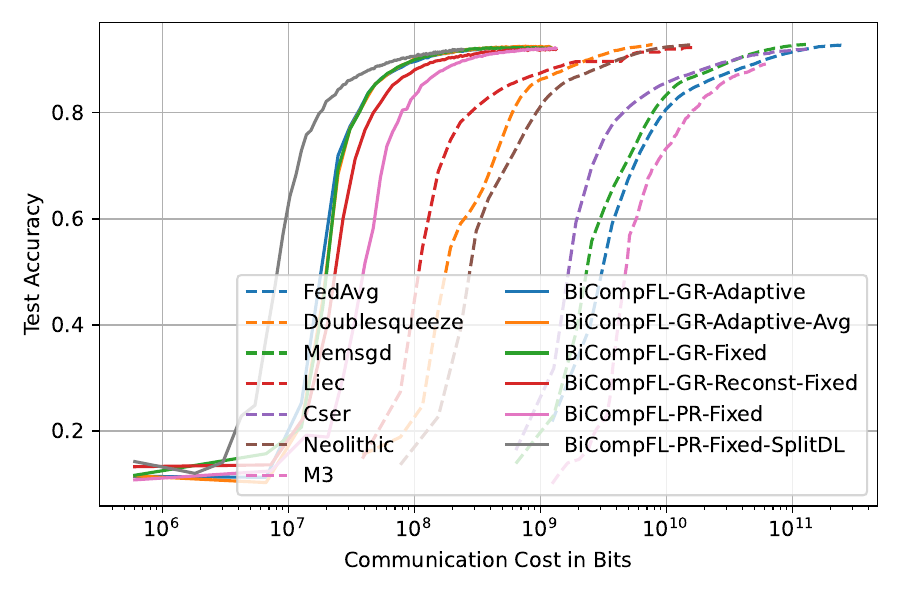}}
\subfigure[Test Accuracy over Bitrate]{\includegraphics[width = .5\linewidth]{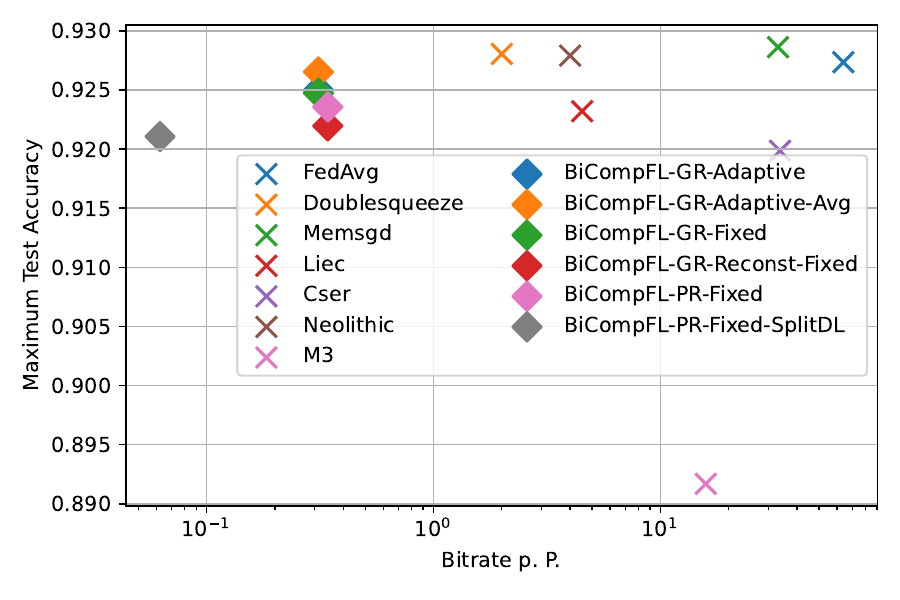}}
\caption{Fashion MNIST 4CNN \iid}
\end{figure}

\begin{table}[H]
    \centering
    \caption{Fashion MNIST 4CNN \iid}
    \begin{tabular}{|l|c|c|c|c|c|}
\toprule
Method & Acc (mean ± std) & bpp & bpp (BC) & Uplink & Downlink \\
\midrule
FedAvg & 0.927 $\pm$ 0.07 & 64.0 & 35.0 & 32.0 & 32.0 \\
Doublesqueeze & \textbf{0.928 $\pm$ 0.1} & 2.0 & 1.1 & 1.0 & 1.0 \\
Memsgd & 0.928 $\pm$ 0.09 & 33.0 & 4.2 & 1.0 & 32.0 \\
Liec & 0.923 $\pm$ 0.08 & 4.5 & 2.5 & 2.3 & 2.3 \\
Cser & 0.92 $\pm$ 0.08 & 34.0 & 4.3 & 1.0 & 33.0 \\
Neolithic & 0.928 $\pm$ 0.09 & 4.0 & 2.2 & 2.0 & 2.0 \\
M3 & 0.892 $\pm$ 0.2 & 16.0 & 2.2 & 8.3 & 7.6 \\
BiCompFL-GR-Adaptive & 0.925 $\pm$ 0.001 & 0.31 & \textbf{0.059} & \textbf{0.031} & 0.28 \\
BiCompFL-GR-Adaptive-Avg & 0.927 $\pm$ 0.0007 & 0.31 & \textbf{0.059} & \textbf{0.031} & 0.28 \\
BiCompFL-GR-Fixed & 0.925 $\pm$ 0.0007 & 0.31 & \textbf{0.059} & \textbf{0.031} & 0.28 \\
BiCompFL-GR-Reconst-Fixed & 0.922 $\pm$ 0.001 & 0.34 & 0.062 & \textbf{0.031} & 0.31 \\
BiCompFL-PR-Fixed & 0.924 $\pm$ 0.002 & 0.34 & 0.34 & \textbf{0.031} & 0.31 \\
BiCompFL-PR-Fixed-SplitDL & 0.921 $\pm$ 0.002 & \textbf{0.062} & 0.062 & \textbf{0.031} & \textbf{0.031} \\
\bottomrule
\end{tabular}

    \label{tab:my_label}
\end{table}

\begin{figure}[H]
\subfigure[Test Accuracy over Communication]{\includegraphics[width = .5\linewidth]{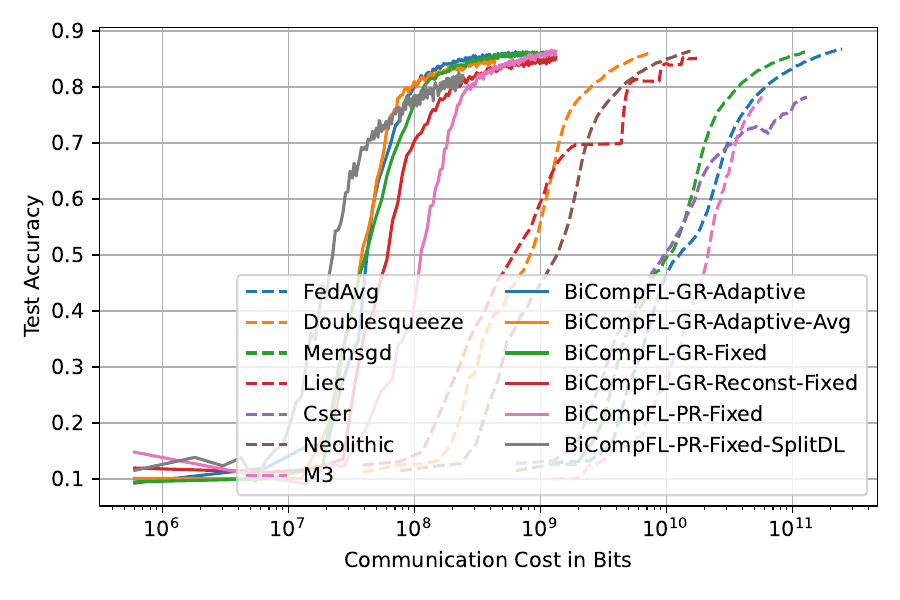}}
\subfigure[Test Accuracy over Bitrate]{\includegraphics[width = .5\linewidth]{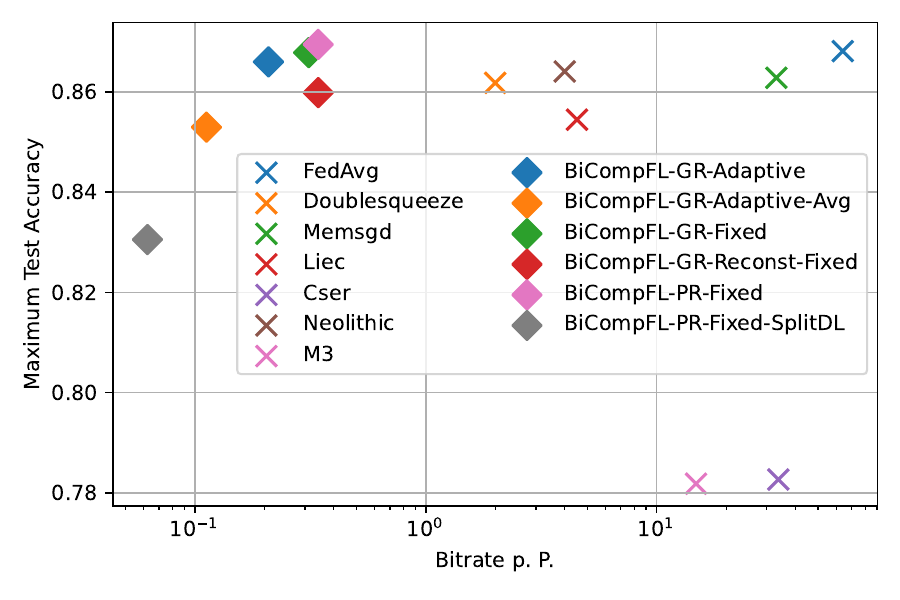}}
\caption{Fashion MNIST 4CNN \noniid}
\end{figure}

The results for Fashion MNIST are similar compared to the MNIST case. However, it becomes clear that \alglocalrand\ can significantly suffer from the unavailability of shared randomness in terms of the achieved accuracy when data is highly heterogeneous.

\begin{table}[H]
    \centering
    \caption{Fashion MNIST 4CNN \noniid}
    \begin{tabular}{|l|c|c|c|c|c|}
\toprule
Method & Acc (mean ± std) & bpp & bpp (BC) & Uplink & Downlink \\
\midrule
FedAvg & 0.867 $\pm$ 0.1 & 64.0 & 35.0 & 32.0 & 32.0 \\
Doublesqueeze & 0.861 $\pm$ 0.2 & 2.0 & 1.1 & 1.0 & 1.0 \\
Memsgd & 0.863 $\pm$ 0.2 & 33.0 & 4.2 & 1.0 & 32.0 \\
Liec & 0.853 $\pm$ 0.1 & 4.5 & 2.5 & 2.3 & 2.3 \\
Cser & 0.781 $\pm$ 0.1 & 34.0 & 4.3 & 1.0 & 33.0 \\
Neolithic & 0.864 $\pm$ 0.2 & 4.0 & 2.2 & 2.0 & 2.0 \\
M3 & 0.782 $\pm$ 0.2 & 15.0 & 2.2 & 8.0 & 6.9 \\
BiCompFL-GR-Adaptive & 0.866 $\pm$ 0.03 & 0.21 & 0.04 & 0.021 & 0.19 \\
BiCompFL-GR-Adaptive-Avg & 0.853 $\pm$ 0.04 & 0.11 & \textbf{0.021} & \textbf{0.011} & 0.1 \\
BiCompFL-GR-Fixed & 0.868 $\pm$ 0.03 & 0.31 & 0.059 & 0.031 & 0.28 \\
BiCompFL-GR-Reconst-Fixed & 0.86 $\pm$ 0.02 & 0.34 & 0.062 & 0.031 & 0.31 \\
BiCompFL-PR-Fixed & \textbf{0.869 $\pm$ 0.03} & 0.34 & 0.34 & 0.031 & 0.31 \\
BiCompFL-PR-Fixed-SplitDL & 0.831 $\pm$ 0.03 & \textbf{0.062} & 0.062 & 0.031 & \textbf{0.031} \\
\bottomrule
\end{tabular}

    \label{tab:my_label}
\end{table}

For 6CNN trained on CIFAR-10, the negative effects of missing global shared randomness and reconstructing in the case of \algglobalrand\ are prominent. For \noniid\ data distributions, the adaptive average allocation shows improvements over the fixed or the average block allocation. Partitioning the model is not a viable option in this setting, especially under \noniid\ data.

\begin{figure}[H]
\subfigure[Test Accuracy over Communication]{\includegraphics[width = .5\linewidth]{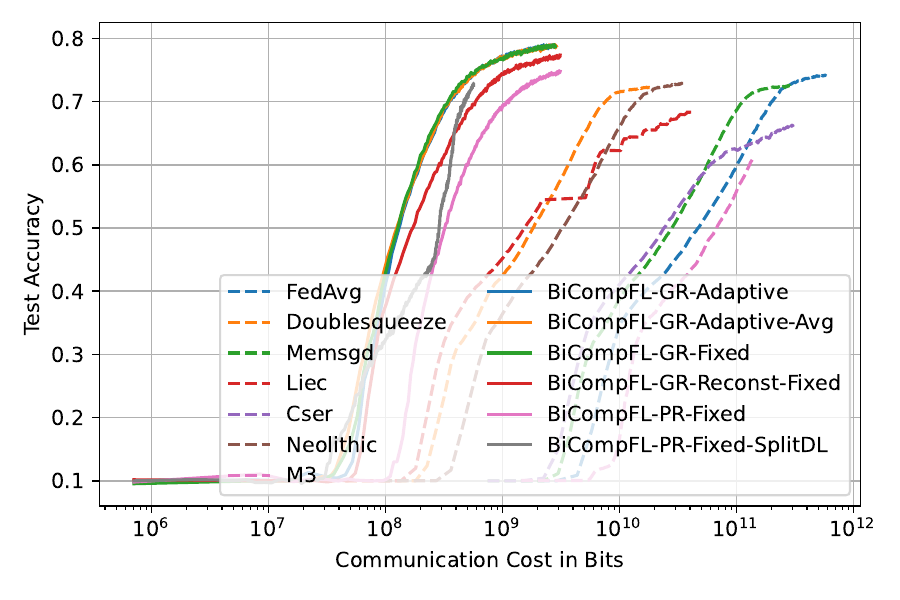}}
\subfigure[Test Accuracy over Bitrate]{\includegraphics[width = .5\linewidth]{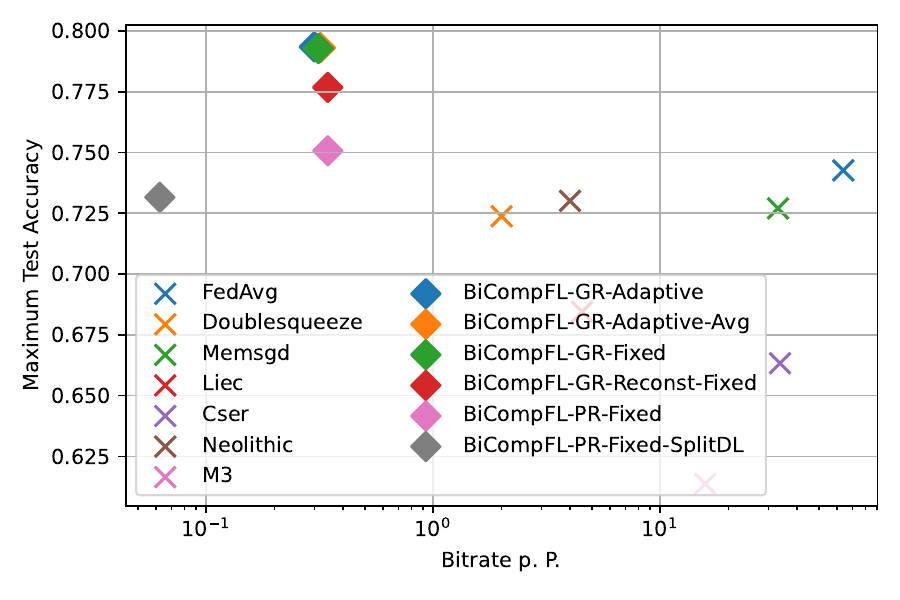}}
\caption{CIFAR-10 6CNN \iid}
\end{figure}

\begin{table}[H]
    \centering
    \caption{CIFAR-10 6CNN \iid}
    \begin{tabular}{|l|c|c|c|c|c|}
\toprule
Method & Acc (mean ± std) & bpp & bpp (BC) & Uplink & Downlink \\
\midrule
FedAvg & 0.742 $\pm$ 0.1 & 64.0 & 35.0 & 32.0 & 32.0 \\
Doublesqueeze & 0.723 $\pm$ 0.1 & 2.0 & 1.1 & 1.0 & 1.0 \\
Memsgd & 0.727 $\pm$ 0.1 & 33.0 & 4.2 & 1.0 & 32.0 \\
Liec & 0.684 $\pm$ 0.09 & 4.5 & 2.5 & 2.3 & 2.3 \\
Cser & 0.663 $\pm$ 0.08 & 34.0 & 4.3 & 1.0 & 33.0 \\
Neolithic & 0.73 $\pm$ 0.1 & 4.0 & 2.2 & 2.0 & 2.0 \\
M3 & 0.614 $\pm$ 0.1 & 16.0 & 2.2 & 8.3 & 7.5 \\
BiCompFL-GR-Adaptive & \textbf{0.793 $\pm$ 0.002} & 0.3 & \textbf{0.057} & \textbf{0.03} & 0.27 \\
BiCompFL-GR-Adaptive-Avg & 0.793 $\pm$ 0.002 & 0.32 & 0.061 & 0.032 & 0.29 \\
BiCompFL-GR-Fixed & 0.793 $\pm$ 0.004 & 0.31 & 0.059 & 0.031 & 0.28 \\
BiCompFL-GR-Reconst-Fixed & 0.777 $\pm$ 0.002 & 0.34 & 0.062 & 0.031 & 0.31 \\
BiCompFL-PR-Fixed & 0.751 $\pm$ 0.003 & 0.34 & 0.34 & 0.031 & 0.31 \\
BiCompFL-PR-Fixed-SplitDL & 0.732 $\pm$ 0.02 & \textbf{0.062} & 0.062 & 0.031 & \textbf{0.031} \\
\bottomrule
\end{tabular}

    \label{tab:my_label}
\end{table}

\begin{figure}[H]
\subfigure[Test Accuracy over Communication]{\includegraphics[width = .5\linewidth]{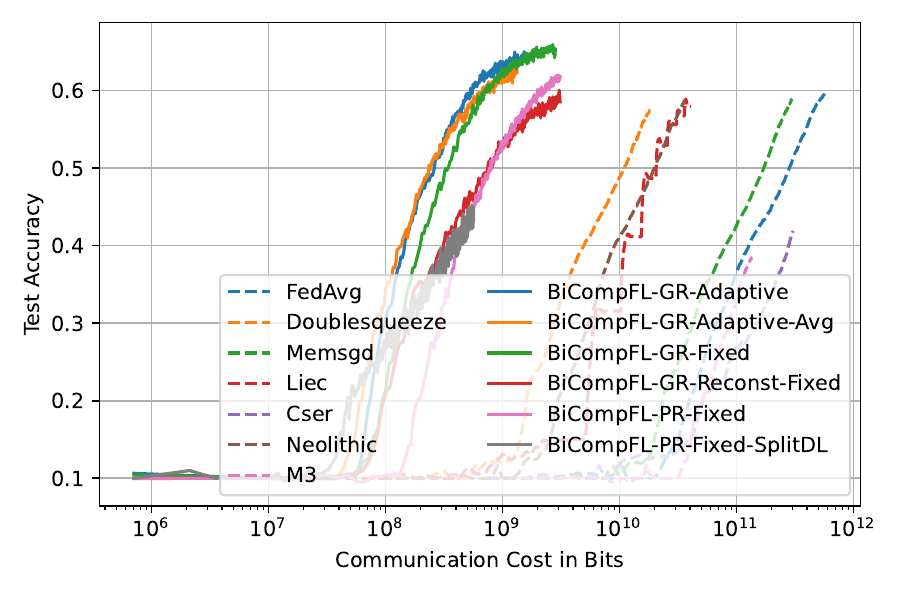}}
\subfigure[Test Accuracy over Bitrate]{\includegraphics[width = .5\linewidth]{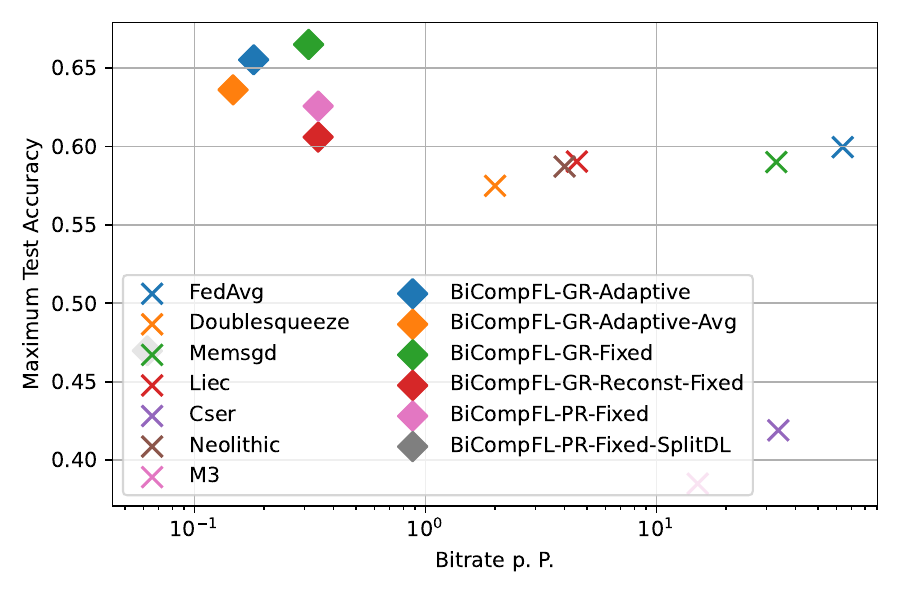}}
\caption{CIFAR-10 6CNN \noniid}
\end{figure}

\begin{table}[H]
    \centering
    \caption{CIFAR-10 6CNN \noniid}
    \begin{tabular}{|l|c|c|c|c|c|}
\toprule
Method & Acc (mean ± std) & bpp & bpp (BC) & Uplink & Downlink \\
\midrule
FedAvg & 0.599 $\pm$ 0.1 & 64.0 & 35.0 & 32.0 & 32.0 \\
Doublesqueeze & 0.575 $\pm$ 0.1 & 2.0 & 1.1 & 1.0 & 1.0 \\
Memsgd & 0.589 $\pm$ 0.1 & 33.0 & 4.2 & 1.0 & 32.0 \\
Liec & 0.589 $\pm$ 0.2 & 4.5 & 2.5 & 2.3 & 2.3 \\
Cser & 0.419 $\pm$ 0.09 & 34.0 & 4.3 & 1.0 & 33.0 \\
Neolithic & 0.587 $\pm$ 0.1 & 4.0 & 2.2 & 2.0 & 2.0 \\
M3 & 0.385 $\pm$ 0.1 & 15.0 & 2.2 & 8.3 & 6.7 \\
BiCompFL-GR-Adaptive & 0.655 $\pm$ 0.04 & 0.18 & 0.034 & 0.018 & 0.16 \\
BiCompFL-GR-Adaptive-Avg & 0.636 $\pm$ 0.05 & 0.15 & \textbf{0.028} & \textbf{0.015} & 0.13 \\
BiCompFL-GR-Fixed & \textbf{0.665 $\pm$ 0.03} & 0.31 & 0.059 & 0.031 & 0.28 \\
BiCompFL-GR-Reconst-Fixed & 0.606 $\pm$ 0.05 & 0.34 & 0.062 & 0.031 & 0.31 \\
BiCompFL-PR-Fixed & 0.626 $\pm$ 0.03 & 0.34 & 0.34 & 0.031 & 0.31 \\
BiCompFL-PR-Fixed-SplitDL & 0.47 $\pm$ 0.07 & \textbf{0.062} & 0.062 & 0.031 & \textbf{0.031} \\
\bottomrule
\end{tabular}

    \label{tab:my_label}
\end{table}

\rev{For completeness, we present in what follows the test accuracies over the number of trained epochs for all scenarios considered above. The setting of interest to this work is that of limited communication cost, and in particular, which performance is achievable given a fixed communication budget. Nonetheless, we can find that our proposed methods are not inferior in convergence speed over epochs compared to the baselines.}

\rev{
\begin{figure}[H]
\subfigure[MNIST LeNet \iid]{\includegraphics[width = .5\linewidth]{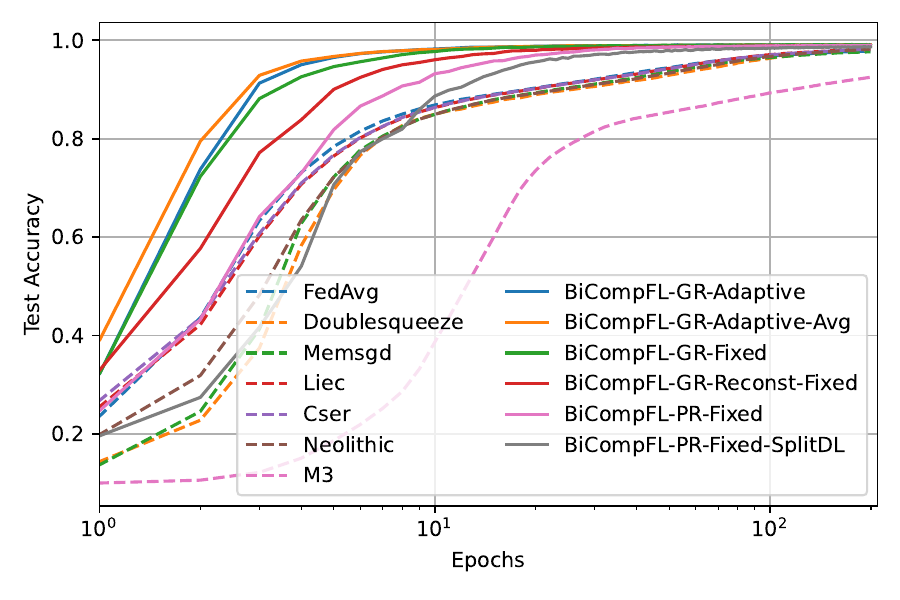}}
\subfigure[MNIST LeNet \noniid]{\includegraphics[width = .5\linewidth]{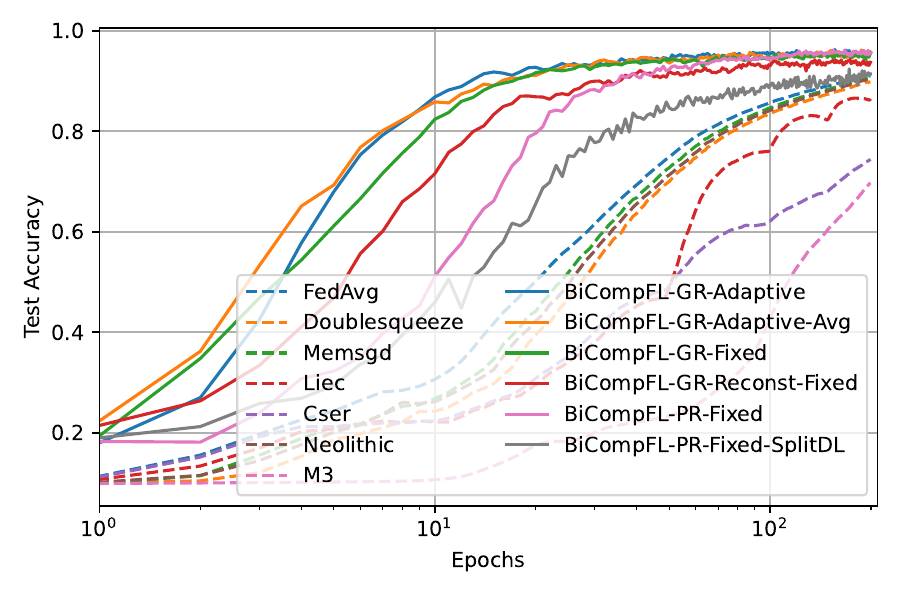}}
\subfigure[MNIST 4CNN \iid]{\includegraphics[width = .5\linewidth]{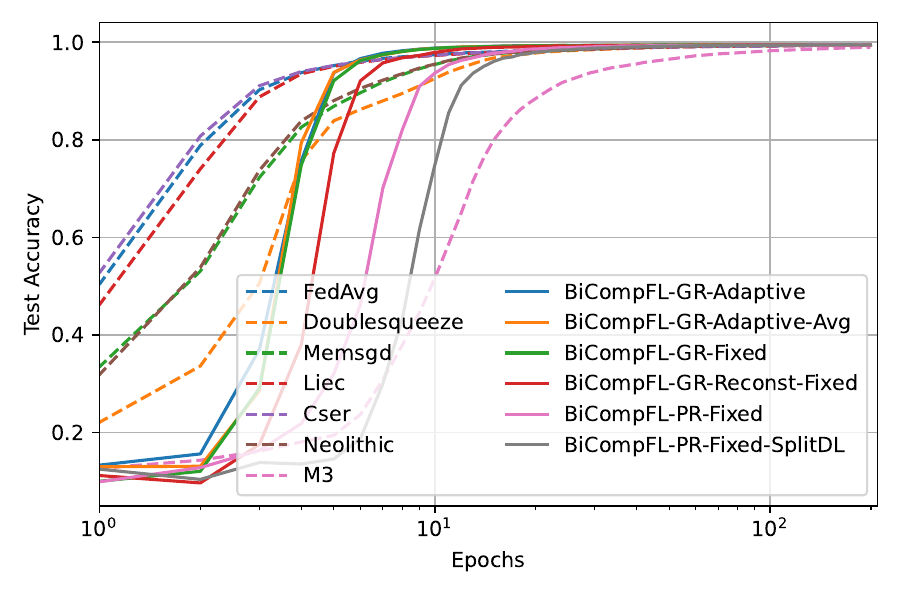}}
\subfigure[MNIST 4CNN \noniid]{\includegraphics[width = .5\linewidth]{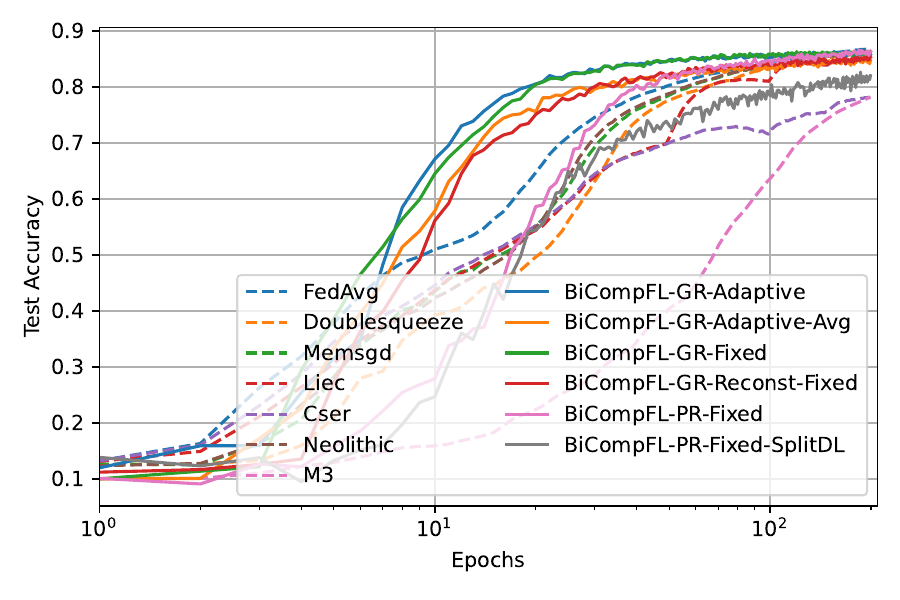}}
\subfigure[Fashion MNIST 4CNN \iid]{\includegraphics[width = .5\linewidth]{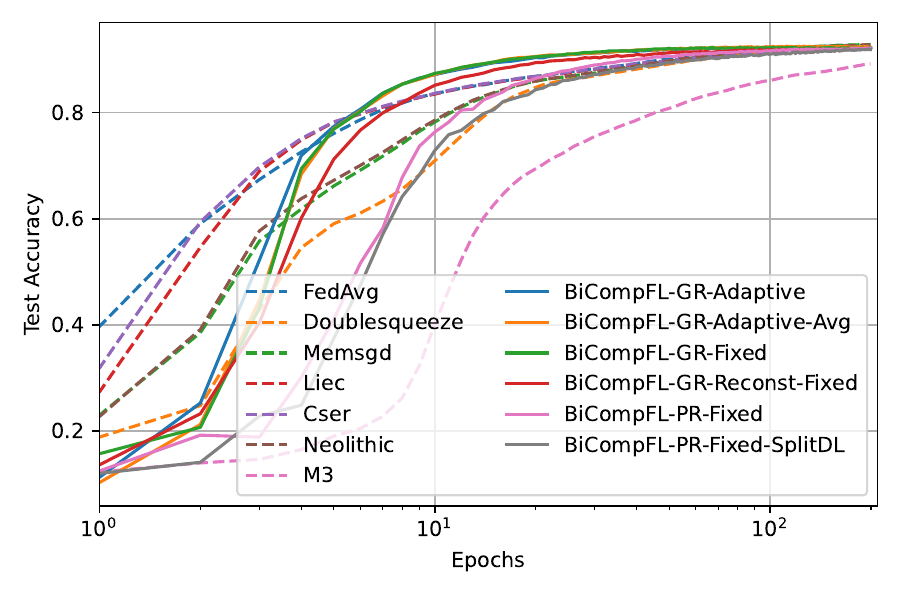}}
\subfigure[Fashion MNIST 4CNN \noniid]{\includegraphics[width = .5\linewidth]{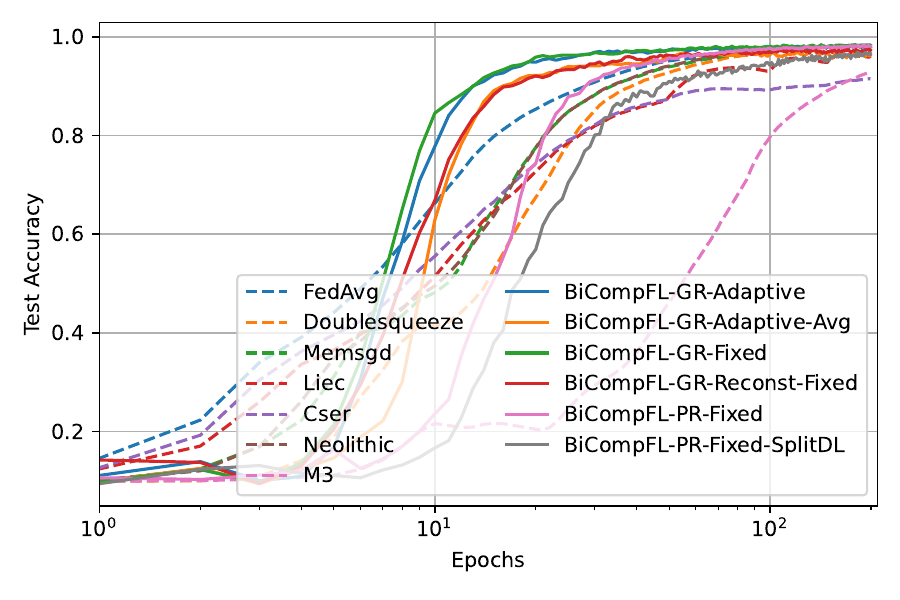}}
\subfigure[Cifar-10 6CNN \iid]{\includegraphics[width = .5\linewidth]{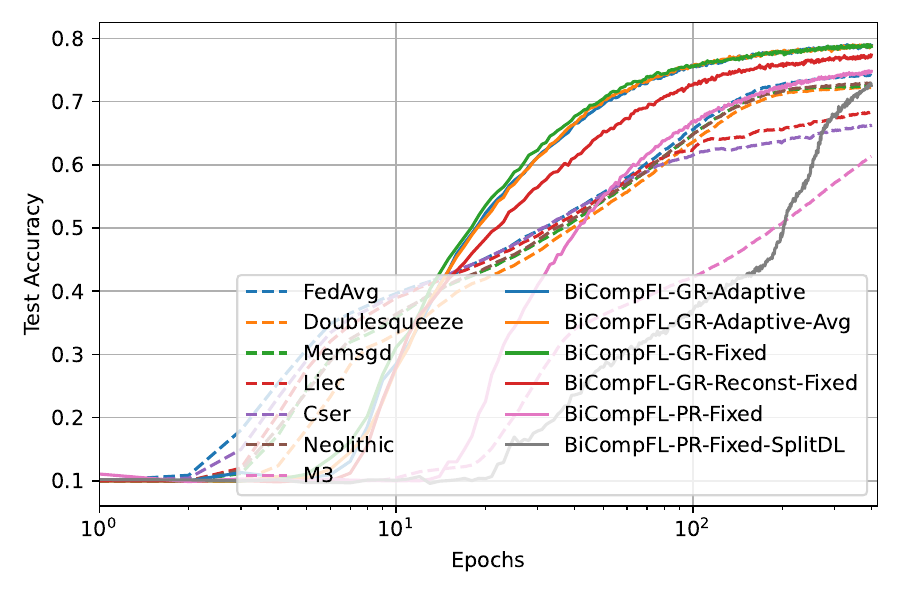}}
\subfigure[Cifar-10 6CNN \noniid]{\includegraphics[width = .5\linewidth]{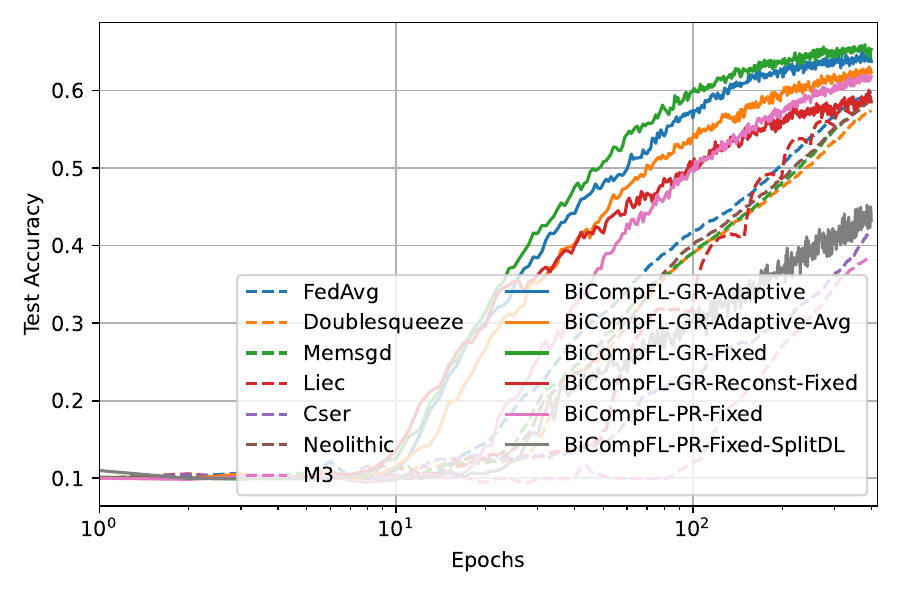}}
\caption{\rev{Test Accuracy over Epochs}}
\end{figure}}

\section{Ablation Studies} \label{sec:hyperparameters}

\subsection{Number of Clients} \label{sec:varclients}

\rev{
We study in what follows the sensitivity to various hyperparameters of our algorithms. For comparability, we conduct all experiments on the model 4CNN, Fashion MNIST, and \iid data. We plot for all experiments the accuracies over the number of epochs, and over the communication cost in bits.
}
\rev{
We first evaluate in \cref{fig:varclients} the effectiveness of \alglocalrand\ and \alglocalrand\ for different numbers of clients. It can be found that both algorithms exhibit satisfying performance even for $\nclients=50$, given that the same data is now distributed on more clients. The overall communication cost increases by roughly the factor of the increase in the number of $\nclients$. To illustrate this further, we additionally plot in \cref{fig:varclients_bitrates} the bitrates per parameter.
\begin{figure}[H]
\subfigure[Test Accuracy over Epochs]{\includegraphics[width = .5\linewidth]{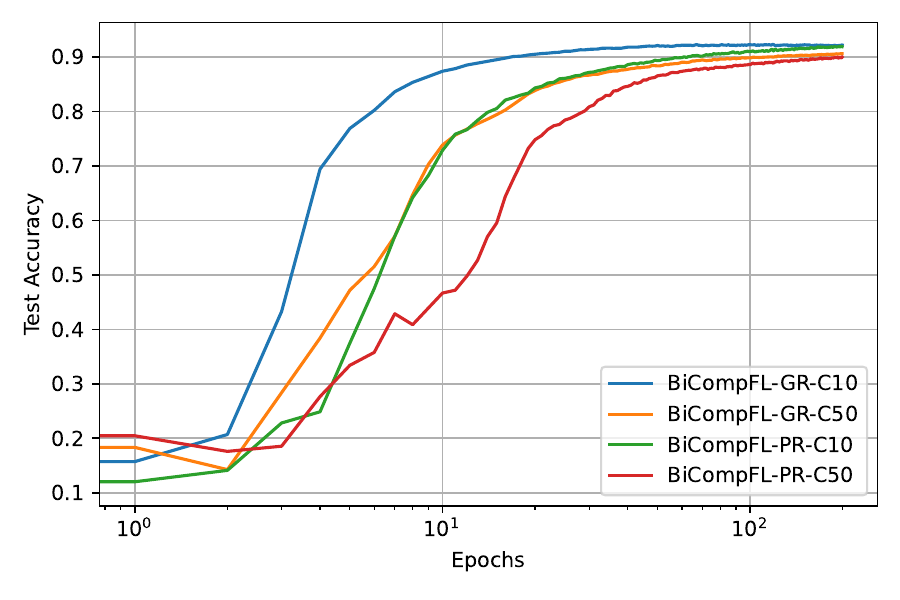}}
\subfigure[Test Accuracy over Communication Cost]{\includegraphics[width = .5\linewidth]{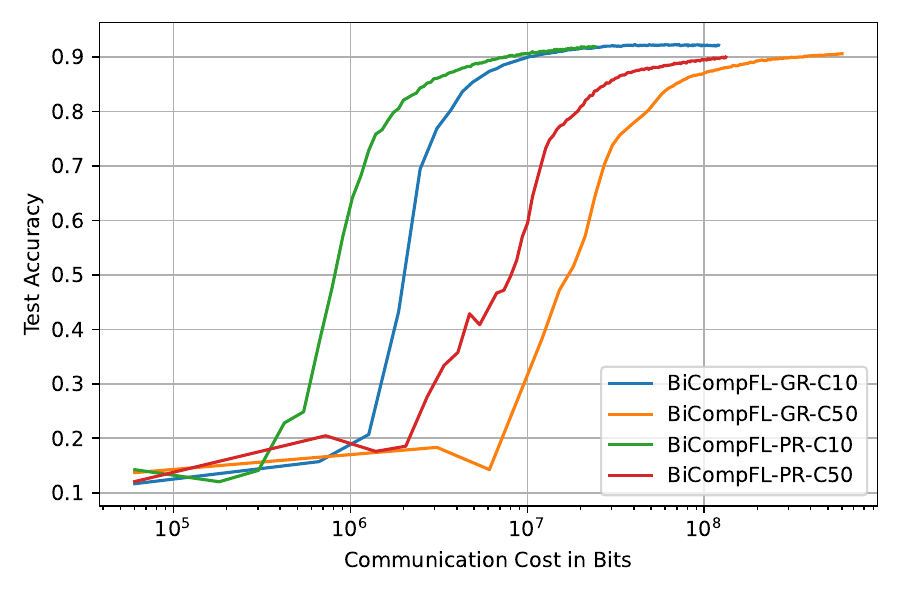}}
\caption{\rev{\algglobalrand\ and \algglobalrand\ With Different Number of Clients}}
\label{fig:varclients}
\end{figure}
\begin{figure}[H]
\centering
\includegraphics[width = .5\linewidth]{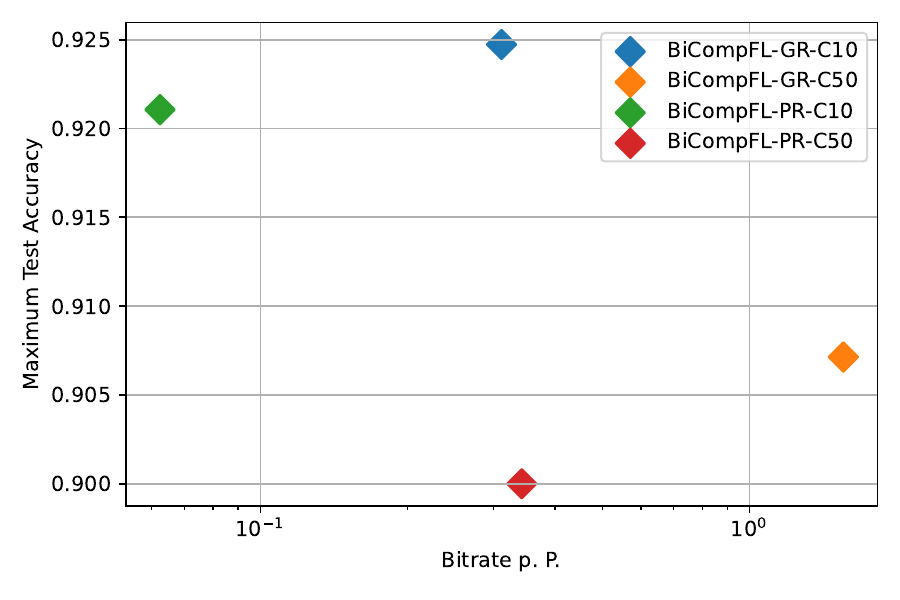}
\caption{\rev{Bitrates for \algglobalrand\ and \algglobalrand\ With Different Number of Clients}}
\label{fig:varclients_bitrates}
\end{figure}
}

\subsection{Optimization of the Prior} \label{sec:optimize_prior}

\rev{
As described in the main body of the paper, \alglocalrand\ allows for optimizing the choice of the prior at the clients by optimizing the convexity parameter $\lambda$ that mixes the global model estimate with the posterior transmitted by the client an iteration ahead, i.e., $\priorul = \lambda \modelest + (1-\lambda) \posteriorest[\epoch-1]$ to reduce the communication cost. To evaluate the potential of this method, we optimize $\lambda$ so that it minimized the KL-divergence between the current posterior $\posterior$ (to be transmitted) and the prior $\priorul$, representative for the uplink communication cost. The KL-minimizing $\lambda$ is transmitted to the federator, which is necessary for the federator to reconstruct the importance samples. This optimization is conducted at each iteration individually at the clients. We present in \cref{fig:varop} the performance of this method compared with the algorithms that use as priors exclusively the global model estimates of the clients. Note that optimizing the prior individually at the clients is only possible for \alglocalrand\. We plot the performance of \algglobalrand\ for reference only. To assess the potential, we ignore for the moment the cost of transmitting $\lambda$, which could be reduced by further compression techniques and leveraging the inter-round dependencies of the choice of $\lambda$.
\begin{figure}[H]
\subfigure[Test Accuracy over Epochs]{\includegraphics[width = .5\linewidth]{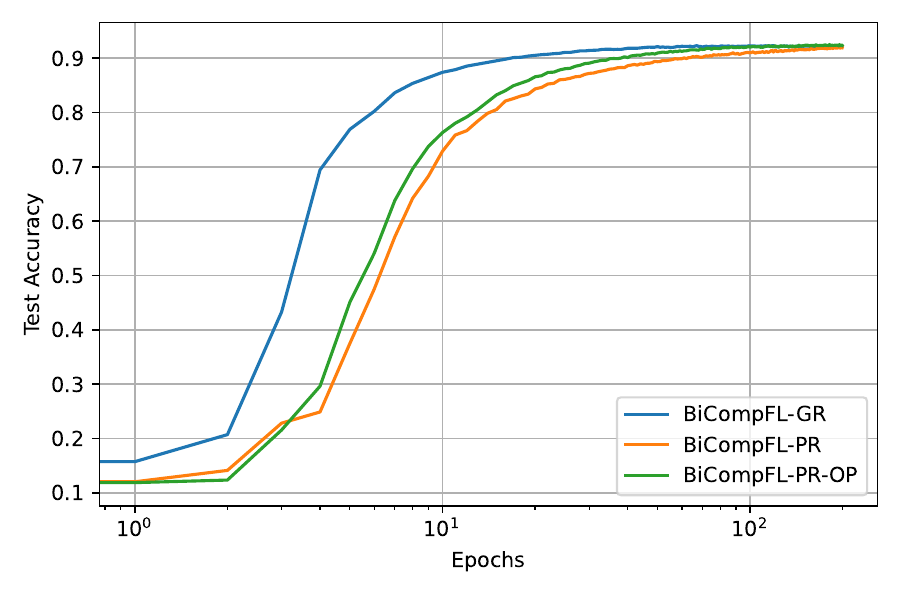}}
\subfigure[Test Accuracy over Communication Cost]{\includegraphics[width = .5\linewidth]{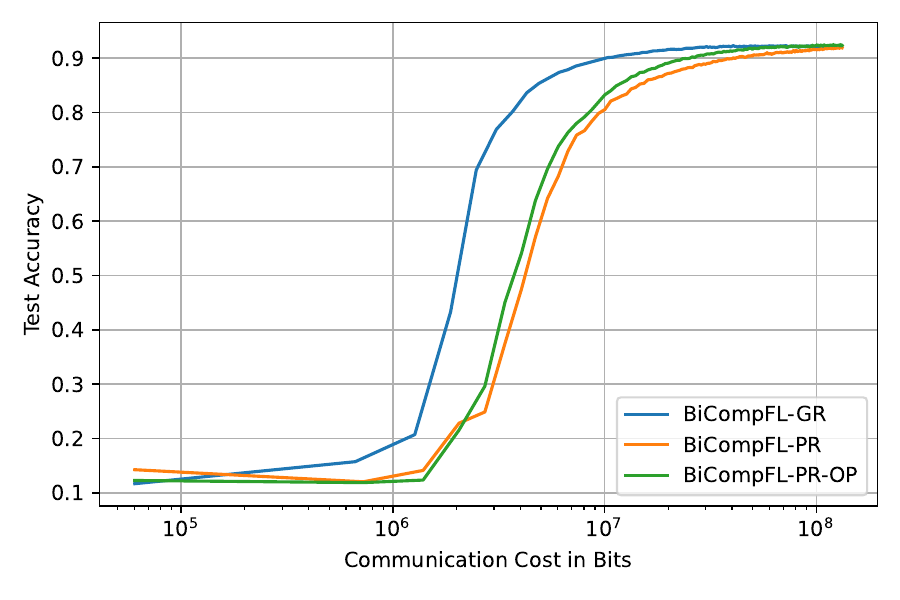}}
\caption{\rev{\alglocalrand\ With and Without Optimization over the Prior. Optimization over the Priors is denoted by OP.}}
\label{fig:varop}
\end{figure}
It can be found that, while optimizing the prior improves the accuracy over epochs and with respect to the communication cost compared to \alglocalrand\, the improvements are rather insignificant. We therefore present for clarity the algorithm with a fixed choice of the prior as the former global model estimate, which additionally reduce the computation overhead at the clients by avoiding the optimization over $\lambda$. Nonetheless, we note that in certain edge cases, there can be merit in the optimization approach, for instance when the number $\nmasksdl$ of samples on the downlink is very small, and hence the global model estimate is inaccurate.
}

\subsection{Number of Samples} \label{sec:varuldl}

\rev{We continue to assess the impact of the number $\nmasksdl$ of samples on the downlink. We therefore evaluate the performance of $\alglocalrand$ for $\nmasksdl \in \{5, 10, 20\}$. We evaluate the differences on \alglocalrand\. The results in \cref{fig:varuldl} reflect the obvious: the larger $\nmasksdl$, the better the accuracy when plotted over the number of epochs. On the contrary, the larger $\nmasksdl$, the larger the communication cost per epoch. The final accuracies do not show substantial differences, and hence, $\nmasksdl=5$ is sufficient in this setting. To avoid assessing our method overly optimistic and provide a fair comparison to other methods, we choose $\nmasksdl=10$ in all our experiments, noting that the communication can further be reduced in certain scenarios by lowering $\nmasksdl$ without notable performance loss.
\begin{figure}[H]
\subfigure[Test Accuracy over Epochs]{\includegraphics[width = .5\linewidth]{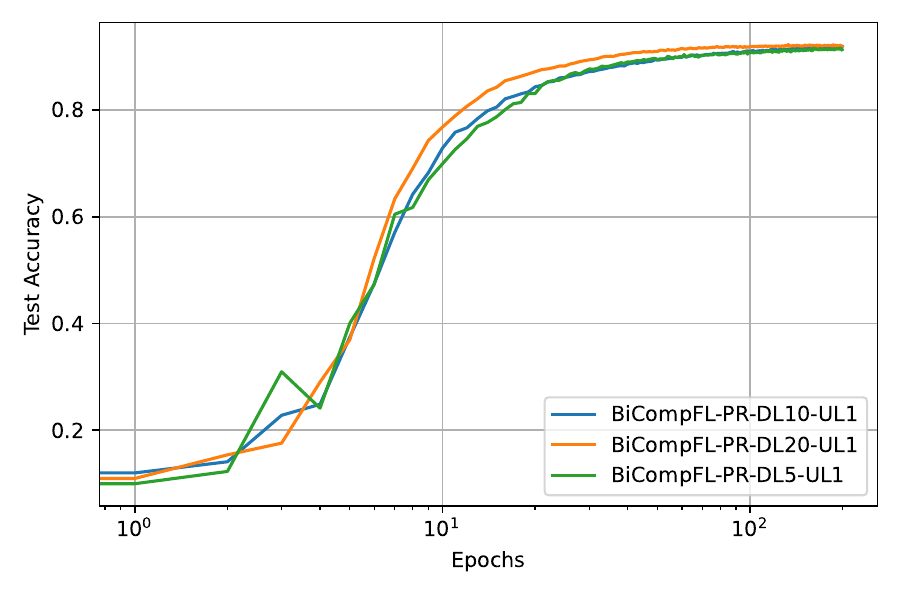}}
\subfigure[Test Accuracy over Communication Cost]{\includegraphics[width = .5\linewidth]{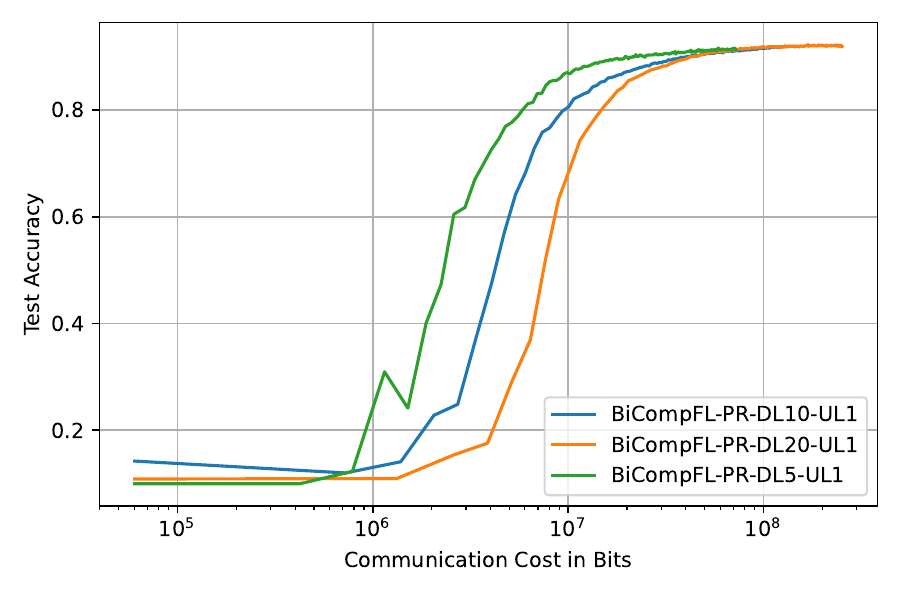}}
\caption{\rev{\alglocalrand\ for Different Number of Downlink Samples and a Single Uplink Sample.}}
\label{fig:varuldl}
\end{figure}
}

\subsection{Block Size} \label{sec:varbs}

\rev{
We compare in \cref{fig:varbs} the performance of \algglobalrand\ for different block sizes BS = $\dimension/\nblocks \in \{128, 256, 512\}$. As expected, fixing $\nisamples$, larger block sizes worsen the performance of the algorithm when evaluated over the number of epochs. However, larger block sizes simultaneously reduce the communication cost, and can hence be beneficial in many scenarios. However, we also note that larger block sizes comes at the expense of increases sampling complexities, and hence, the maximum block sizes are also dominated by the resources of the clients and the federator.
\begin{figure}[H]
\subfigure[Test Accuracy over Epochs]{\includegraphics[width = .5\linewidth]{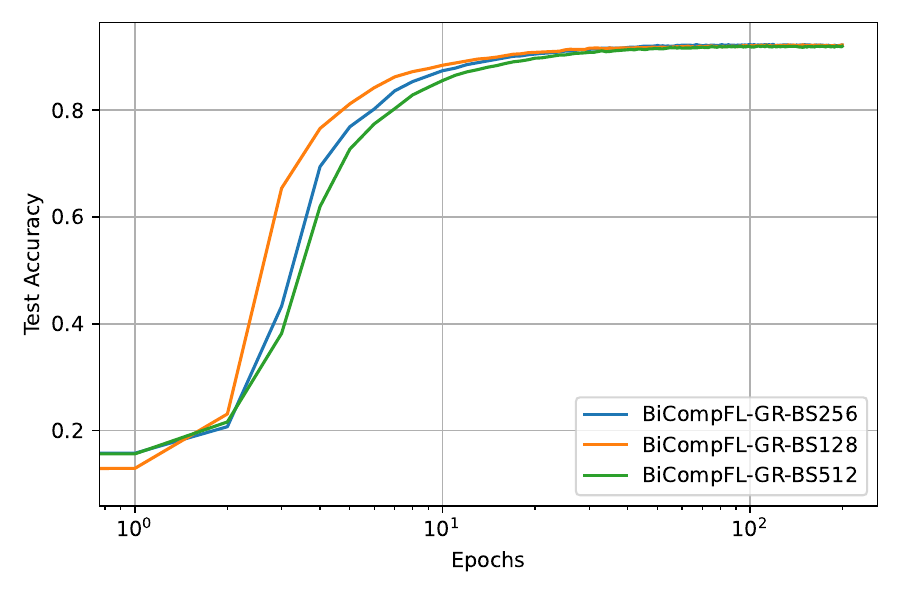}}
\subfigure[Test Accuracy over Communication Cost]{\includegraphics[width = .5\linewidth]{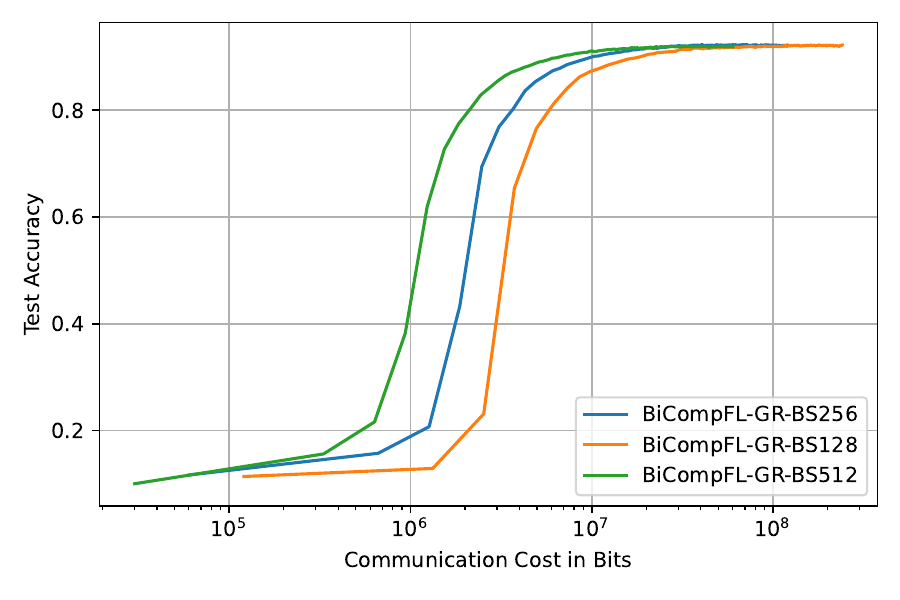}}
\caption{\rev{\algglobalrand\ With Fixed Block Allocation for Varying Block Sizes (BS) $\dimension/\nblocks$.}}
\label{fig:varbs}
\end{figure}
}

\subsection{Number of Importance Samples} \label{sec:varss}

\rev{
In \cref{fig:varss}, we study the sensitivity of our algorithms with respect to the number of importance samples $\nisamples$ at the example of \algglobalrand\. While larger number of $\nisamples$ slightly improves the performance as of the epoch number, the improvements do not outweigh the additional communication costs. Overall, our algorithm proves rather stable within reasonable ranges for $\nisamples$. We fix in all our experiments $\nisamples=256$, presenti
\begin{figure}[H]
\subfigure[Test Accuracy over Epochs]{\includegraphics[width = .5\linewidth]{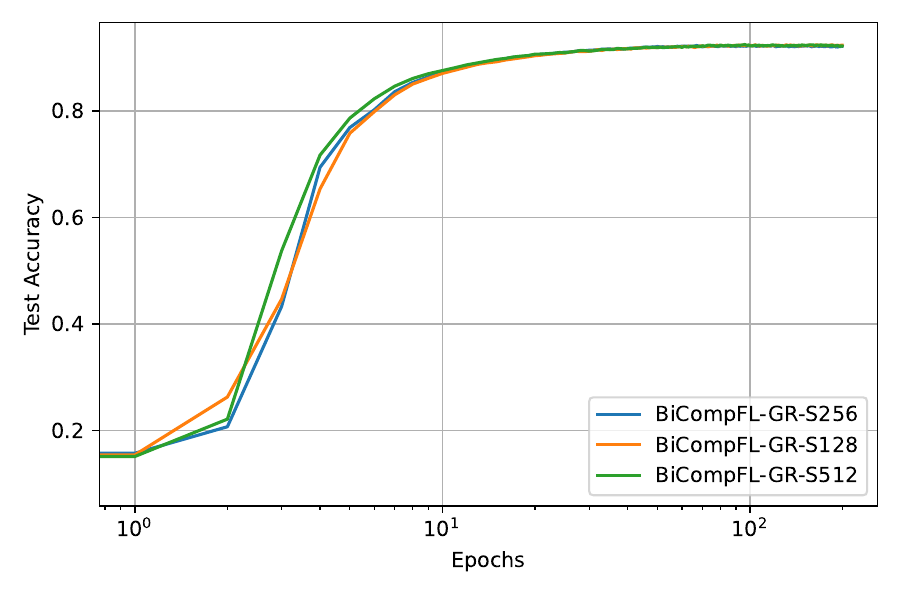}}
\subfigure[Test Accuracy over Communication Cost]{\includegraphics[width = .5\linewidth]{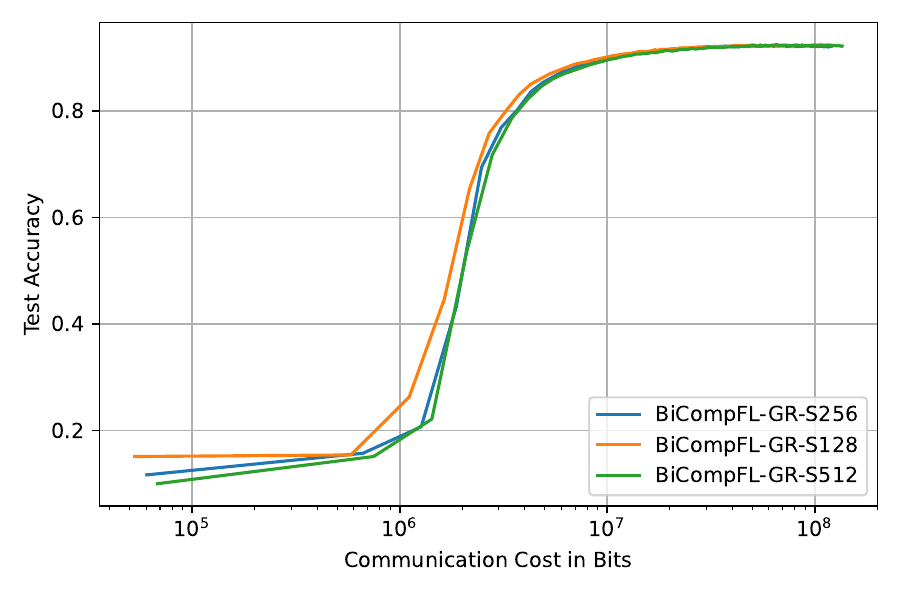}}
\caption{\rev{\algglobalrand\ with Varying Number of Importance Samples $\nisamples$ per Block.}}
\label{fig:varss}
\end{figure}
}

\end{document}